\documentclass{elsarticle}

\usepackage[utf8]{inputenc} 
\usepackage[T1]{fontenc}    
\usepackage{hyperref}       
\usepackage{url}            
\usepackage{booktabs}       
\usepackage{amsfonts}       
\usepackage{nicefrac}       
\usepackage{microtype}      
\usepackage{color}
\usepackage{verbatim}
\usepackage{booktabs}       
\usepackage{comment}
\usepackage{multirow}
\usepackage{scalerel,stackengine}
\usepackage{bbm}
\usepackage{bm}
\usepackage[english]{babel}
\usepackage{times}
\usepackage{mathtools}
\usepackage[font=scriptsize]{subcaption}
\usepackage{graphicx}
\usepackage{amsthm}

\allowdisplaybreaks
\DeclarePairedDelimiter\abs{\lvert}{\rvert}

\newtheorem{exe}{Theorem}
\newtheorem{theorem}{Theorem}

\newtheorem{lemma}[theorem]{Lemma}
\newcommand\norm[1]{\left\|#1\right\|}
\newcommand\bnorm[1]{\left\|#1\right\|}
\newcommand{\fn}{f}
\newcommand{\E}{\mathbb{E}}
\newcommand{\R}{\mathbb{R}}
\newcommand{\1}{\mathbbm{1}}
\newcommand{\Tr}{\mathrm{Tr}}

\newcommand\independent{\protect\mathpalette{\protect\independenT}{\perp}}
\def\independenT#1#2{\mathrel{\rlap{$#1#2$}\mkern2mu{#1#2}}}
\newcommand{\highlight}[1]{\textcolor{black}{#1}}

\addto\captionsenglish{\renewcommand{\appendixname}{Appendix }}

\journal{Pattern Recognition}

\begin{document}
\begin{frontmatter}

\title{\highlight{On the Sample Complexity of \\ Rank Regression from Pairwise Comparisons}}

\author[1]{Berkan Kad\i o\u{g}lu\corref{cor} \fnref{fn1}}
\ead{kadioglu.b@ece.neu.edu}
\author[1]{Peng Tian\corref{2} \fnref{fn1}}
\ead{pengtian@ece.neu.edu}
\author[1]{Jennifer Dy}
\ead{jdy@ece.neu.edu}
\author[1]{Deniz Erdo\u{g}mu\c s}
\ead{erdogmus@ece.neu.edu}
\author[1]{Stratis Ioannidis}
\ead{ioannidis@ece.neu.edu}

\address[1]{Department of Electrical and Computer Engineering, Northeastern University, Boston, MA, 02115, USA.}
\cortext[cor]{Corresponding author}
\fntext[fn1]{Berkan Kad\i o\u{g}lu and Peng Tian are both first authors of this paper.}

\begin{abstract}
We consider a rank regression setting, in which a dataset of $N$ samples with features in $\mathbb{R}^d$ is ranked by an oracle via $M$ pairwise comparisons. 
Specifically, there exists a latent total ordering of the samples; when presented with a pair of samples, a noisy oracle identifies the one ranked higher with respect to the underlying total ordering.
A learner observes a dataset of such comparisons and wishes to regress sample ranks from their features.
\highlight{We show that to learn the model parameters with $\epsilon > 0$ accuracy, it suffices to conduct $M \in \Omega(dN\log^3 N/\epsilon^2)$ comparisons uniformly at random when $N$ is $\Omega(d/\epsilon^2)$.}
\end{abstract}

\end{frontmatter}
\section{Introduction}
We consider a rank regression setting, in which a dataset of samples with features in $\mathbb{R}^d$ is ranked by an oracle via pairwise comparisons. 
Specifically, there exists a latent total ordering of the samples; when presented with a pair of samples, the (possibly noisy) oracle identifies the one ranked higher w.r.t. the underlying total ordering. 
A learner observes a dataset of such comparisons and wishes to regress sample ranks.

Rank regression has a broad range of applications in fields as diverse as social science \citep{schultz2004learning,zheng2009mining,koren2011ordrec}, economics \citep{mcfadden1973conditional,ryzin1999relationship}, and medicine \citep{tian2019severity,yildiz2019classification,guo2019variational}, to name a few. 
For example, disease severity can be regressed from patient records by presenting pairs to a medical expert and asking her to rank them. 
A dataset of such pairwise comparisons is more informative than a dataset with class labels containing diagnostic outcomes because comparisons reveal \emph{intra-class, relative} severity within, e.g., the healthy or diseased class, that cannot be inferred from class labels alone.
As an additional practical benefit, comparison labels also often exhibit lower variability across experts: experts are more likely to agree when comparing pairs rather than making absolute diagnoses: 
this has been observed in a variety of domains, including medicine \citep{campbell2016plus,kalpathy2016plus, stewart2005absolute}, movie recommendations \citep{brun2010towards, desarkar2010aggregating, desarkar2012preference, liu2014ordinal}, travel recommendations~\citep{zheng2009mining}, music recommendations \citep{koren2011ordrec}, and web page recommendations \citep{schultz2004learning}.
These advantages make learning from comparisons quite advantageous in practice; 
in an extreme example illustrating this, \citet{yildiz2019classification} used comparisons among just 80 images to train a neural network of 5.9 million parameters that attained a 0.92 AUC on a test set.

This empirical success motivates us to study the sample complexity of algorithms for learning from comparisons. 
However, doing so poses a significant challenge. 
In contrast to the standard probably approximately correct (PAC) learning setting, where samples 
are assumed to be i.i.d, learning from comparisons necessarily leads to a violation of independence. 
Even in a simple generative model where (a) samples are drawn independently and (b) pairs presented to the oracle are selected uniformly at random, any two pairs sharing a sample are correlated. 
This dependence complicates the application of concentration inequalities such as, e.g., Chernoff bounds in this setting. 

\highlight{
The main contributions of our work are as follows. 
We propose an estimator for the parameters of a generalized linear parametric model, which encompasses classical preference models such as Bradley-Terry~\citep{bradley1952rank} and Thurstone~\citep{thurstone1927law}.
We overcome the aforementioned violation of independence and prove a sample complexity guarantee on model parameters. In particular, assuming Gaussian distributed features, we characterize the convergence of the estimator to a rescaled version of the model parameters w.r.t. the ambient dimension $d$, the number of samples $N$, and the number of comparisons $M$ presented to the oracle. 
We show that to attain an accuracy $\epsilon>0$ \highlight{in model parameters}, it suffices to conduct $\Omega(dN\log^3 N/\epsilon^2)$ comparisons when the number of samples is $\Omega{(d/\epsilon^2)}$. 
Finally, we confirm this dependence with experiments on synthetic data.
}

\section{Related Work}
\label{subs:rel_work}
In \emph{rank aggregation}~\citep{fligner1993probability,dwork2001rank,cattelan2012models,marden2014analyzing}, subsets of samples are ranked by a noisy oracle, and a learner attempts to reconstruct a total ordering from these noisy rankings without access to sample features. 
Works on noisy sorting assume that the observed pairwise comparisons deviate from an existing underlying ordering via i.i.d.~Bernoulli noise.
\citet{braverman2008noisy} propose a tractable active learning algorithm that requires $\Omega(N\log(N))$ comparisons to recover the underlying ordering with high probability.
\citet{jamieson2011active} actively rank samples with $\Omega(d\log^2N)$ pairwise comparisons when samples are embedded into an unobserved $d$-dimensional space.
In the passive learning setting, assuming that the comparisons are samples from an unknown distribution over the underlying ordering, \citet{ammar2011ranking} propose a maximum entropy method with $\Omega(N^2)$ pairwise comparisons.
Under the same non-parametric model, \citet{negahban2012iterative} learn the ordering via an iterative rank aggregation algorithm requiring a total of $\Omega(N\log N)$ comparisons in which each pair needs to be repeated $\Omega(\log N)$ times. 
\citet{shah2016stochastically} show that a minimax optimal estimator can estimate the preference matrix with $O(\log^2 N/N)$ error.
By showing that the preference matrix has rank $r \ll N$ under a suitable transformation, \citet{rajkumar2016can} show that $\Omega(rN\log N)$ comparisons suffice. 
\citet{saha2019many} use pairwise comparisons to construct a graph $G([N], E)$, where the nodes are samples and edges represent the comparison labels.
Assuming that the neighboring samples in the graph are proximal in the ordering, they propose a support vector machine (SVM) algorithm that discovers a consistent total ordering with high probability.
This algorithm has a sample complexity of $\Omega(N^2\chi(\overline G))^\frac{2}{3}$, where $\chi(\overline G)$ is the \emph{chromatic number} of the complement graph $\overline G$.


Among parametric models, \citet{hajek2014minimax} show that the maximum likelihood estimator under Plackett-Luce model ~\citep{plackett1975analysis} requires $\Omega(N\log N)$ comparisons to learn Plackett-Luce scores. 
\citet{vojnovic2016parameter} show that estimating Thurstone ~\citep{thurstone1927law} scores via MLE requires $O(N\log N/\lambda)$ comparisons, where $\lambda$ is the smallest nonzero eigenvalue of the Laplacian of a graph generated by comparisons.
Assuming comparison labels are independent, \citet{ailon2012active} proposes an active learning algorithm that requires $\Omega(\epsilon^{-6}N\log^5N)$ comparison labels for a risk of $\epsilon$ times the optimal risk, where risk is a function that is minimized at the correct ordering. 
Spectral ranking methods also learn sample scores with theoretical guarantees.
\citet{negahban2017rank} show that the rank centrality algorithm learns scores in $\Theta(N\log^3N)$ comparisons, while several algorithms generalize this setting and improve upon this bound \citep{maystre2015fast, agarwal2018accelerated}.
For example, ASR \citep{agarwal2018accelerated} learns scores in $\Omega(\xi^{-2} m^3 N\text{poly}(\log N))$ $m$-way comparisons, error on BTL parameters where $\xi$ is the spectral gap of the graph Laplacian.

The \emph{rank regression} setting we study departs from the above works in regressing rankings from sample features. 
Even though inference algorithms for ranking regression and applications abound \citep{joachims2002optimizing, pahikkala2009efficient, tian2019severity, burges2005learning, chang2016automatic, yildiz2019classification, dubey2016deep}, in contrast to rank aggregation, sample complexity results are sparse.
Using independent pairwise comparisons, \citet{canonne2015testing} propose an algorithm over sample pairs that tests whether the empirical distribution is close to a target distribution.
\citet{kane2017active} propose an active learning algorithm to infer class labels via a special pairwise comparison oracle, that indicates which sample is closer to the separating hyperplane of class labels.

Our model encompasses the Bradley-Terry \citep{bradley1952rank} and Thurstone \citep{thurstone1927law} models; under both, our setting can be seen as learning a linear classifier over sample differences. 
Learning linear classifiers is of course classic in both the standard PAC learning setting \citep{vapnik1971uniform,valiant1984theory,ehrenfeucht1989general,kearns1994introduction, vapnik2006estimation, balcan2013active, balcan2017sample} and variants, including agnostic \citep{kalai2008agnostically,mammen1999smooth} and active \citep{hanneke2007bound,cavallanti2011learning,balcan2013active,balcan2017sample,awasthi2016learning,zhang2018efficient} learning. 
We stress that all of the above works operate on linear classifiers under the assumption of i.i.d.~samples, and therefore do not readily generalize or apply to our setting. This is precisely because pairs of samples are correlated, a phenomenon that is not present in standard PAC learning.

\highlight{
Closest to our setting, \citet{niranjan2017inductive} and \citet{chiang2017rank} analyze pairwise rank regression and provide sample complexity bounds.
\citet{niranjan2017inductive} recover the correct ranking with $N=\Omega(d^2)$ samples and a number of comparisons that are polylogarithmic in $N$, while \citet{chiang2017rank} provide a guarantee that depends on the $\ell_2$-distortion (due to noise) of the pairwise comparison matrix.
Nevertheless, both works ignore dependence across sample pairs. 
In particular, they analyze the concentration of labels over pairs of samples using Rademacher complexity bounds from \cite{bartlett2002rademacher}, that apply only if sample pair differences $\bm{x}_i-\bm{x}_j$ are independent. 
As a result, guarantees provided in both \citep{niranjan2017inductive} and \citep{chiang2017rank} only hold if every sample  appears in only a single pair. 
We depart by explicitly addressing this, and providing guarantees in the (more realistic) setting where samples can be compared more than once.
}

\section{Problem Formulation}
\label{sec:prob_for}

\noindent \textbf{Notation.}
For $N\in\mathbb{N}$, we denote by $[N] \equiv \{1, 2, \dots, N \}\subset \mathbb{N}$ the set of integers from $1$ to $N$, and use $\norm{\cdot}$ for Euclidean (spectral) norm of vectors (matrices).
The minimum and maximum singular values of a matrix $\bm A$ is denoted with $\lambda_{\min}(\bm A)$ and $\lambda_{\max}(\bm A)$, respectively.
We denote by $\1_{\mathcal A}$ the indicator function of a predicate $\mathcal A$, i.e., $\1_{\mathcal{A}}= 1$ if $\mathcal A$ is true and $0$ otherwise.

\noindent \textbf{Generative Model.}
We consider a setting in which an expert is presented with pairs of samples from a dataset. 
The expert produces a (possibly noisy) \emph{comparison label} for each pair, \emph{i.e.}, she selects among two samples the one ranked higher with respect to an underlying total ordering of the samples.
Formally, we are given a dataset of {$2N$} samples, each denoted by {$i \in [2N]$}. 
Each sample $i$ has a corresponding feature vector $\bm X_i \in \R^d$. 
{Using the first half of the dataset (i.e., $[N]$), the} expert is presented with $M$ pairs of samples {$(I_m, J_m) \in [N]\times[N]$} where $m \in [M]$ and produces a comparison label $Y_m \in \{+1, -1\}$ where $Y_m = +1$ if $I_m$ ranks higher than $J_m$ and $-1$ otherwise. 
We denote the dataset of all comparisons by $\mathcal{D}=\{(I_m, J_m, Y_m) \}_{m=1}^{M}$.

We assume that the feature vectors $\bm X_i \in \R^d$ are independent and identically distributed (i.i.d.) {Gaussian vectors with mean $\bm \mu \in \R^d$ and positive definite covariance $\bm \Sigma \in \R^{d\times d}$, i.e., $\bm X_i \sim \mathcal{N}\big(\bm \mu, \bm \Sigma\big)$. 
We assume that the eigenvalues of $\bm \Sigma$ are ordered so that $\lambda_1 \geq \lambda_2 \geq \dots \geq \lambda_d > 0$.} 
Furthermore, we assume that $I_m, J_m$ are sampled uniformly at random from {$[N]$} and are independent of each other and {$\{\bm X_i\}_{i=1}^{2N}$.}
Labels $Y_m$ are independent of all other variables conditioned on $I_m, J_m, \bm X_{I_m}, \bm X_{J_m}$ {and are distributed according to the following model:} there exists a $\bm \beta \in \R^d$ such that the conditional distribution of $Y_m$ is given by
\begin{equation}
    \label{eq:conditional}
    \Pr(Y_m = 1 |\bm X_{I_m}=\bm x, \bm X_{J_m} = \bm y) = \fn(\bm \beta^\top(\bm x - \bm y)),
\end{equation}
where the function $\fn : \R \rightarrow [0, 1]$ is (a) non-decreasing, continuously differentiable 
and (b) satisfies
\begin{equation}
    \label{eq:properties}
    \lim_{x\rightarrow\infty} \fn(x) = 1, ~~~ \lim_{x\rightarrow-\infty} \fn(x) = 0, ~~~ \fn(-x) = 1-\fn(x).
\end{equation}
For example, $\fn(x)$ could be the sigmoid function, i.e. $\fn(x) = 1/(1+e^{-x})$, which results in the well known Bradley-Terry model \citep{bradley1952rank}.
Alternatively, $\fn$ could be the cumulative distribution function of standard normal distribution, i.e. $\fn(x) = (1+\text{erf}(x))/2$, which corresponds to the Thurstone model~\citep{thurstone1927law}. 
Both of these examples satisfy the aforementioned properties (a) and (b).
\begin{figure}
\centering
\subcaptionbox{\label{fig:data_unwhitten}}{\includegraphics[width=0.25\textwidth]{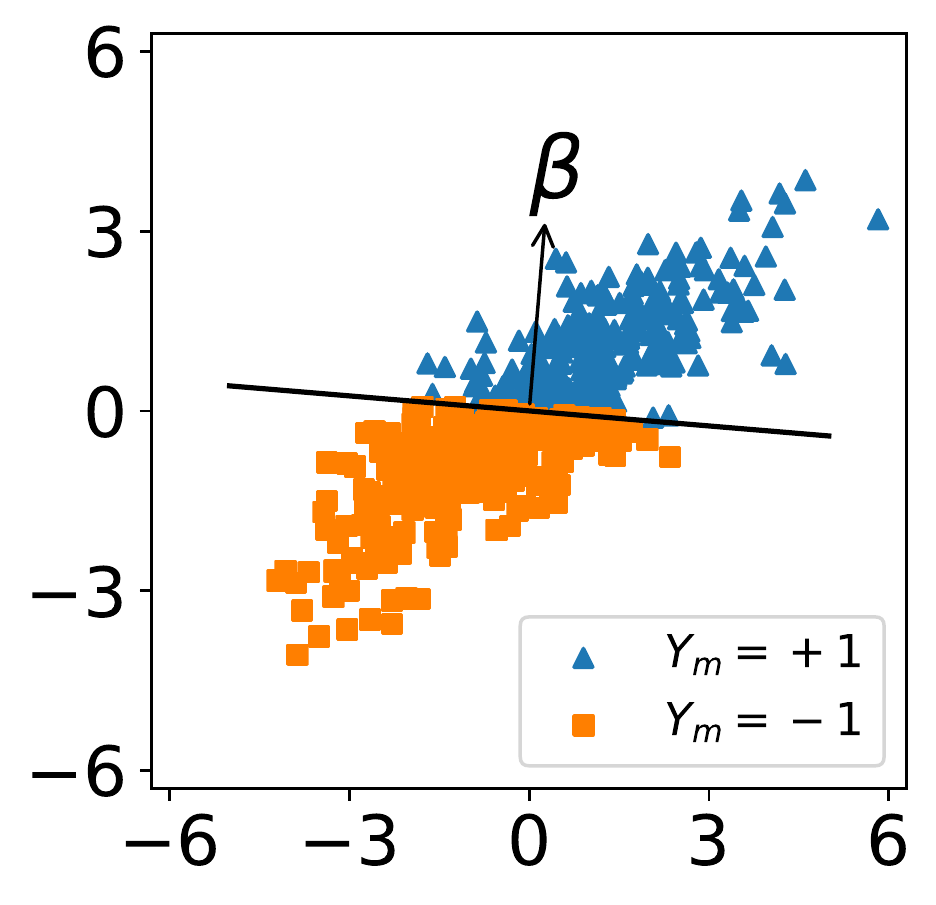}}%
\subcaptionbox{\label{fig:data}}{\includegraphics[width=0.25\textwidth]{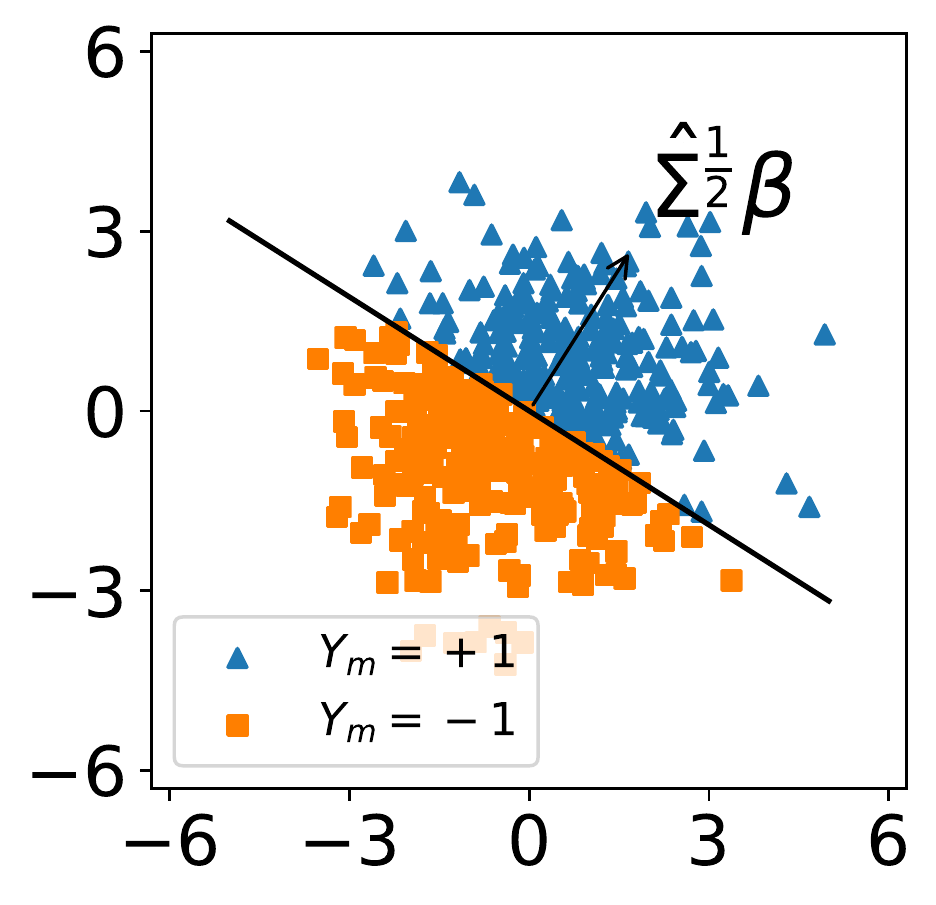}}%
\subcaptionbox{\label{fig:data_mirror}}{\includegraphics[width=0.25\textwidth]{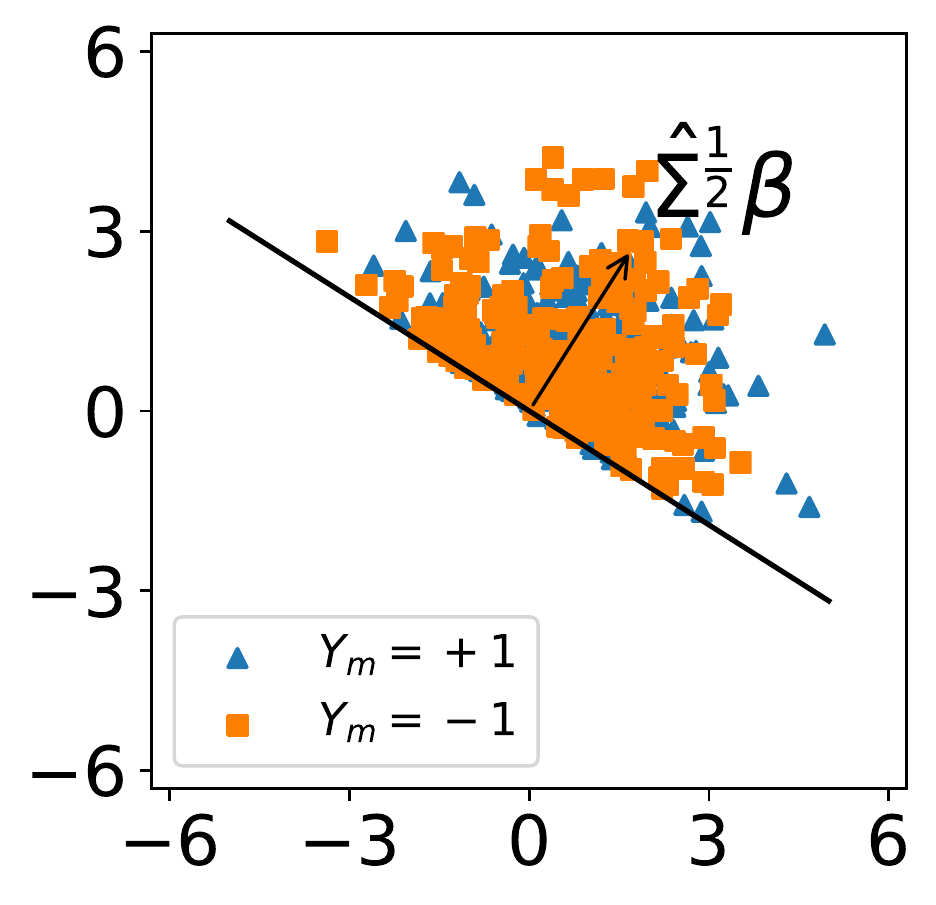}}%
\subcaptionbox{\label{fig:data_mirror_est}}{\includegraphics[width=0.25\textwidth]{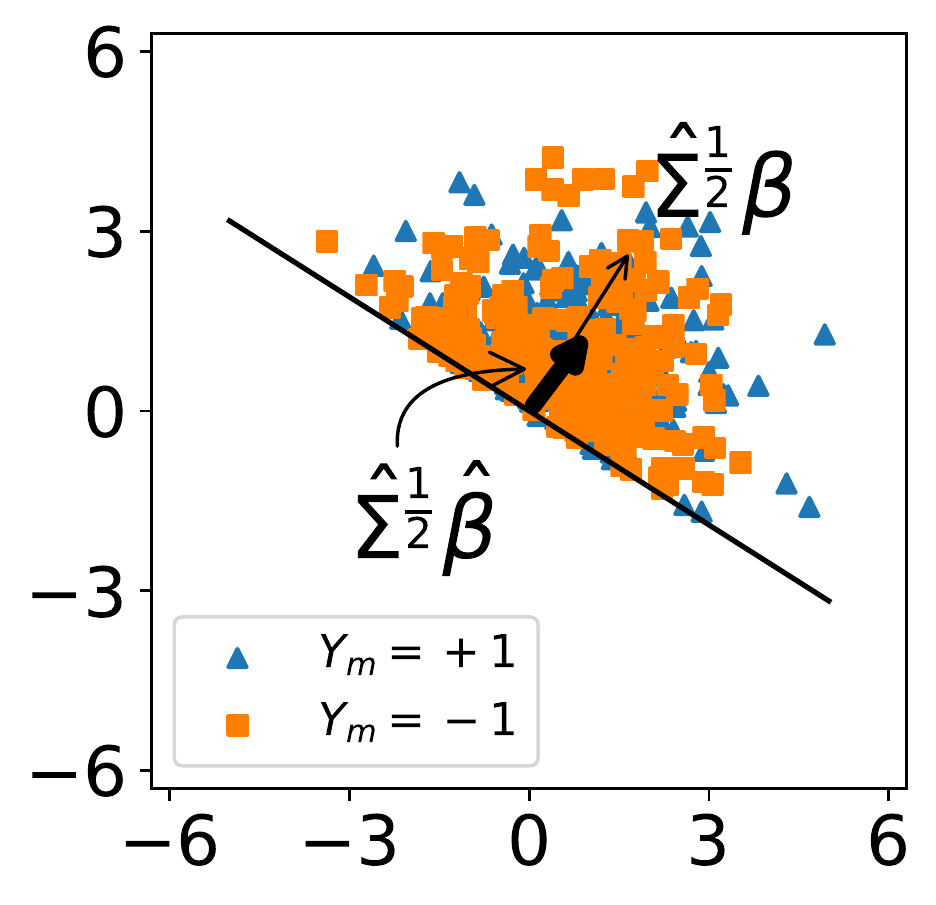}}%
	\caption{Intuition behind the estimator in Eq.~\eqref{eq:estimator}. 
	We consider a dataset of i.i.d. Gaussian samples $\{\bm X_i\}_{i=1}^{2N}$.
	Differences $\bm{X}_{I_m}-\bm{X}_{J_m}$ are shown in Fig.~\ref{fig:data_unwhitten}, along with ${\bm\beta}$ and the corresponding separating hyperplane. Colors indicate labels $Y_m\in\{-1,+1\}$. 
	We can rewrite Eq.~\eqref{eq:estimator} as $\hat{\bm{\beta}} =\bm{\hat{\Sigma}}^{-\frac{1}{2}} \cdot \frac{1}{M}\sum_{m=1}^M Y_m \hat{\bm{\Sigma}}^{-\frac{1}{2}} (\bm{X}_{I_m}-\bm{X}_{J_m})$. 
	Multiplying vectors $\bm{X}_{I_m}-\bm{X}_{J_m}$ with $\hat{\bm \Sigma}^{-\frac{1}{2}}$ gives the whitened sample pairs in Fig.~\ref{fig:data}; in this coordinate system, the separating hyperplane has normal $\bm{\hat \Sigma}^{\frac{1}{2}}\bm{\beta}$. 
	The resulting whitened pairs are multiplied by the labels $Y_m$ in Fig.~\ref{fig:data_mirror}; this results in a ``mirroring'' over the separating hyperplane defined by $\bm{\hat \Sigma}^{\frac{1}{2}}\bm{\beta}$. 
    Their average (i.e., $\hat{\bm \Sigma}^{\frac{1}{2}} \hat{\bm \beta}$) is approximately co-linear with $\bm{ \hat \Sigma}^{\frac{1}{2}}\bm{\beta}$.
    The final multiplication with $\hat{\bm \Sigma}^{-1/2}$ recovers ${\bm \beta}$ (up to a multiplicative constant).}
\label{fig:intuition}
\end{figure}

\noindent \textbf{Parameter Estimation.}
The learner observes $\mathcal{D}$ and estimates $\bm \beta$ via:
\begin{equation}
\label{eq:estimator}
	\hat{\bm \beta} =\textstyle\frac{1}{M} \sum_{m = 1}^M Y_m {\bm{\hat\Sigma}^{-1}}(\bm X_{I_m} - \bm X_{J_m})\in \mathbb{R}^d,
\end{equation}
where $\bm{\hat\Sigma}$ is an estimator of $\bm{\Sigma}$, computed over the second half of the samples through:
\begin{align}
    \bm{\hat\Sigma} &=\textstyle \frac{1}{N-d-2}\sum_{i=N+1}^{2N}(\bm X_i - \bm{\hat\mu})(\bm X_i - \bm{\hat\mu})^\top, ~ \text{where}~ \bm{\hat\mu} = \frac{1}{N}\sum_{i=N+1}^{2N}\bm X_i. \label{eq:sigma_hat}
\end{align}
Note that $\E[\bm{\hat \Sigma}^{-1}] = \bm \Sigma^{-1}$ (see, e.g., \cite{hartlap2007your}).
We separate the dataset in two halves to ensure the independence of $\bm{\hat{\Sigma}}$ from labels in $\mathcal{D}$. 
Eq.~\eqref{eq:estimator} resembles a two-class linear discriminant analysis (LDA) estimator (see, e.g., \cite{friedman2001elements}) and is indeed unbiased up to a positive multiplicative constant (see Lemma~\ref{lem:unbiased}); this is a consequence of Stein's Lemma \citep{stein1973estimation}, stated formally in Section~\ref{app:technical}.
Fig.~\ref{fig:intuition} provides some intuition as to why this is the case. 
Despite the simplicity of our proposed estimator, characterizing its sampling complexity poses a significant challenge. 
Non-asymptotic bounds establishing consistency typically rely on i.i.d.~assumptions; this is indeed natural to assume for samples $\{\bm{X}_i\}_{i=1}^{2N}.$ 
However, pairwise comparisons introduce correlations in labels $\{\bm{Y}_m\}_{m=1}^M$: this is precisely because samples are re-used in pairs. 
We stress that conditioning on $\{\bm{X}_i\}_{i=1}^{2N}$ \emph{does not} resolve this issue, as labels are still dependent through random variables $({I}_m,{J}_m)$.
\begin{table}[t!]
\caption{Summary of Notation}
\label{tab:notation}
\centering
{\scriptsize
\renewcommand{\arraystretch}{1.1}
\begin{tabular}{c|l|c|l}
\hline
$N$ & number of samples & $\bm{X}_i$ & Gaussian feature vector \\
$M$ & number of comparisons & $d$ & dimensionality of a feature vector \\
$\norm{\cdot}$ & $\ell_2$ (spectral) norm of vectors (matrices) & $i,n$ & sample index in $[N]$\\
$Y_m$ & comparison label & $m$ & comparison index in $[M]$\\
$I_m, J_m$ & uniform random variables in $[N]$ & $\mathcal{D}$ & comparison dataset \\
$[N]$ & set of integers from $1$ to $N$ & $\bm{\beta}$ & parameter vector/model in $\mathbb{R}^d$ \\
$c_i$ & constants & \\
\hline
\end{tabular}
}
\vskip -0.1in
\end{table}

\section{Technical Preliminary}\label{app:technical}
\noindent In this section, we review some known results.
The first is a variant of Stein's lemma from \citet{liu1994siegel}; we use this to show that our estimator is unbiased up to a constant. 
\begin{lemma}
[Stein's Lemma \citep{stein1973estimation,liu1994siegel}]\label{lemma:stein} Let $\bm{X} \in \R^d$, $\bm{X}' \in \R^{d'}$ be jointly Gaussian random vectors. Let the function $h:\R^{d'} \rightarrow \R$ be differentiable almost everywhere and satisfy $\E\left[| \partial h\left(\bm{X}'\right)/\partial X_i|\right] < \infty$, $i \in [d']$, then $\mathrm{Cov}\left(\bm{X}, h(\bm{X}')\right) = \mathrm{Cov}\left(\bm{X}, \bm{X}'\right)\E\left[ \nabla h\left( \bm{X}' \right) \right]$.
\end{lemma}
\noindent The second lemma we utilize bounds the tail of the norm of standard Gaussian vectors.
\begin{lemma}[Centralized Chi-Squared Tail Bound \citep{dasgupta2003elementary}]\label{lem:chi_tail}
Let $F_X(x;k)$ be the CDF of centralized chi-square distribution with $k$ degrees of freedom.
Then, $1 - F_X(zk; k) \leq (z e^{1 - z})^{k/2}$ for $z > 1$.
\end{lemma}
\noindent A consequence of the way we select random pairs is that the joint distribution of the number of times each sample is selected is multinomial. 
The next inequality provides a bound for such variables:
\begin{lemma}[Bretagnolle-Huber-Carolle Inequality \citep{wellner2013weak}]
\label{lemma:huber_carol}
Let $\{M_i\}_{i=1}^N$ be multinomially distributed r.v.s with parameters $M$ , $\{p_i\}_{i=1}^N$. Then $\Pr\left( \sum_{i = 1}^N \left| \frac{M_i}{M} - p_i \right| \geq \epsilon \right) \leq 2^N e^{-\frac{\epsilon^2M}{2} }.$
\end{lemma}
\noindent We also state the following classic inequality:
\begin{lemma}[Hoeffding's Inequality \citep{hoeffding1994probability}]
\label{lem:hoeffding}
Let $X = \frac{1}{N}\sum^N_{i=1}X_i$, where $a_i \leq X_i \leq b_i$ and $X_i$ are independent, and $\mu = \E[X]$. Then $\Pr(|X \!-\! \mu|\! \geq\! \epsilon) \leq 2 e^{-2N^2\epsilon^2/\sum_{i=1}^N(b_i - a_i)}$.
\end{lemma}
\noindent Recall that a random variable $X\in\mathbb{R}$ is sub-gaussian if there exists a $c>0$ for all $t\geq0$ s.t. $\Pr(|X| > t) \leq 2 \exp(-t^2/c)$. Then, we define the sub-gaussian norm of $X$, denoted by $\norm{X}_{\psi_2}$ as $\norm{X}_{\psi_2} = \inf\left\{ t > 0: \E\left[ e^{X^2/t^2} \right] \leq 2 \right\}$.
Moreover, a random vector $\bm X\in \mathbb{R}^d$ is called sub-gaussian if one dimensional marginals $\bm v^\top \bm X$ are sub-gaussian for all $ \bm v \in S^{d-1}$, where $S^{d-1} = \{\bm x \in \R^d: \norm {\bm x} = 1 \}$. 
The sub-gaussian norm of $\bm X$ is then defined as $\norm{\bm X}_{\psi_2} = \sup_{\bm v \in S^{d-1}}\norm{\bm v^\top \bm X}_{\psi_2}$.
\highlight{
The next lemma provides lower and upper bounds for the singular values of random design matrices.
}
\begin{lemma}
[Theorem 5.39 of \citet{vershynin2012}]
\label{th:singvalbound}
Let $\bm A \in \R^{N\times d}$ be a matrix whose rows $\bm A_i$ are independent sub-gaussian isotropic random vectors. Then for every $t\geq0$, with probability at least $1-2e^{-ct^2}$ one has $\sqrt{N} - C\sqrt d - t \leq \lambda_{\min}[\bm A] \leq \lambda_{\max}[\bm A] \leq \sqrt{N} + C\sqrt d + t$ where $c, C > 0$ depend only on the maximum sub-gaussian norm $\max_i \norm{\bm A_i}_{\psi_2}$ of the rows.
\end{lemma}
\highlight{
We use Lemma~\ref{th:singvalbound} to bound the eigenvalues of the feature covariance matrix.
Lastly, the next lemma is used for bounding the norm of sub-gaussian random vectors.
}
\begin{lemma}[Theorem 1 of \citet{hsu2012tail}]
\label{th:subgnorm}
Let $\bm A \in \R^{m\times n}$ be a matrix, and let $\Sigma \equiv \bm A ^\top \bm A$. Suppose that $\bm x \in \R^d$ is a sub-gaussian random vector with mean $\bm \mu \in \R^d$ and $\sigma = \norm{\bm x}_{\psi_2}$. For all $t > 0$, $\Pr( \norm{\bm A \bm x}^2 > \sigma^2( \Tr(\Sigma)+2\sqrt{\Tr(\Sigma^2)t} + 2\norm{\Sigma}t) + \norm{\bm A \bm \mu}^2( 1 + 4( \frac{\norm{\Sigma}^2}{\Tr(\Sigma^2)}t)^{1/2} + \frac{4\norm{\Sigma}^2}{\Tr(\Sigma^2)}t)^{1/2}) \leq e^{-t}$.
\end{lemma}

\section{Main Results}
\label{sec:main}
We first establish that $\bm{\hat \beta}$ is an unbiased estimator of $\bm \beta$ up to a multiplicative constant.
\begin{lemma}
\label{lem:unbiased}
For $\hat{\bm \beta}$ in Eq.~\eqref{eq:estimator}, $ \E[\hat{\bm{\beta}}]\! =\! c_1\bm{\beta},$ where $c_1\! =\! 4\E\left[ \fn'\left( \bm{\beta}^T(\bm X_{I_m}\! -\! \bm X_{J_m}) \right) \right]\! >\! 0$.
\end{lemma}
The proof can be found in \ref{app:unbiased}.
This result is a consequence of Stein's lemma \citep{stein1973estimation} (see Lemma~\ref{lemma:stein} in Section~\ref{app:technical}).
Learning ${\bm \beta}$ up to a multiplicative constant {suffices}, as only the direction is enough to reveal the separating hyperplane between positive and negative sample pairs.
\highlight{
Constant $c_1$ captures label noise: by \eqref{eq:properties}, $f'$ is non-negative and maximized at zero; for functions $f$ that are ``flatter'' around zero the maximum value of $f'$ and, therefore, $c_1$ is smaller.
Such $f$ also result in noisy labels.
Crucially, although our guarantees depend on $c_1$ (see Theorem.~\ref{th:bound} below), 
}
our estimator does not depend on $c_1$: no knowledge of $c_1$ is required to compute $\hat{\bm \beta}$ via Eq.~\eqref{eq:estimator}.
Theorem~\ref{th:bound} establishes that the parameters $\bm \beta$ are PAC learnable.
\setcounter{exe}{0}
\begin{exe}
\label{th:bound}
For $\epsilon > 0$, sample count $N/\log^2N = \Omega\left(\frac{d}{\epsilon^2\lambda_d}\right)$ and comparison count $M = \Omega\left(\frac{dN\log^3N}{\epsilon^2\lambda_d}\right)$,
\begin{align}
    \label{eq:bound}
    \Pr\left(\norm{\hat{\bm{\beta}} - c_1\bm{\beta}} \geq \epsilon \right) &\leq c_2N\max\left\{\left(\frac{\sqrt{6\log N}}{N}\right)^d, e^{-\frac{N\epsilon^2\lambda_d}{c_3d\log N}} \right\},
\end{align}
where $c_1>0$ is given by Lemma~\ref{lem:unbiased} and $c_2, c_3 > 0$ are absolute constants.
\end{exe}
Theorem~\ref{th:bound}, which we prove below, allows us to characterize the sample complexity of $\bm{\hat \beta}$ in terms of the ambient dimension $d$, number of samples $N$, and number of comparisons $M$. 
It implies that to attain an accuracy $\epsilon > 0$ with high probability, the estimator requires $\Omega(d\log^2N/\epsilon^2\lambda_d)$ samples; this is of the same order as standard PAC learning guarantees for linear classifiers \citep{ehrenfeucht1989general,balcan2013active, haussler1994predicting} and is also corroborated by our experiments in Section~\ref{sec:experiments}.
Moreover, the number of comparisons required to attain an accuracy $\epsilon > 0$ is $\Omega(dN\log^3N/\epsilon^2\lambda_d)$, i.e. comparisons scale almost linearly with $N$.

\highlight{We emphasize that, to identify the separating plane, it suffices to know $\bm{\beta}$ up to a non-negative multiplicative constant. This motivates the l.h.s.~of Eq.~\eqref{eq:bound} in Theorem~\ref{th:bound}. Nevertheless, the above guarantee should become more stringent for smaller $c_1>0$. Recalling that $c_1$ captures the level of label noise (smaller indicates more noise), the latter's impact on this bound is captured by replacing desired accuracy $\epsilon$ with $\epsilon' = c_1\epsilon$, so that $c_1$ appears as an additional constant in the r.h.s.~of Eq.~\eqref{eq:bound}.}

\section{Proof of Theorem \ref{th:bound}}
\label{sec:proof}
\highlight{
The proof proceeds in the following manner.
We first use a union bound to bound the tail of $\norm{\bm{\hat \beta} - c_1\bm\beta}$ via several constituent terms. 
Contrary to standard concentration proofs, however, sums appearing in these terms involve dependent random variables. 
We nevertheless bound these terms by union bounds, conditioning, and leveraging the boundedness of random variables summed.
From a technical standpoint, we leverage the Bretagnolle-Huber-Carolle inequality (see Lemma~\ref{lemma:huber_carol}), and combine it with classic concentration inequalities (like Hoeffding's inequality, Lemma~\ref{lem:hoeffding}, and Lemma~\ref{lem:eigenineq}, due to \citep{vershynin2012}).
}
We start with a simple bound on $\norm{\hat{\bm{\beta}} - c_1\bm{\beta}}$.
\begin{lemma}
\label{lem:2terms}
The estimator $\hat{\bm\beta}$ given by Eq.~\eqref{eq:estimator} satisfies:
\begin{align*}
    &\Pr\left(\norm{\hat{\bm{\beta}} - c_1\bm{\beta}} > \epsilon\right) \leq 4\Pr\left(\norm{\bm{\hat \Sigma}^{-1} - \bm{\Sigma}^{-1}}\cdot\norm{\E[Y_m (\bm{X}_{I_m} - \bm{\mu})]} > \epsilon/6\right)  \nonumber \\
    &+ 4\Pr\bigg(\bnorm{\frac{1}{M}\!\sum_{m = 1}^M\! Y_m\bm{\Sigma}^{-1/2}(\bm{X}_{I_m} \!\! -\! \bm{\mu}) \!-\! \E[Y_m\bm{\Sigma}^{-1/2}(\bm X_{I_m} \!\!-\! \bm \mu)]} \!>\! \frac{\sqrt{\lambda_{d}}\epsilon}{6}\bigg).
\end{align*}
\end{lemma}
The proof, via a union bound, can be found in \ref{app:2terms}.
The terms $Y_m$, $\bm{X}_{I_m}-\bm{\mu}$ are not independent.
This is because (a) the same sample ${\bm X_i}$ can be selected more than once, and, crucially, (b) the labels $Y_m$ are coupled via the selection of the second sample in each pair. 
As a consequence, standard concentration bounds do not immediately apply.
As a remedy, we  condition on events under which the above variables are independent and refine this bound further.
To do so, we introduce several quantities of interest.
\highlight{
Let 
\begin{align}
\bm W_n = \bm \Sigma^{-1/2}(\bm X_n - \bm\mu),
\end{align}
be the normalized feature vectors.
}
For $n\in [N]$, let the number of times $I_m = n$ be
\begin{align}
\label{eq:Mn}
M_n = \textstyle\sum_{m = 1}^M \1_{I_m=n}. 
\end{align}
For $n, j \in [N]$, let $g_{n, j} : (\R^d)^2 \rightarrow [-1, 1]$ be the expected comparison label conditioned on the features of samples $n,j\in [N]$ selected in a pair, i.e.:
\begin{align*}
\textstyle  g_{n,j}(\bm x_n, \bm x_j) &= \E[Y_m|I_m = n, J_m = j, \{\bm X_{n'} = \bm x_{n'}\}_{n' = 1}^{N}] 
= 2\fn(\bm \beta^\top (\bm x_n - \bm x_j)) - 1. \nonumber 
\end{align*}
Let $g_n: \R^d \rightarrow [-1, 1]$ be the expected label conditioned on the {$I_m$-th sample}:
\begin{align}
\label{eq:gn}
g_n(\bm x) &=  \textstyle \E[Y_m \mid I_m = n, \bm X_n = \bm x] = \int g_{n, 1}(\bm x, \bm y) \bm f_{\bm X_1}(\bm y) \bm d\bm y.
\end{align}
We will also need a similar quantity, $\tilde g_n : (\R^d)^{N} \rightarrow [-1, 1]$:
\begin{align}
    \label{eq:gtn}
    \tilde g_n(\{\bm x_{n'}\}_{n' = 1}^{N}) &= \textstyle \E[Y_m \mid I_m = n, \{\bm X_{n'} = \bm x_{n'}\}_{n' = 1}^{N}] 
    = \frac{1}{N}\sum_{j=1}^{N} g_{n, j}(\bm x_n, \bm x_j). 
\end{align}
Note that $g_n$ and $\tilde{g}_n$ are distinct, but the latter can be seen as a quantity that concentrates to $g_n$ as $N$ becomes large. We denote by $z_n : (\R^d)^{N} \rightarrow [-2, 2]$ their difference, i.e.:
\begin{align}
    \label{eq:zn}
  \textstyle z_n(\{\bm x_{n'}\}_{n' = 1}^{N}) = \tilde g_n(\{\bm x_{n'}\}_{n' = 1}^N) - g_n(\bm x_n).
\end{align}
Finally, let $\Delta_n\! :\! (\R^d)^{N}\! \rightarrow\! [-2, 2]$ be the difference between true label averages and $\tilde{g}_n$: 
\begin{align}
    \label{eq:Deltan}
\textstyle \Delta_n(\{\bm x_{n'}\}_{n' = 1}^{N}) = \frac{1}{M_n}\sum_{m:~I_m=n} Y_m - \tilde g_n(\{\bm x_{n'}\}_{n' = 1}^N).
\end{align}
Our next lemma bounds the second term in the r.h.s.~of Lemma~\ref{lem:2terms}.
\begin{lemma}
\label{lem:4terms}
For $M_n, g_n, \tilde g_n, z_n, \Delta_n$ given by Equations~\eqref{eq:Mn}, \eqref{eq:gn}, \eqref{eq:gtn}, \eqref{eq:zn}, \eqref{eq:Deltan},
\begin{align}
    &\norm{\frac{1}{M}\sum_{m = 1}^M Y_m \bm{\Sigma}^{-1/2}(\bm{X}_{I_m} - \bm{\mu})- \E[Y_m\bm{\Sigma}^{-1/2}(\bm X_{I_m} - \bm \mu)]} \nonumber \\
    &\leq \norm{\sum_{n=1}^N \left(\frac{M_n}{M}-\frac{1}{N}\right)\bm W_n \tilde g_n(\{\bm X_{n'}\}_{n' = 1}^N)} + \norm{\frac{1}{M}\sum_{n = 1}^N \bm W_n M_n \Delta_n(\{\bm X_{n'}\}_{n' = 1}^N)} \nonumber \\
    &+ \norm{\frac{1}{N}\sum_{n=1}^N \bm W_n z_n(\{\bm X_{n'}\}_{n' = 1}^N)} + \norm{\frac{1}{N}\sum_{n=1}^N\bm W_n g_n(\bm X_n)\! -\! \E[Y_m \bm{\Sigma}^{-1/2}(\bm X_{I_m}\! -\! \bm \mu)]}. \nonumber
\end{align}
\end{lemma}
The proof can be found in \ref{app:4terms}.
The four terms in the r.h.s. are bounded individually in the rest of the proof.
We bound the first term in Lemma \ref{lem:4terms} with Lemma~\ref{lem:binomial_gaussian_combined}.
\begin{lemma}
\label{lem:binomial_gaussian_combined}
For all $\delta_0 > d$, 
\begin{align*}
    &\Pr\bigg(\bnorm{\sum_{n=1}^N \left(\frac{M_n}{M} - \frac{1}{N}\right)\bm W_n \tilde g_{n}(\{\bm X_{n'}\}_{n' = 1}^N)}\! > \!\epsilon \bigg)\! \leq\! N\left( \frac{\delta_0}{d} e^{1 - \frac{\delta_0}{d}}\right)^{d/2} \!+\! 2^Ne^{-\frac{\epsilon^2M}{2\delta_0}}.
\end{align*}
\end{lemma}
The proof can be found in \ref{app:binomial_gaussian_combined}.
We rely on the fact that $|\tilde g_n| \leq 1$, as well as (a) the norm $\norm{\bm W_n}$ can be bounded by a centralized Chi-Squared tail bound,
while (b) the quantity $|\frac{M_n}{M} - \frac{1}{N}|$ can be bounded by the Bretagnolle-Huber-Carol Inequality (see Lemma \ref{lemma:huber_carol} in Section~\ref{app:technical}).
Next, we bound the second term in Lemma \ref{lem:4terms}.
\begin{lemma}
\label{lem:xnmnDeltan}
For all $\delta_1 < \frac{\epsilon^2}{4d}$ and $\delta_2 > d$,
\begin{align*}
    &\Pr\bigg(\bnorm{\frac{1}{M}\sum_{n = 1}^N \bm W_n M_n \Delta_n\{\bm X_{n'}\}_{n' = 1}^N}\! >\! \epsilon \bigg) \leq N\left(\frac{\epsilon^2}{4d\delta_1}e^{1-\frac{\epsilon^2}{4d\delta_1}}\right)^{d/2} \\
    &+ N\left(\frac{\delta_2}{d}e^{1-\frac{\delta_2}{d}}\right)^{d/2} + 2^{N}e^{-\frac{\epsilon^2M}{8\delta_1\delta_2}} + 2e^{\log N -\frac{M\delta_1}{2N} - o(\frac{M\delta_1}{2N})}.
\end{align*}
\end{lemma}
The proof is in \ref{app:xnmnDeltan}.
We bound individual terms, $\norm{\bm W_n}$, $|\Delta_n|$, $\big|\frac{M_n}{M} - \frac{1}{N}\big|$ respectively using a centralized Chi-Squared tail bound, Hoeffding's inequality, and the moment generating function of the binomial distribution.
Our next lemma bounds the third term in Lemma~\ref{lem:4terms}:
\begin{lemma}
\label{lem:xnzn}
For all $\delta_3 \leq \epsilon^2/d$,
\begin{align*}
    \Pr\bigg(\bnorm{\frac{1}{N}\sum_{n=1}^N \bm W_n z_n(\{\bm X_{n'}\}_{n' = 1}^N)} > \epsilon\bigg) \leq N\left( \frac{\epsilon^2}{d\delta_3} e^{1-\frac{\epsilon^2}{d\delta_3}}\right)^{d/2} + 2Ne^{-\frac{N\delta_3}{2}}.
\end{align*}
\end{lemma}
The proof can be found in \ref{app:xnzn}.
We bound terms $\norm{\bm W_n}$ and $|z_n|$ individually. 
For the former, we again use a centralized Chi-Squared tail bound.
For the latter, we indeed show that, for large sample sizes $N$, $\tilde{g}_n$ concentrates around $g_n$ using Hoeffding's inequality.
We bound the last term in Lemma~\ref{lem:4terms} as follows:
\begin{lemma}
\label{lem:xngn}
For an absolute constant $c_2 > 0$,
\begin{align*}
    \Pr\!\bigg(\!\bnorm{\frac{1}{N}\!\sum_{n=1}^N\!\bm W_n g_n(\bm X_n)\! - \! \E[Y_m \bm{\Sigma}^{-1/2}(\bm X_{I_m}\! -\! \bm\mu)]}\! >\! \epsilon\! \bigg) \! \leq\! e^{ -\frac{1}{4}\left(\sqrt{\frac{N\epsilon^2}{c_2}- d} - \sqrt{d} \right)^2}.
\end{align*}
\end{lemma}
The proof, which is in \ref{app:xngn}, shows that individual terms are sub-gaussian and uses a concentration bound due to \citet{hsu2012tail}.
The second term in Lemma~\ref{lem:2terms} is bounded as follows: 
\begin{lemma}
\label{lem:precision}
For the estimator $\bm{\hat \Sigma}$ given by Eq.~\eqref{eq:sigma_hat}, and for $\sqrt N \!>\! \sqrt{\frac{N-d-2}{\frac{\lambda_d}{ d\sqrt{2\lambda_1}}\epsilon + 1}}\! +\! c_3\!\sqrt d$ where $c_3, c_4 > 0$ are absolute constants,
\begin{align*}
    \Pr\left(\norm{\bm{\hat \Sigma}^{-1} - \bm{\Sigma}^{-1}}\cdot\norm{\E[Y_m (\bm{X}_{I_m} - \bm{\mu})]} > \epsilon\right) &\leq 2e^{-c_4\big(\sqrt N - \sqrt{\frac{N-d-2}{\frac{\lambda_d}{ d\sqrt{2\lambda_1}}\epsilon + 1}} - c_3\sqrt d \big)^2}.
\end{align*}
\end{lemma}
The proof can be found in \ref{app:precision}. 
We use a concentration bound on the minimum singular value of the design matrix due to \citet{vershynin2012}.
Combining Lemmas \ref{lem:2terms}, \ref{lem:4terms}, \ref{lem:binomial_gaussian_combined}, \ref{lem:xnmnDeltan}, \ref{lem:xnzn}, \ref{lem:xngn} and \ref{lem:precision} via a union bound gives:
\begin{align*}
    \Pr\left(\norm{\hat{\bm{\beta}} - c_1\bm{\beta}} \geq \epsilon \right) &\leq 8e^{-c_4\bigg(\sqrt N - \sqrt{\frac{N-d-2}{\frac{\lambda_d}{ 6d\sqrt{2\lambda_1}}\epsilon + 1}} - c_3\sqrt d \bigg)^2} + 4e^{-\frac{1}{4}\left(\sqrt{\frac{\epsilon^2N\lambda_{d}}{c_2}- d} - \sqrt{d} \right)^2} \nonumber \\
    &+ 2^{N+2}e^{-\frac{\epsilon^2M\lambda_{d}}{4608\delta_1\delta_2}} + 8e^{\log N -\frac{M\delta_1}{2N} - o\left(\frac{M\delta_1}{2N}\right)} + 8Ne^{-\frac{N\delta_3}{2}} \nonumber \\
    &+ 4N\left(\frac{\epsilon^2\lambda_{d}}{2304d\delta_1}e^{1-\frac{\epsilon^2\lambda_{d}}{2304d\delta_1}}\right)^{d/2} + 4N\left( \frac{\epsilon^2\lambda_{d}}{576d\delta_3} e^{1-\frac{\epsilon^2\lambda_{d}}{576d\delta_3}} \right)^{d/2} \nonumber\\
    &+ 2^{N+2}e^{-\frac{\epsilon^2M\lambda_{d}}{1152\delta_0}} + 4N\left( \frac{\delta_0}{d} e^{1 - \frac{\delta_0}{d}} \right)^{d/2} + 4N\left(\frac{\delta_2}{d}e^{1-\frac{\delta_2}{d}}\right)^{d/2}.
\end{align*}
Setting $M = \Omega(\frac{dN\log^3N}{\lambda_d\epsilon^2})$, $\delta_0 = d\log^2N$, $\delta_1 = 4\lambda_d\epsilon^2/d\log^2N$, $\delta_2 = d\log^2N$ and $\delta_3 = \epsilon^2\lambda_d/1152d\log N$, the bound reduces to $ \Pr\left(\norm{\hat{\bm{\beta}} - c_1\bm{\beta}} \geq \epsilon \right) \leq c_6N$ $\max\left\{\left(\frac{\sqrt{6\log N}}{N}\right)^d, e^{-\frac{N\epsilon^2\lambda_d}{c_7d\log N}} \right\}$ for $N > \frac{c_8d\log^2N}{\epsilon^2\lambda_d},$ where $c_6, c_7, c_8 > 0$ are absolute constants.
We derive this in \ref{app:choosing_delta}. \qed
\section{Experiments}
\label{sec:experiments}
\noindent\textbf{Synthetic Experiment Setup.}
\highlight{To support our theoretical findings, we evaluate\footnote{Code available online: \url{https://git.io/Jkbk1}} the estimator given by Eq.~\eqref{eq:estimator} with a synthetic dataset as follows:
We sample the true parameter $\bm \beta \in \R^d$ from $\mathcal{N}(\bm 0, 10\bm I)$.
We sample $\bm \mu\in \R^d$ uniformly at random from $[-5, 5]^{d}$.
To assess the impact of minimum eigenvalue of $\bm{\Sigma} \in \R^{d\times d}$ on estimator accuracy, we generate $\bm \Sigma$ as follows. 
We generate a random orthonormal basis of $\R^d$, and choose a smallest eigenvalue $\lambda_d\in (0,1]$. 
We then a construct $\bm \Sigma$ whose eigenvectors are the selected orthonormal basis, and $d$ eigenvalues equidistributed in $[\lambda_d, 1]$. 
We treat $\lambda_d$ as a tunable parameter.
Each feature vector $\bm{x}_i \in \R^{d}, i\in[2N]$, is independently sampled from $\mathcal{N}(\bm \mu, \bm \Sigma)$.
We sample pairs $(I_m, J_m), m \in [M]$, uniformly at random from $[N]\times[N]$.
Noisy labels $Y_m$ are sampled using Eq.~\eqref{eq:conditional} where $f(x) = (1+e^{-\alpha x})^{-1}$ and $0\! <\! \alpha\! <\! \infty$.
By adjusting $\alpha$, we choose the fraction $p_e \in [0, 1]$ of $M$ comparisons that are flipped, \emph{i.e.} are incorrect.
We repeat all experiments with 10 random generations of parameters $\bm \beta$, $\bm \mu$, $\bm \Sigma$.
We estimate quantities $c_1$ and $p_e$ numerically (see \ref{app:est_c_p}).
}


\noindent\textbf{Metrics.}
\highlight{
We measure the performance of the estimator $\hat{\bm{\beta}}$ with two metrics.
}
\highlight{
The first error metric is $\norm{\bm{\hat\beta} - c_1\bm\beta}$.
The second metric is 
\begin{align}
    \angle(\bm{\hat \beta}, \bm\beta) = \cos^{-1}{(\bm{\hat \beta}^\top\bm{\beta}/||\bm{\hat \beta}||||\bm{\beta}||)},
\end{align}
\emph{i.e.}, the angle between $\bm{\hat\beta}$ and $\bm\beta$.
We report both the average and standard deviation across different
random generations.
}

\begin{figure}[t]
\centering
\begin{subfigure}[b]{0.48\textwidth}
   \includegraphics[width=1\linewidth]{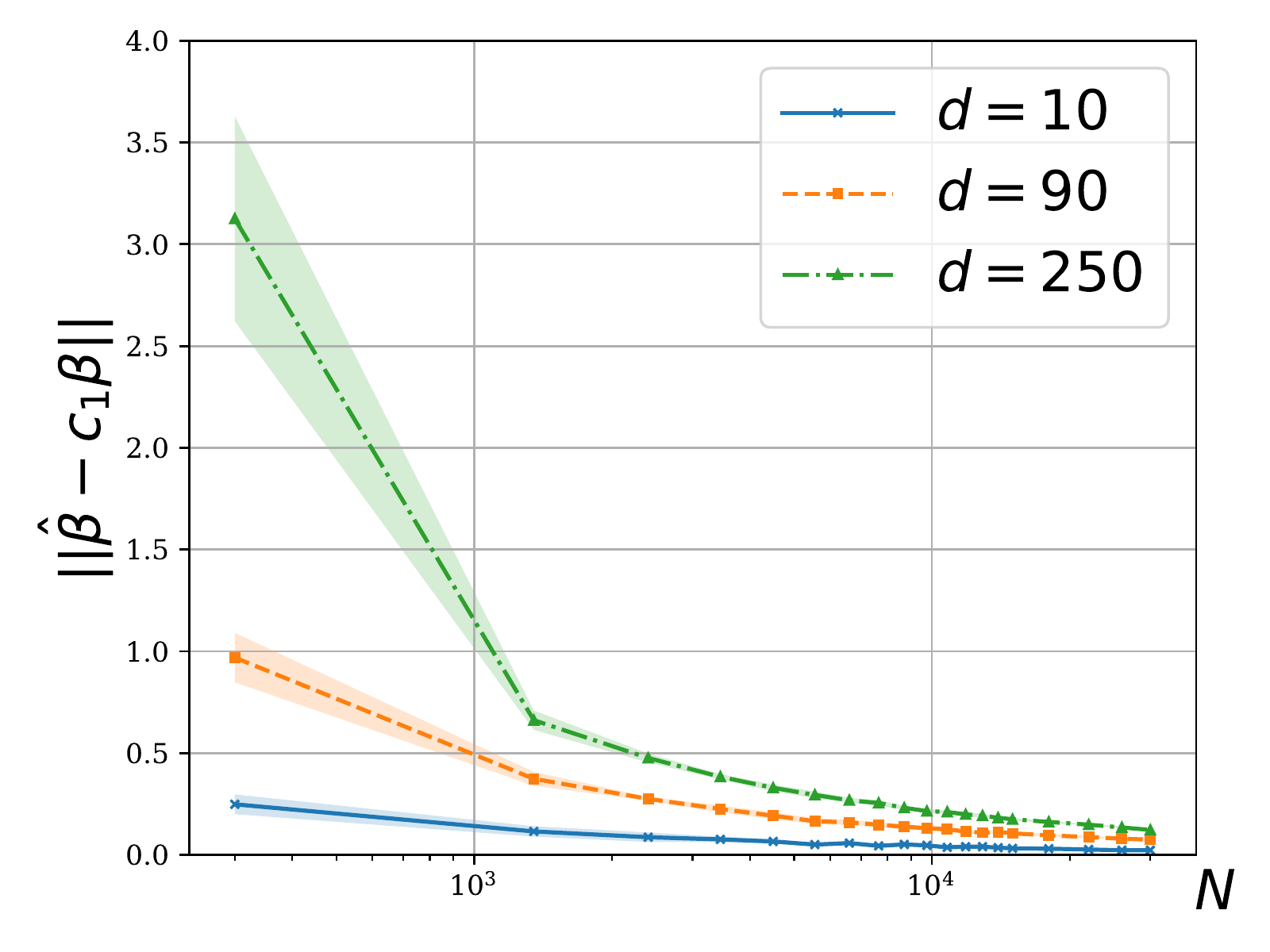}
   \caption{$\lambda_d = 1, p_e=0.2$}
   \label{fig:mbN-norm} 
\end{subfigure}
\begin{subfigure}[b]{0.48\textwidth}
   \includegraphics[width=1\linewidth]{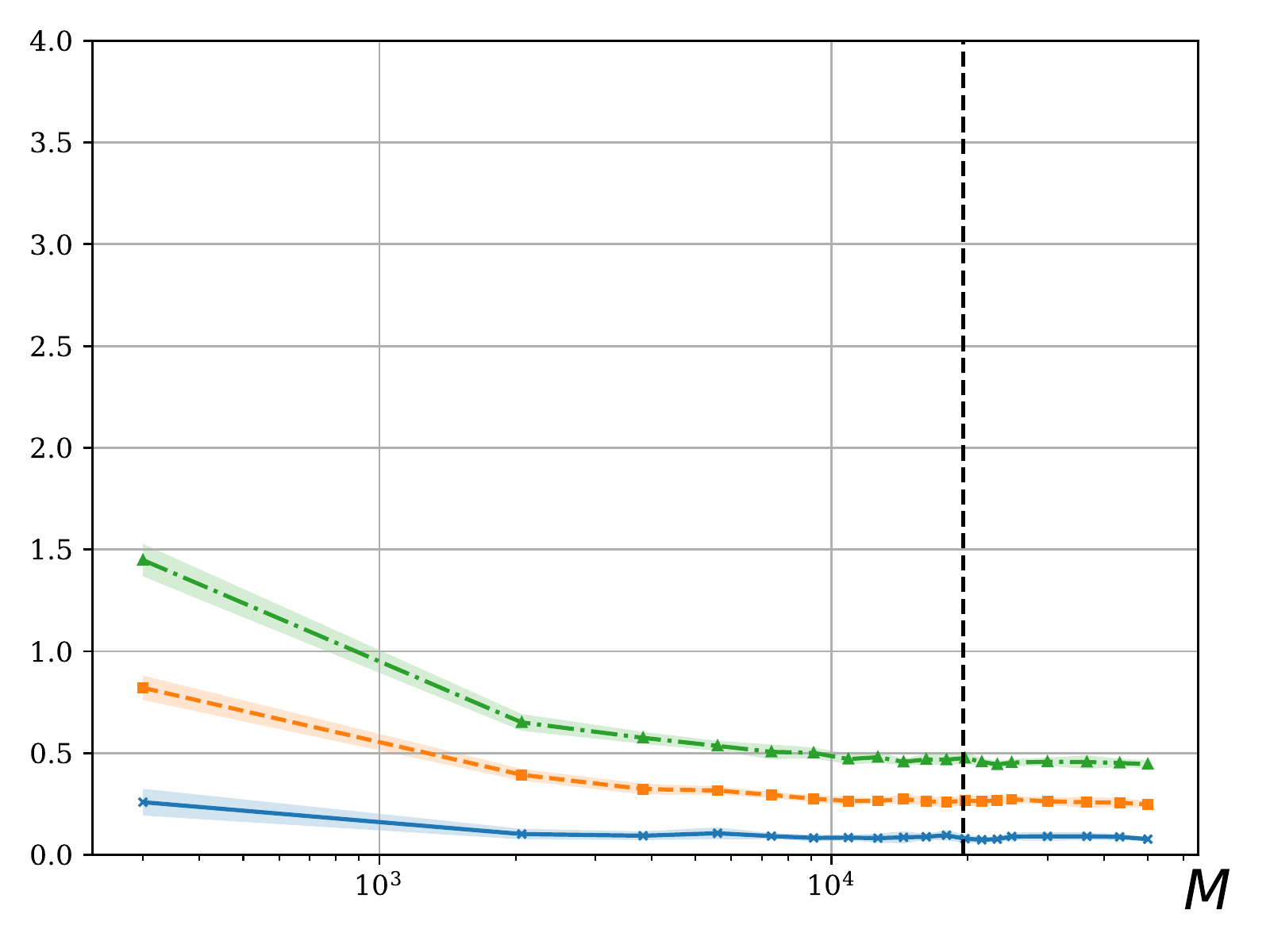}
   \caption{$\lambda_d = 1, p_e=0.2$}
   \label{fig:mbM-norm}
\end{subfigure}
\caption[Convergence of the estimator and real data.]{(a) The error of the estimator given by Eq.~\eqref{eq:estimator} indeed reduces as $N$ increases when $M=\lceil N\log N \rceil$ and we see that the estimator is converging to $c_1\bm\beta$. (b) The error reduces when $M$ increases while $N$ is kept constant, however the decay is insignificant after $M=N\log N$, which is denoted with the black dashed line. This agrees with our theory that $M = \tilde\Omega(N)$. The shaded area is the standard
deviation.}
\end{figure}
\begin{figure}[t]
\centering
\begin{subfigure}[b]{0.31\textwidth}
   \includegraphics[width=1\linewidth]{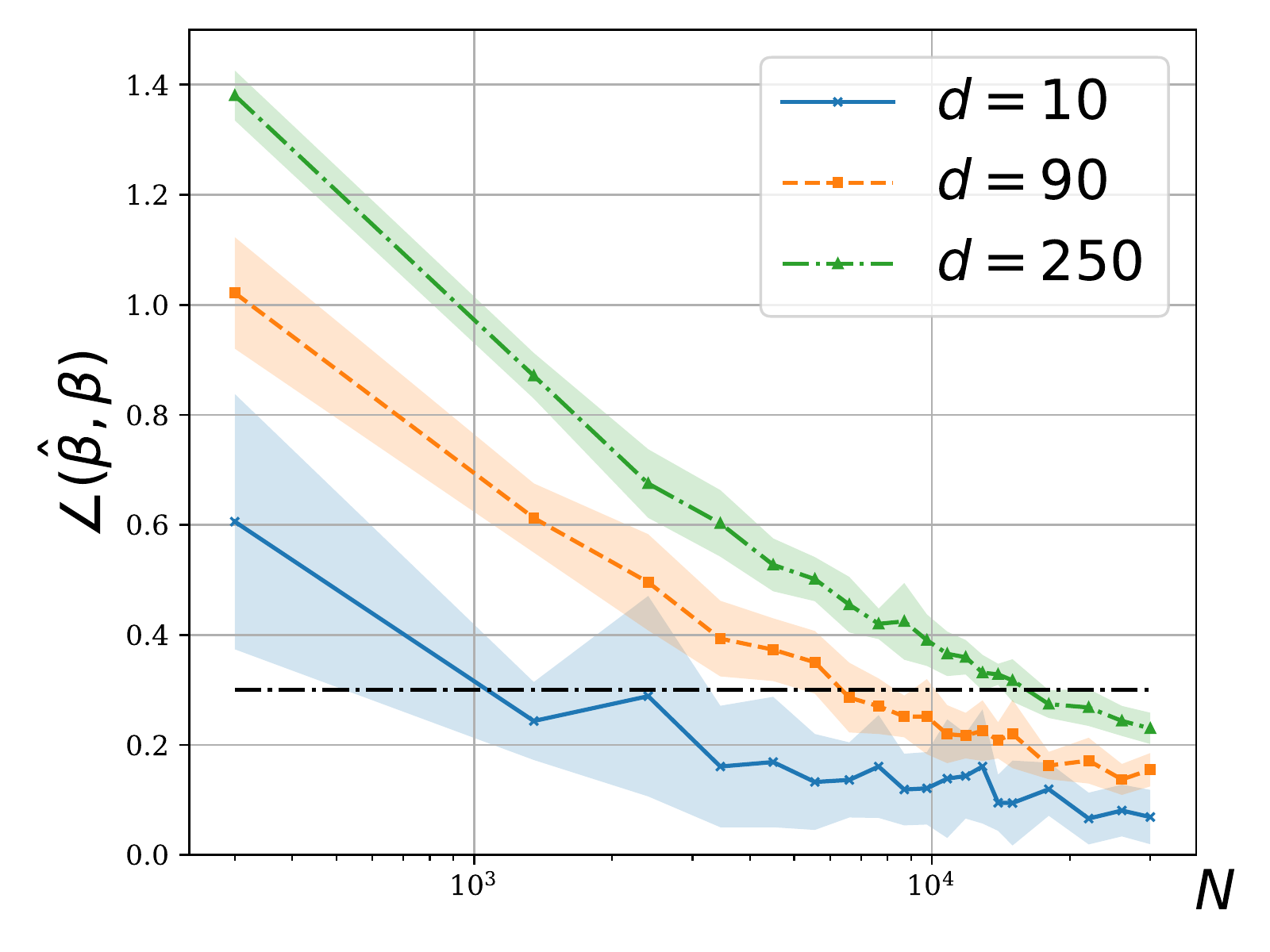}
   \caption{$p_e=0$}
   \label{fig:mbN1} 
\end{subfigure}
\begin{subfigure}[b]{0.31\textwidth}
   \includegraphics[width=1\linewidth]{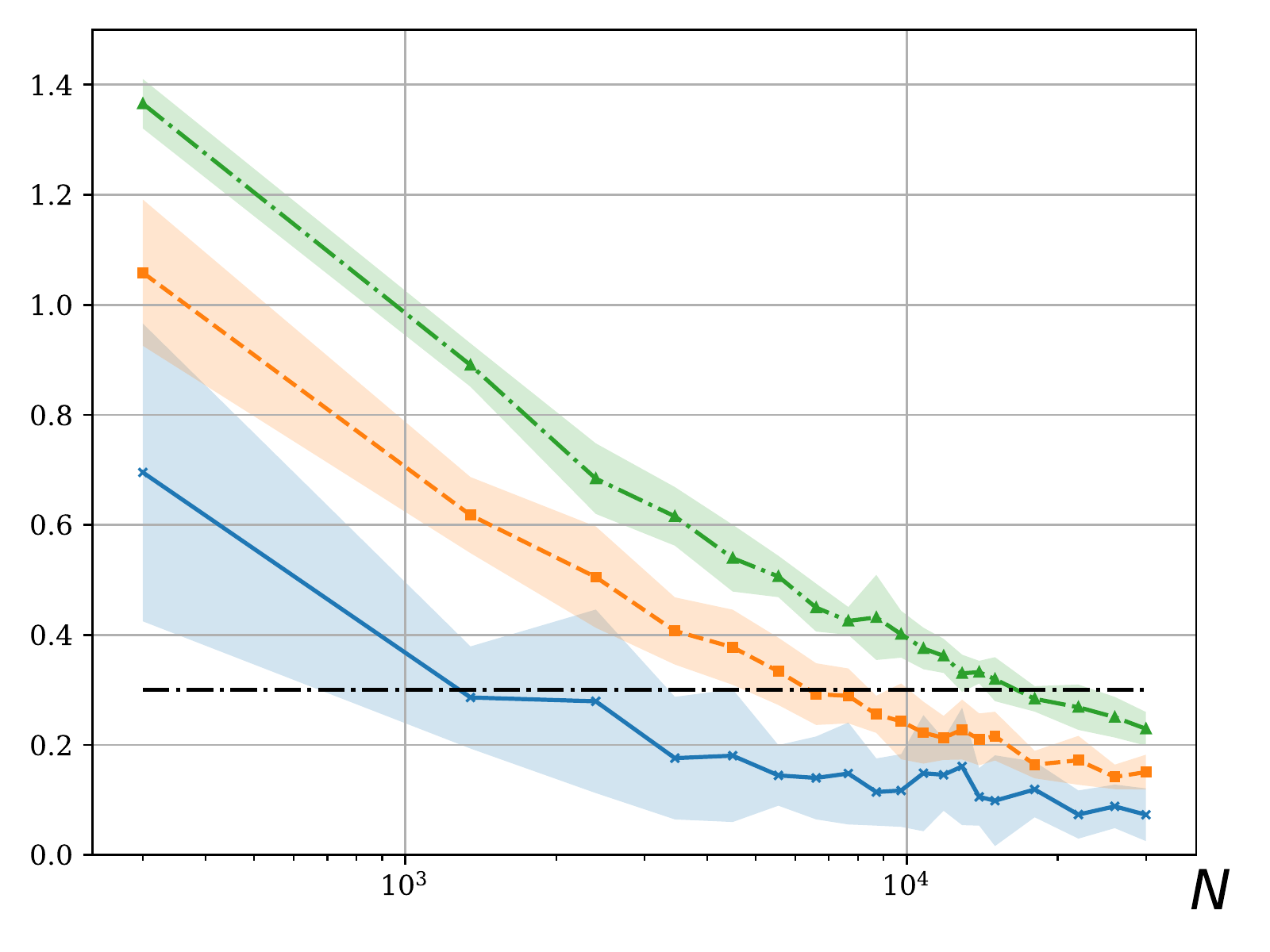}
   \caption{$p_e=0.2$}
   \label{fig:mbN2}
\end{subfigure}
\begin{subfigure}[b]{0.31\textwidth}
   \includegraphics[width=1\linewidth]{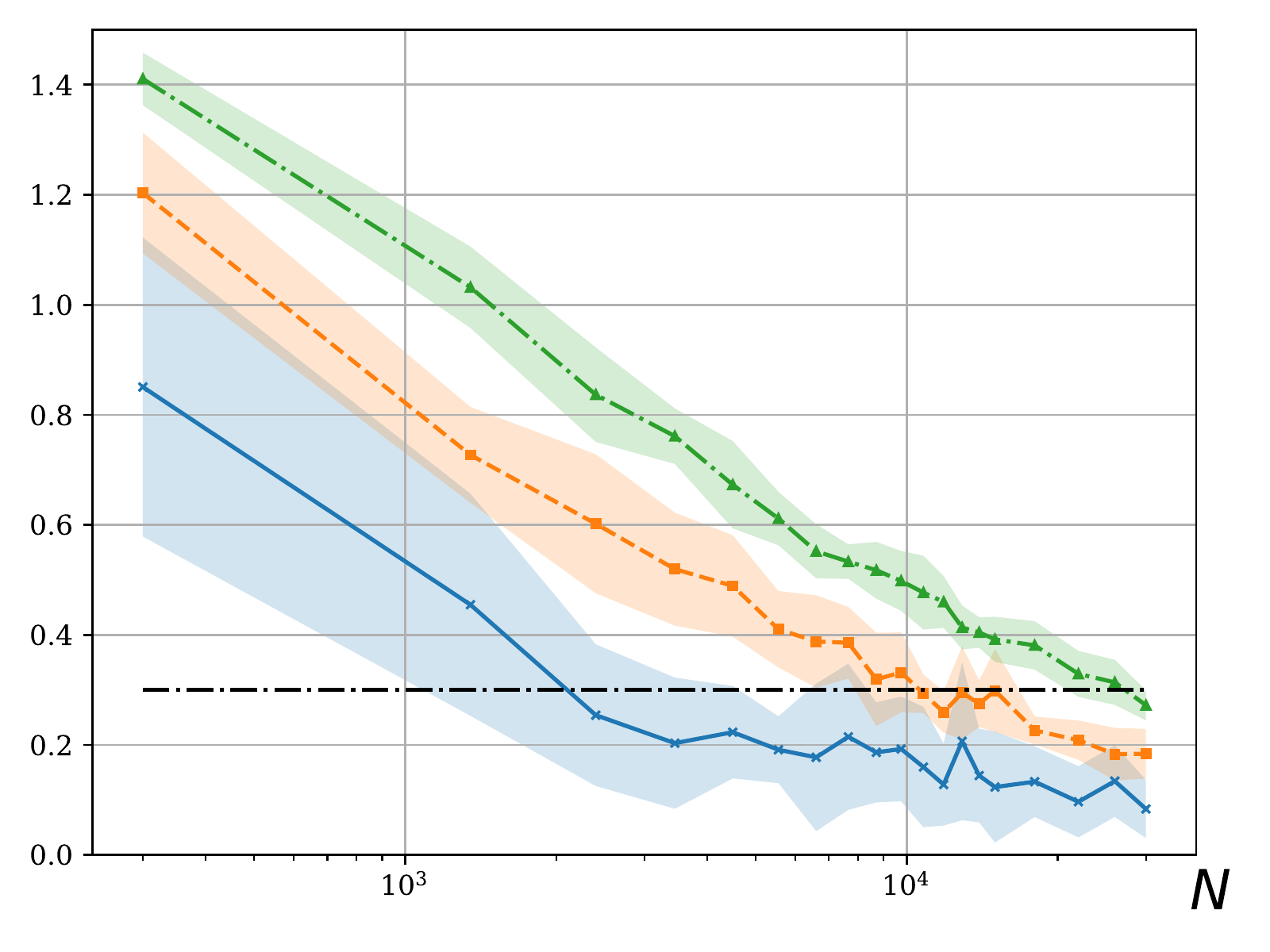}
   \caption{$p_e=0.4$}
   \label{fig:mbN3}
\end{subfigure}
\caption[Dependence on $N$.]{The angle between the true parameter $\bm\beta$ and the estimator $\bm{\hat\beta}$, plotted against $N$ for different error probabilities $p_e$ when $M=\lceil N\log N \rceil$ and $\lambda_d = 1/200$. (a) The noiseless case. (b) The case where $20\%$ of the labels are flipped. (c) The case where $40\%$ of the labels are flipped. Even though noise increases error, increasing $N$ allows the estimator to reduce the error arbitrarily. This shows that the estimator $\bm{\hat \beta}$ is able to recover the direction of the true parameter $\bm\beta$ as $N$ increases. The shaded area is the standard deviation.}
\label{fig:mbN}
\end{figure}
\begin{figure}[t]
\centering
\begin{subfigure}[b]{0.31\textwidth}
   \includegraphics[width=1\linewidth]{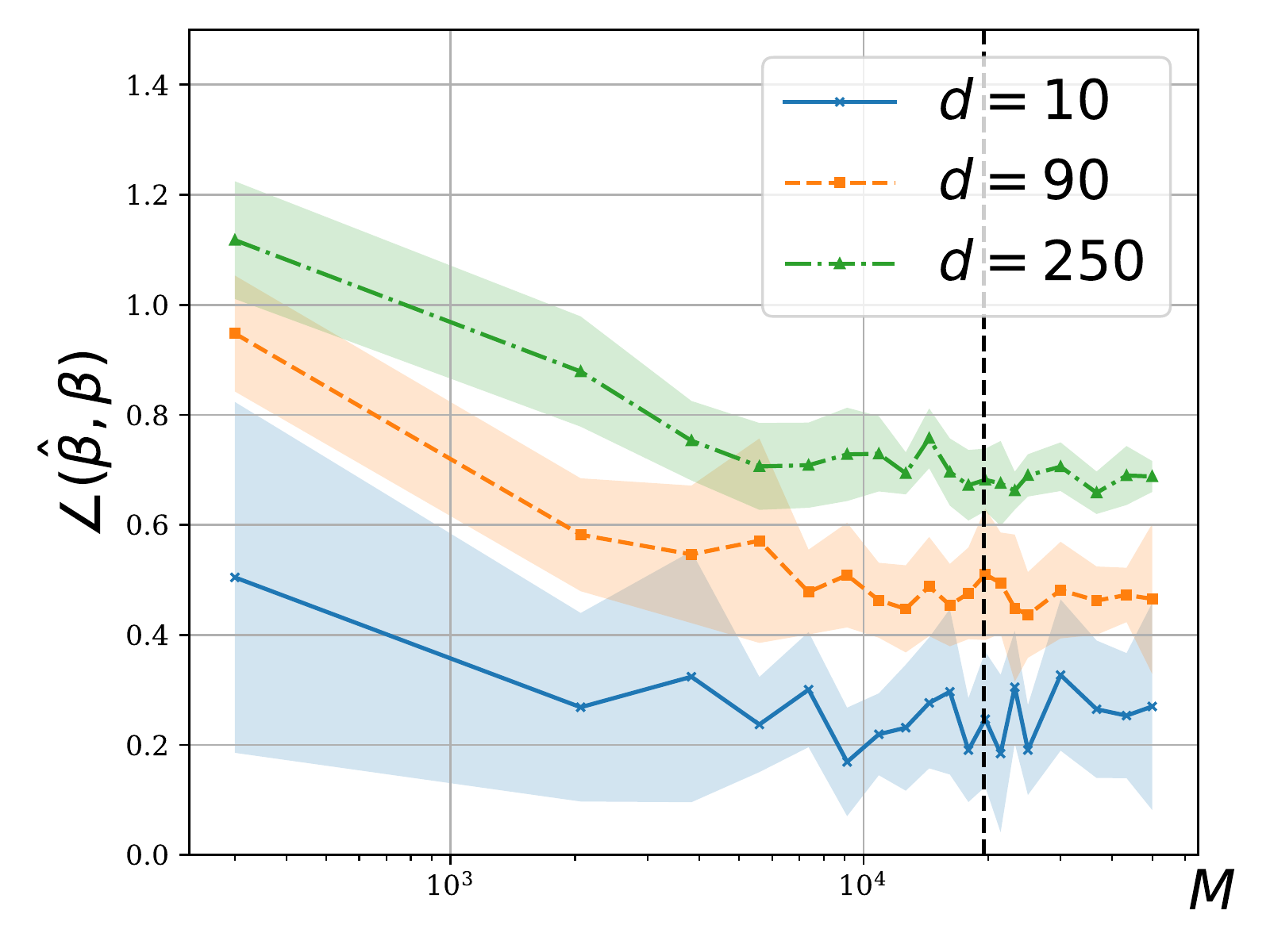}
  \caption{$p_e=0$}
   \label{fig:mbM1} 
\end{subfigure}
\begin{subfigure}[b]{0.31\textwidth}
   \includegraphics[width=1\linewidth]{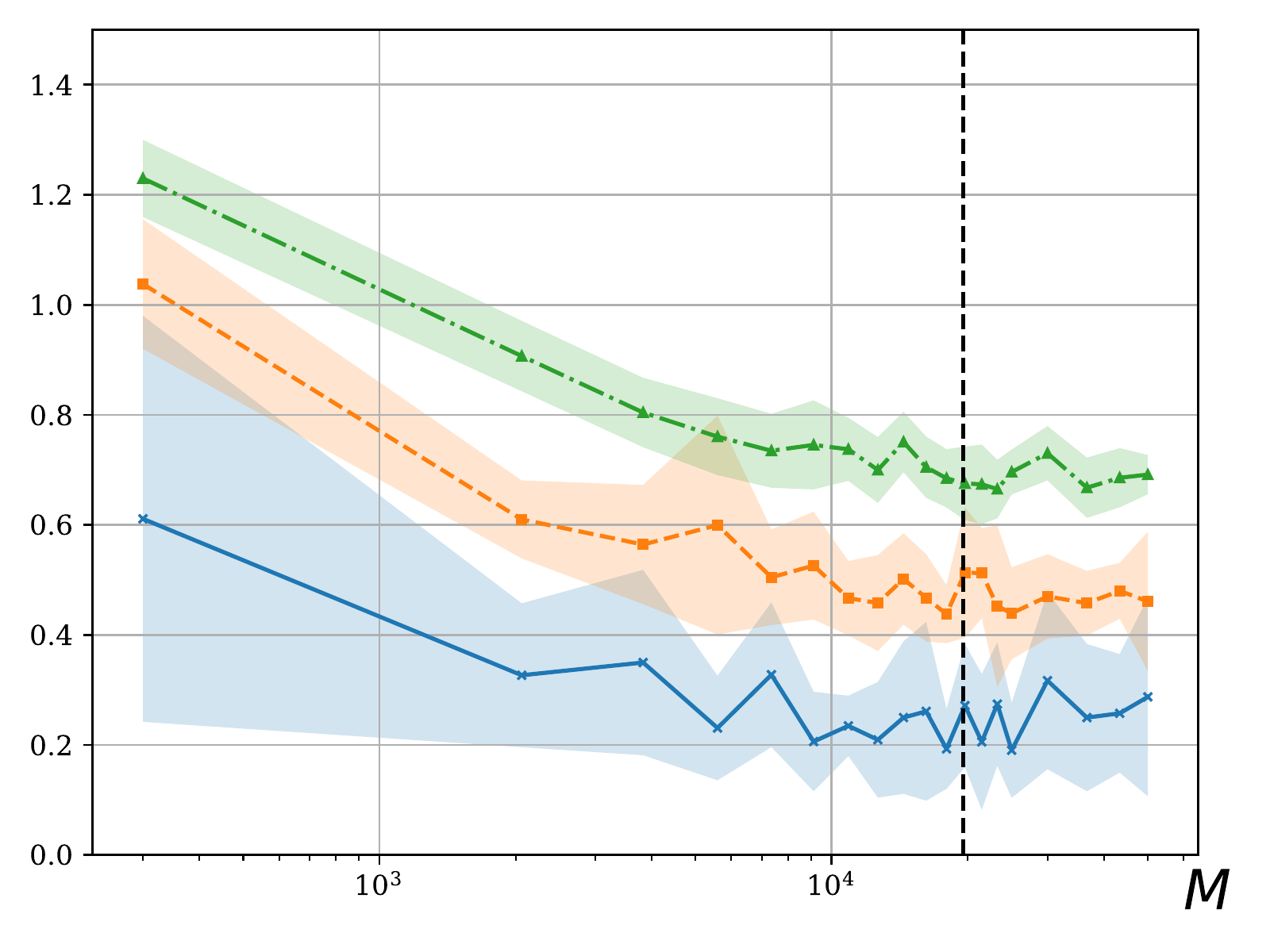}
   \caption{$p_e=0.2$}
   \label{fig:mbM2}
\end{subfigure}
\begin{subfigure}[b]{0.31\textwidth}
   \includegraphics[width=1\linewidth]{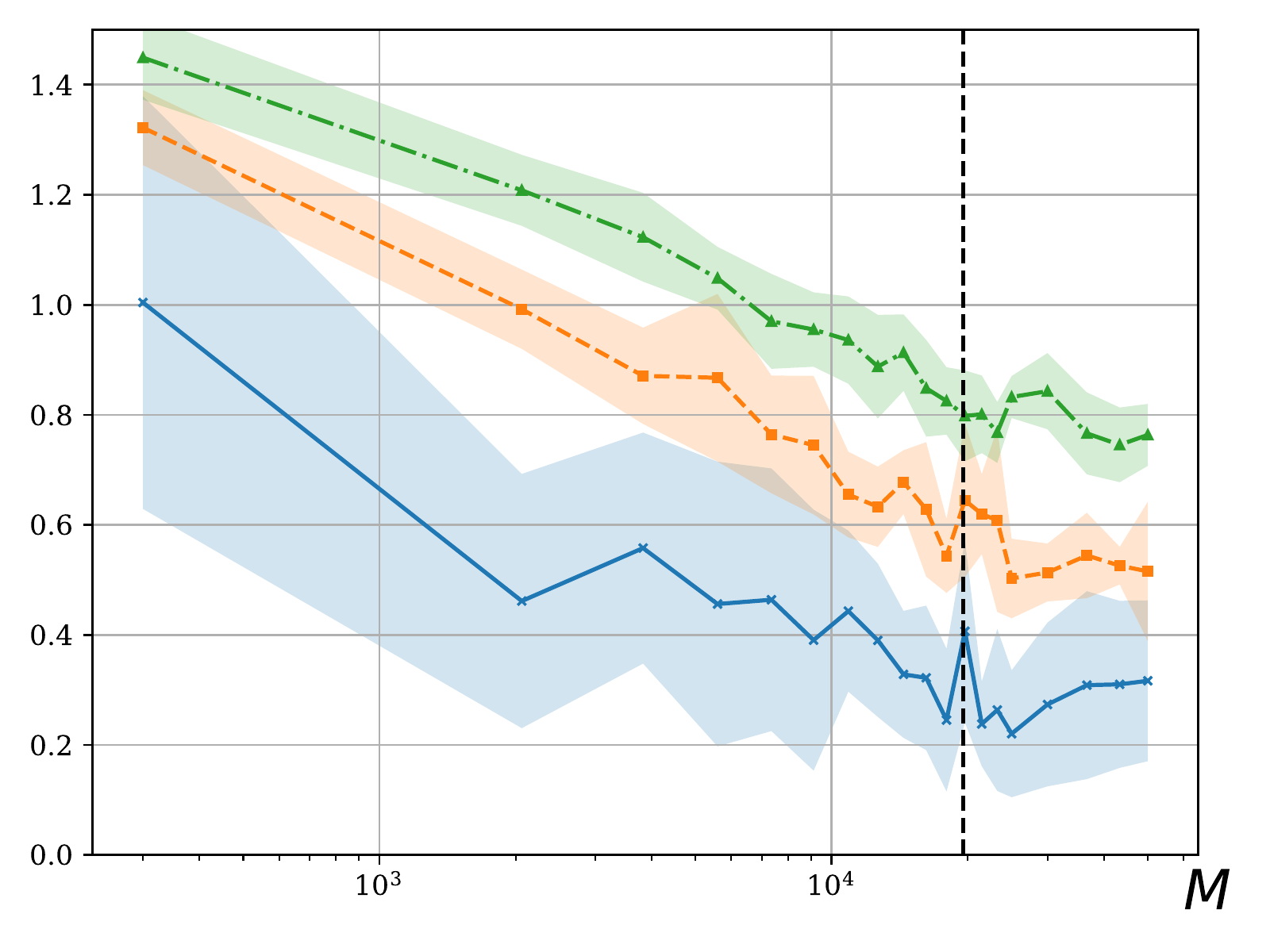}
   \caption{$p_e=0.4$}
   \label{fig:mbM3}
\end{subfigure}
\caption[Dependence on $M$.]{The angle between the true parameter $\bm\beta$ and the estimator $\bm{\hat\beta}$, plotted against $M$ for different error probabilities $p_e$ when $N=2.5\times10^3$ and $\lambda_d = 1/200$. (a) The noiseless case. (b) The case where $20\%$ of the labels are flipped. (c) The case where $40\%$ of the labels are flipped. Increasing $M$ reduces the error for all $p_e$ values. However, once $M=N\log N$ which is denoted with the black dashed line, the reduction is insignificant. This is due to the fact that a higher $N$ is required for smaller $\epsilon$ as $N = \Omega(d/\lambda_d\epsilon^2)$, \emph{i.e.} $N$ needs to scale inversely quadratic with $\epsilon$. This supports our theory that the estimator $\bm{\hat \beta}$ does better as $M$ increases; however, for arbitrarily small $\epsilon$, $N$ needs to increase too. The shaded area is the standard deviation.}
\label{fig:mbM}
\end{figure}
\begin{figure}[t]
\centering
\begin{subfigure}[b]{0.31\textwidth}
   \includegraphics[width=1\linewidth]{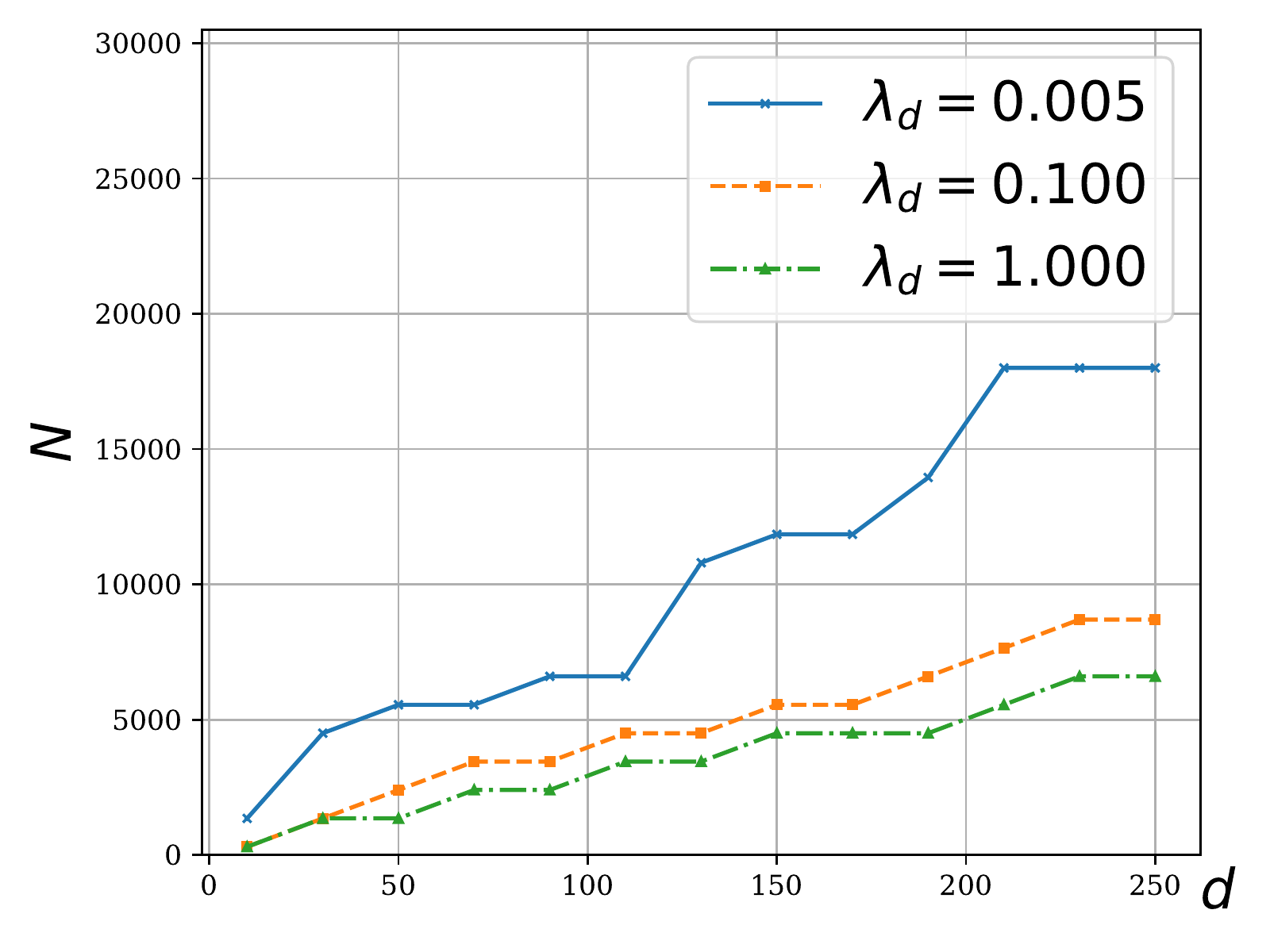}
   \caption{$p_e=0$}
   \label{fig:dbNl1}
\end{subfigure}
\begin{subfigure}[b]{0.31\textwidth}
   \includegraphics[width=1\linewidth]{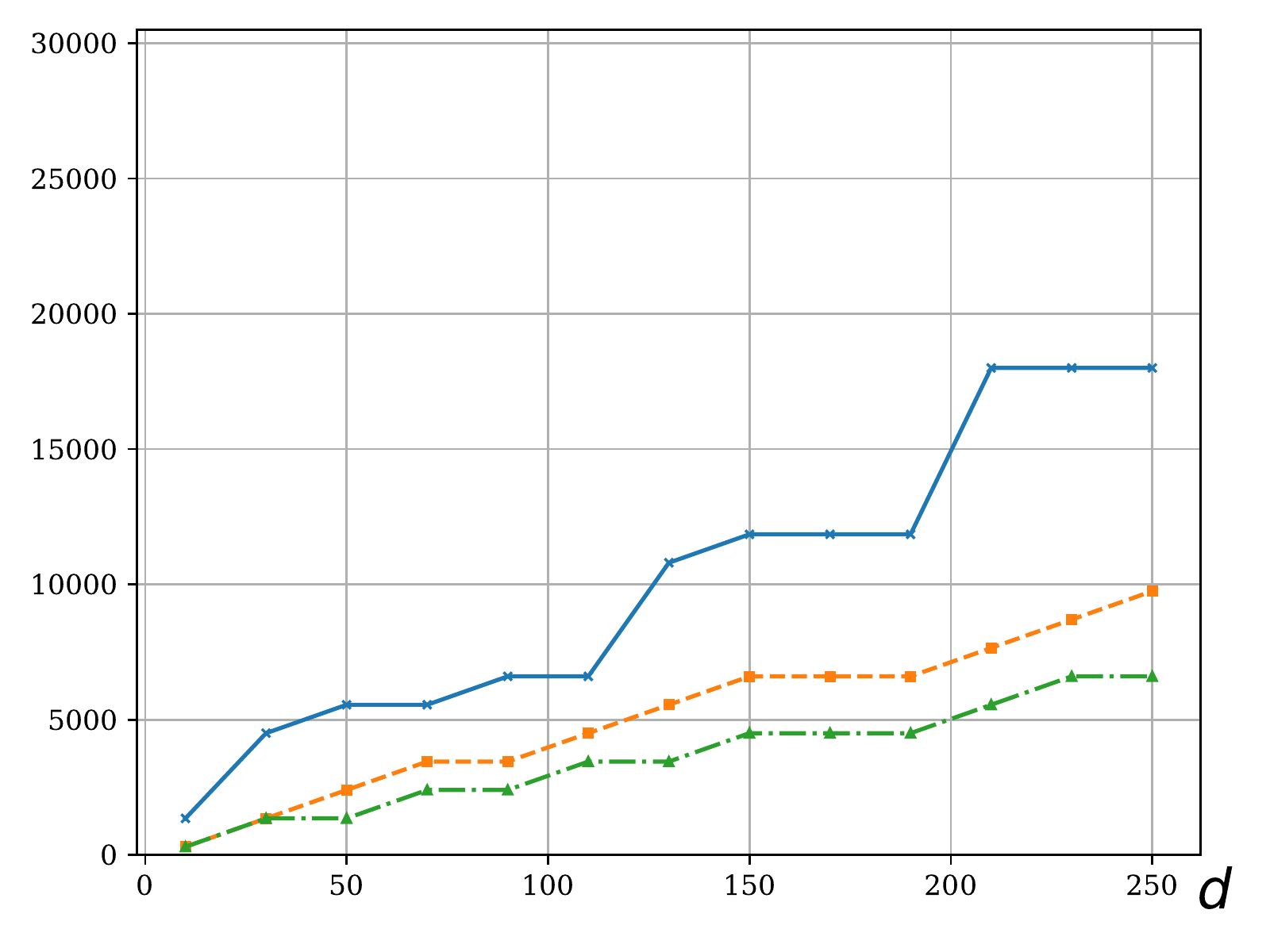}
   \caption{$p_e=0.2$}
   \label{fig:dbNl2}
\end{subfigure}
\begin{subfigure}[b]{0.31\textwidth}
   \includegraphics[width=1\linewidth]{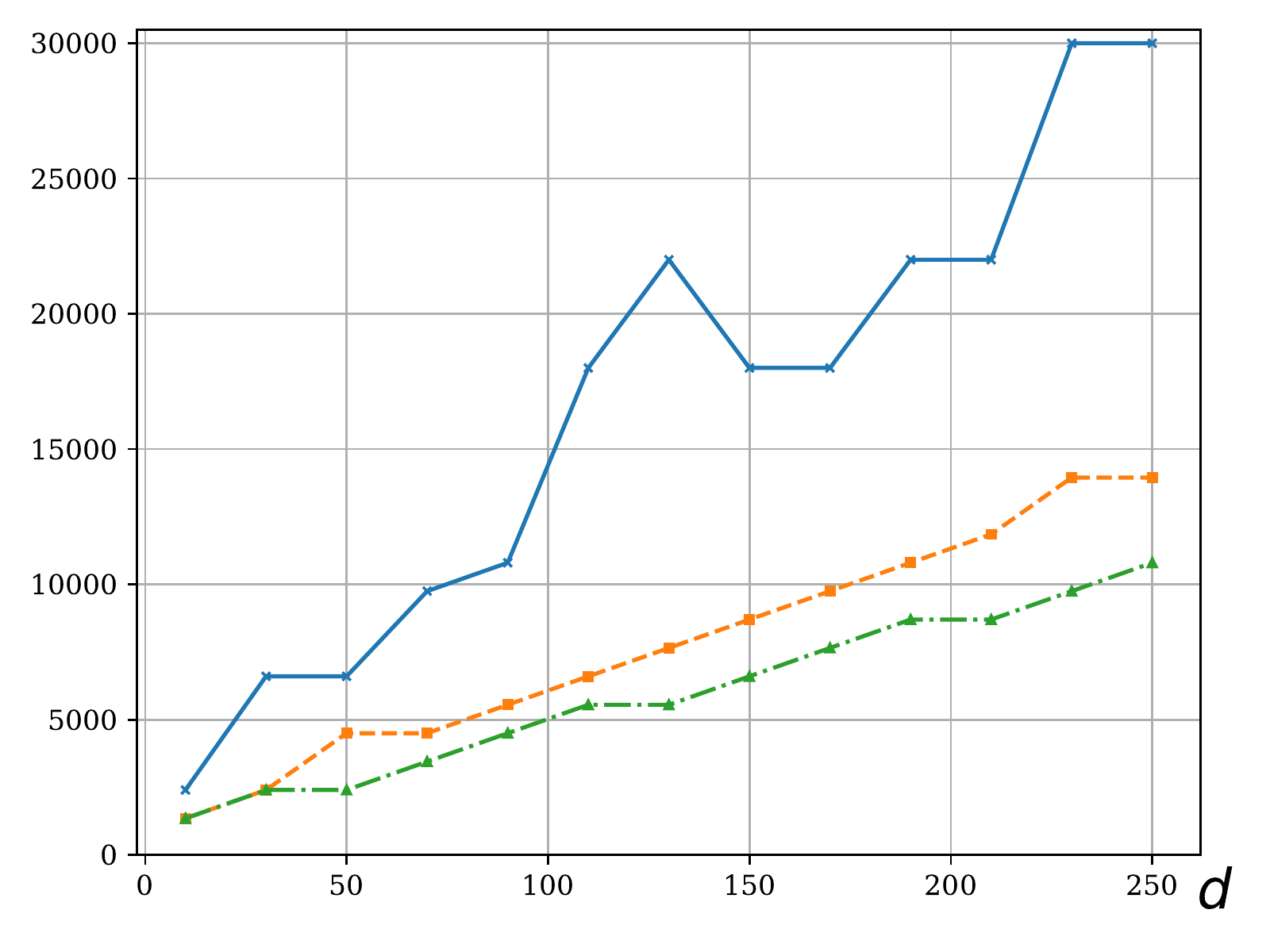}
   \caption{$p_e=0.4$}
   \label{fig:dbNl3}
\end{subfigure}
\caption[Dependence on $d$.]{Minimum $N$ that achieves $\angle(\bm{\hat \beta}, \bm\beta) \leq 0.3$ plotted against dimensionality $d$ for different values of the minimum eigenvalue $\lambda_d$ of the feature covariance and error probabilities $p_e$. (a) The noiseless case. (b) The case where $20\%$ of the labels are flipped. (c) The case where $40\%$ of the labels are flipped. As the probability of error and the condition number of the feature covariance increases, we require more samples. Crucially, we observe the linear dependence of $N$ to $d$ and this supports that $N=\Omega(d/\lambda_d\epsilon^2)$.}
\label{fig:Nbd}
\end{figure}
\begin{figure}[t]
\centering
\begin{subfigure}[b]{0.31\textwidth}
   \includegraphics[width=1\linewidth]{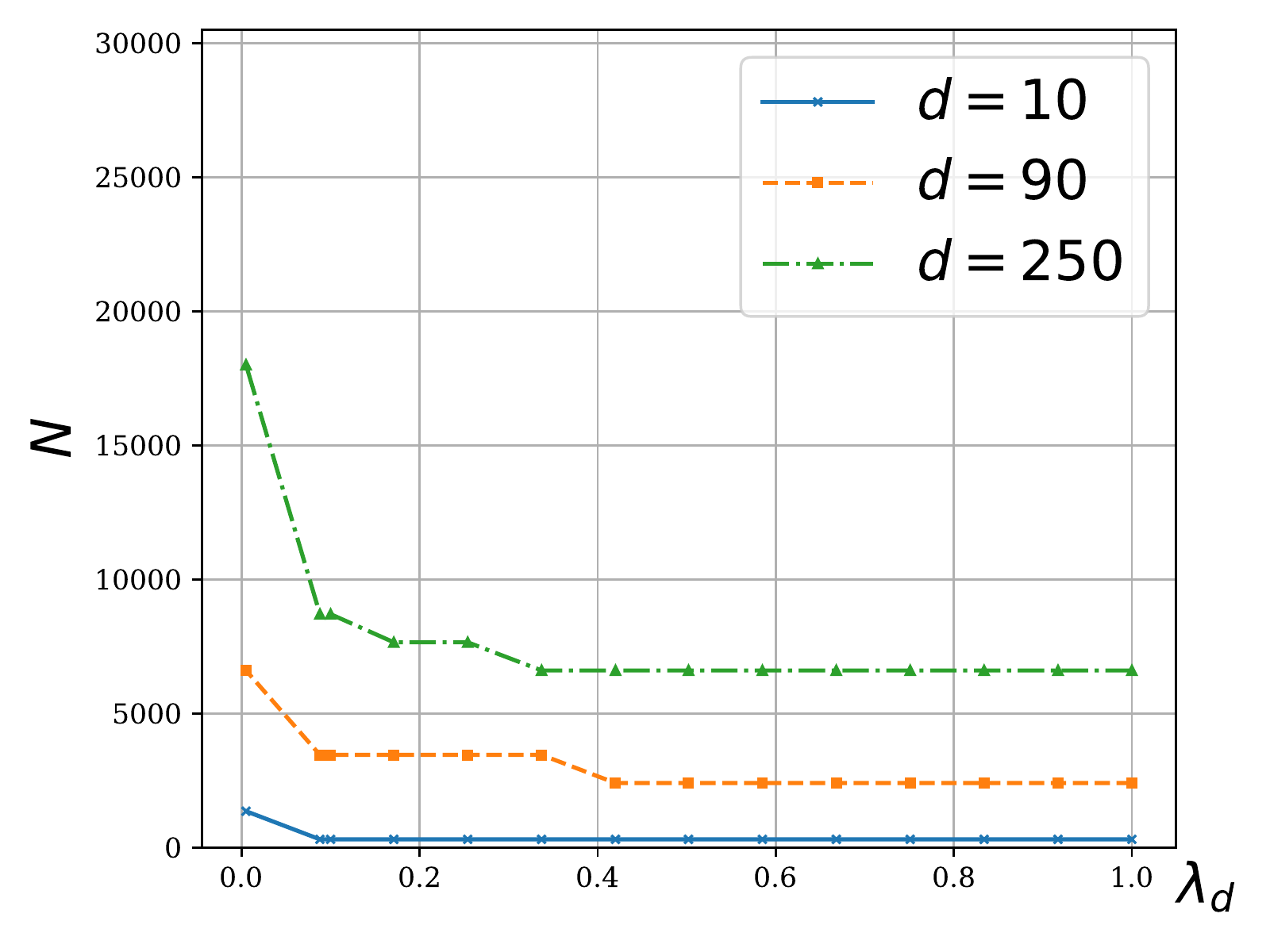}
   \caption{$p_e=0$}
   \label{fig:1}
\end{subfigure}
\begin{subfigure}[b]{0.31\textwidth}
   \includegraphics[width=1\linewidth]{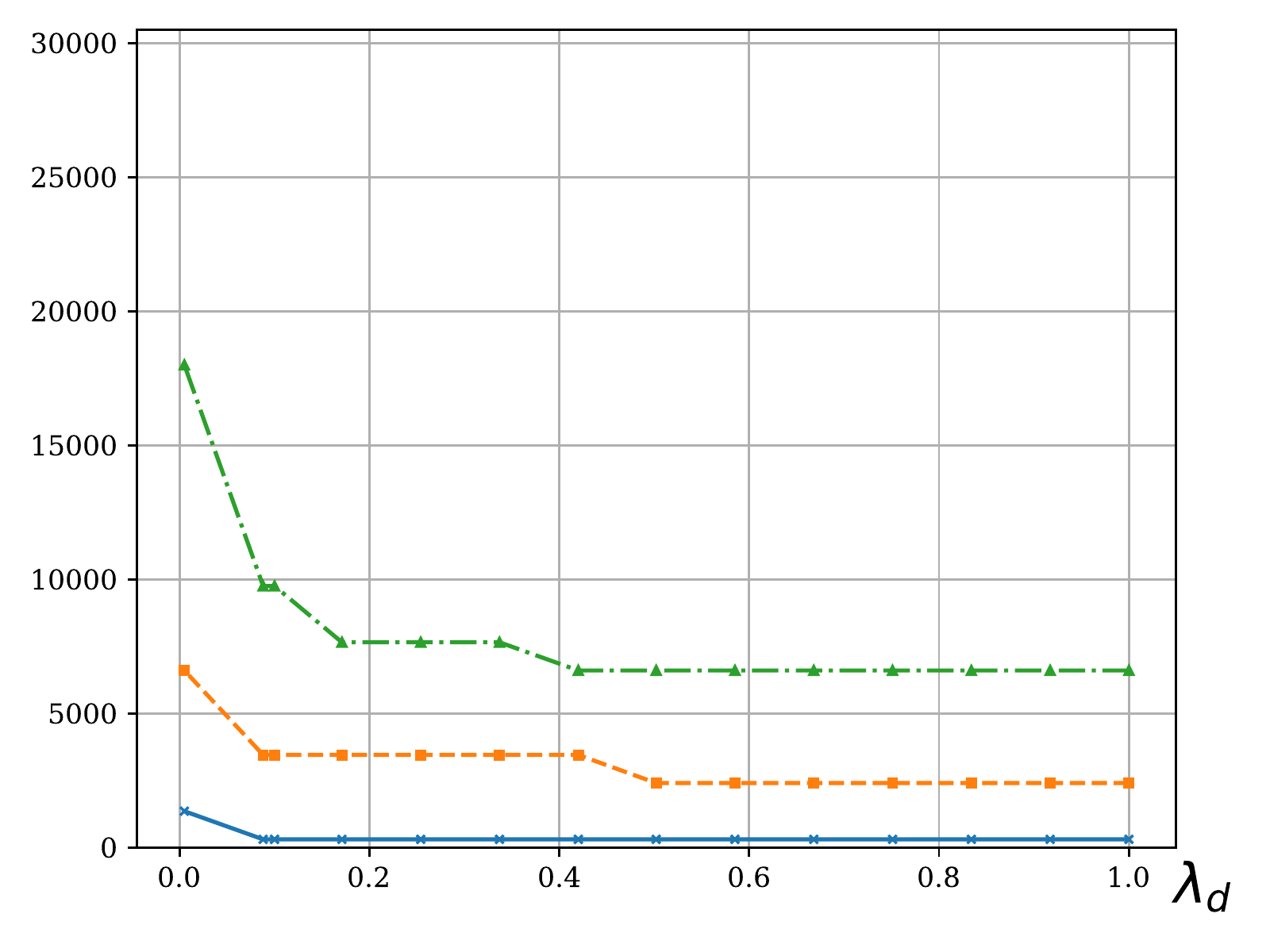}
   \caption{$p_e=0.2$}
   \label{fig:2}
\end{subfigure}
\begin{subfigure}[b]{0.31\textwidth}
   \includegraphics[width=1\linewidth]{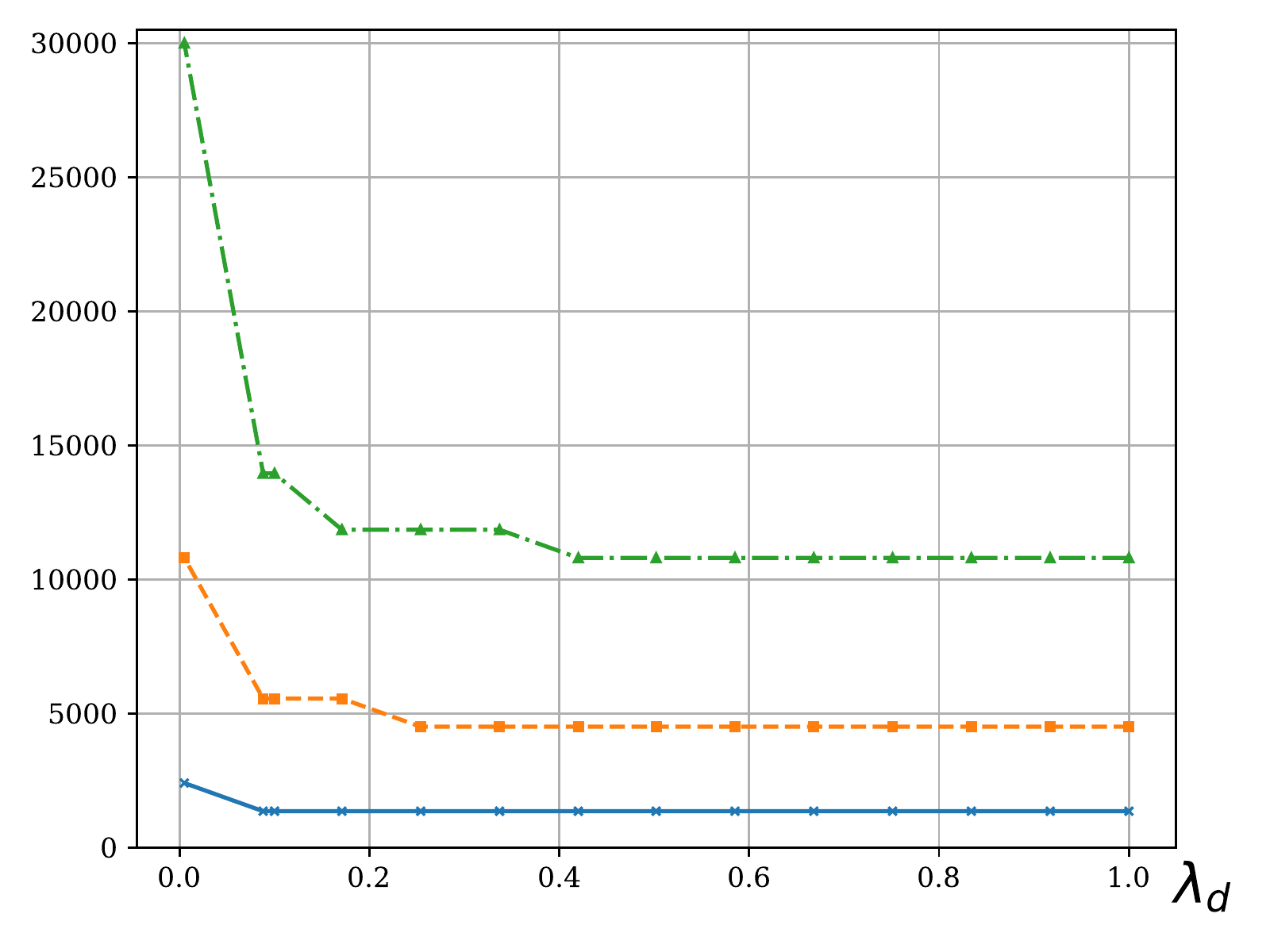}
   \caption{$p_e=0.4$}
   \label{fig:3}
\end{subfigure}
\caption[Dependence on $\lambda_d$.]{Minimum $N$ that achieves $\angle(\bm{\hat \beta}, \bm\beta) \leq 0.3$ plotted against the minimum eigenvalue $\lambda_d$ of the feature covariance for different dimensionality $d$ and probability of error $p_e$. (a) The noiseless case. (b) The case where $20\%$ of the labels are flipped. (c) The case where $40\%$ of the labels are flipped. We observe that increasing error increases the required $N$ and the inversely proportional dependence of $N$ to $\lambda_d$ is observed.}
\label{fig:Nbld}
\end{figure}

\noindent\textbf{Convergence.}
\highlight{
In order to investigate the convergence of $\bm{\hat\beta}$, we vary the number of samples $N$ in $[300, 3\times10^4]$, dimensionality $d\in\{10, 90, 250\}$ while we set $M=\lceil N\log N\rceil$, and $\lambda_d=1$. 
We select $\alpha$ so that $p_e = 0.2$.
In Fig.~\ref{fig:mbN-norm}, we plot the error $\norm{\bm{\hat\beta} - c_1\bm\beta}$ as a function of the dataset size $N$. 
We observe that for each $d$, the error decreases as $N$ increases and $\bm{\hat\beta}$ indeed converges to $c_1\bm\beta$.
In Fig.~\ref{fig:mbM-norm}, we vary $M$ in $[300, 5\times10^4]$ and $d\in\{10, 90, 250\}$ while we set $N=2.5\times10^3$.
We observe that increasing $M$ reduces the error.
However, the reduction in error is insignificant after $M=N\log N$, which is denoted with the black dashed line in Fig.~\ref{fig:mbM-norm}. 
This is consistent with the bound in Theorem~\ref{th:bound}, which anticipates that $M$ is polylogarithmic in $N$.
}

\noindent\textbf{Dependence on $N$.}
\highlight{
We investigate the required number of samples $N$ to attain $\angle(\hat{\bm\beta}, \bm\beta) = 0.3$.
For each $d\in\{10, 90, 250\}$ and $p_e\in\{0, 0.2, 0.4\}$ we vary $N \in [300, 3\times10^4]$, setting $M=\lceil{N\log N}\rceil$, and $\lambda_d = 1/200$.
In Fig.~\ref{fig:mbN}, we plot the error $\angle(\bm{\hat \beta}, \bm\beta)$ versus the dataset size $N$ under different noise levels $p_e$.
We observe that as $N$ increases, $\bm{\hat\beta}$ indeed achieves $\angle(\hat{\bm\beta}, \bm\beta) = 0.3$ for all $d$, while the error increase with $d$.
This implies that irrespective of the noise level and the corresponding $c_1$ value, the estimator $\bm{\hat \beta}$ is able to recover the direction of $\bm\beta$ as $N$ increases.
}

\noindent\textbf{Dependence on $M$.}
\highlight{
We repeat our experiments on the impact of $M$, this time focusing on the angle metric and varying $p_e$. 
We fix $N=2.5\times10^3$, $\lambda_d = 1/200$ and vary $M\in [300, 5\times 10^4]$, $d \in \{10, 90, 250\}$, and $p_e\in\{0, 0.2, 0.4\}$.
Fig.~\ref{fig:mbM} plots $\angle(\bm{\hat \beta}, \bm\beta)$ versus $M$ under different noise levels.
These plots show that the benefit of increasing $M$ again diminishes beyond $M=N\log N$ for all $p_e\in\{0, 0.2, 0.4\}$, represented by the dashed black line. 
This is again consistent with Theorem~\ref{th:bound}. 
}

\noindent\textbf{Dependence on $d$.}
\highlight{
In Fig.~\ref{fig:Nbd}, we plot the smallest $N$ that achieves $\angle(\hat{\bm\beta}, \bm\beta) \leq 0.3$ while we vary $d \in [10, 250]$, $\lambda_d \in \{1/200, 0.1, 1\}$ and $p_e \in \{0, 0.2, 0.4\}$. 
We observe that the required $N$ increases linearly in $d$. 
This is consistent with the linear dependence of $N$ on $d$ anticipated by Theorem~\ref{th:bound}.
}

\noindent\textbf{Dependence on $\lambda_d$.}
\highlight{
In Fig.~\ref{fig:Nbld}, we investigate the effect of $\lambda_d$ on the smallest $N$ that achieves $\angle(\hat{\bm\beta}, \bm\beta) \leq 0.3$ for each $p_e \in \{0, 0.2, 0.4\}$ and $d \in \{10, 90, 250\}$.
We observe the inversely proportional relationship between $\lambda_d$ and $N$, which is consistent with the $N = \Omega(d/\lambda_d\epsilon^2)$ requirement implied by Theorem~\ref{th:bound}.
}


\section{Conclusion}
\label{sec:conc}
Our results suggest that learning \highlight{parameters of a linear preference model} from comparisons can come with guarantees, despite the lack of independence between comparisons. 
Though our bound on $N$ is tight (as the number of samples cannot be lower than $d$), our experimental results suggest that the bound on $M$ could be sharpened; lower bounds on this quantity also remain open. 
Given the practical benefit of learning from comparisons over datasets with small samples, understanding the inherent trade-offs between samples, comparisons, and label variance is a very interesting direction to explore. 
In particular, determining regimes in which learning from comparisons outperforms learning from categorical labels is an important question; our work can serve as a starting point for exploring this formally.
\section*{Acknowledgements}
Our work is supported by NIH (R01EY019474), NSF (SCH-1622542 at MGH; SCH-1622536 at Northeastern; SCH-1622679 at OHSU), and by unrestricted departmental funding from Research to Prevent Blindness (OHSU).
\section*{References}
\bibliographystyle{elsarticle-num-names}
\bibliography{0_main}

\begin{thebibliography}{69}
\providecommand{\natexlab}[1]{#1}
\providecommand{\url}[1]{\texttt{#1}}
\providecommand{\urlprefix}{URL }
\expandafter\ifx\csname urlstyle\endcsname\relax
  \providecommand{\doi}[1]{doi:\discretionary{}{}{}#1}\else
  \providecommand{\doi}[1]{doi:\discretionary{}{}{}\begingroup
  \urlstyle{rm}\url{#1}\endgroup}\fi
\providecommand{\bibinfo}[2]{#2}

\bibitem[{Schultz and Joachims(2004)}]{schultz2004learning}
\bibinfo{author}{M.~Schultz}, \bibinfo{author}{T.~Joachims},
  \bibinfo{title}{Learning a distance metric from relative comparisons}, in:
  \bibinfo{booktitle}{Advances in Neural Information Processing Systems
  (NeurIPS)}, \bibinfo{pages}{41--48}, \bibinfo{year}{2004}.

\bibitem[{Zheng et~al.(2009)Zheng, Zhang, Xie, and Ma}]{zheng2009mining}
\bibinfo{author}{Y.~Zheng}, \bibinfo{author}{L.~Zhang},
  \bibinfo{author}{X.~Xie}, \bibinfo{author}{W.~Y. Ma}, \bibinfo{title}{Mining
  interesting locations and travel sequences from GPS trajectories}, in:
  \bibinfo{booktitle}{International Conference on World Wide Web (WWW)},
  \bibinfo{pages}{791--800}, \bibinfo{year}{2009}.

\bibitem[{Koren and Sill(2011)}]{koren2011ordrec}
\bibinfo{author}{Y.~Koren}, \bibinfo{author}{J.~Sill}, \bibinfo{title}{OrdRec:
  {An} ordinal model for predicting personalized item rating distributions},
  in: \bibinfo{booktitle}{ACM Conference on Recommender Systems (RecSys)},
  \bibinfo{pages}{117--124}, \bibinfo{year}{2011}.

\bibitem[{McFadden(1973)}]{mcfadden1973conditional}
\bibinfo{author}{D.~McFadden}, \bibinfo{title}{Conditional Logit Analysis of
  Qualitative Choice Behavior}, \bibinfo{publisher}{Institute of Urban and
  Regional Development, University of California}, \bibinfo{year}{1973}.

\bibitem[{van Ryzin and Mahajan(1999)}]{ryzin1999relationship}
\bibinfo{author}{G.~van Ryzin}, \bibinfo{author}{S.~Mahajan},
  \bibinfo{title}{On the relationship between inventory costs and variety
  benefits in retail assortments}, \bibinfo{journal}{Management Science}
  \bibinfo{volume}{45}~(\bibinfo{number}{11}) (\bibinfo{year}{1999})
  \bibinfo{pages}{1496--1509}.

\bibitem[{Tian et~al.(2019)Tian, Guo, Kalpathy-Cramer, Ostmo, Campbell, Chiang,
  Dy, Erdogmus, and Ioannidis}]{tian2019severity}
\bibinfo{author}{P.~Tian}, \bibinfo{author}{Y.~Guo},
  \bibinfo{author}{J.~Kalpathy-Cramer}, \bibinfo{author}{S.~Ostmo},
  \bibinfo{author}{J.~P. Campbell}, \bibinfo{author}{M.~F. Chiang},
  \bibinfo{author}{J.~Dy}, \bibinfo{author}{D.~Erdogmus},
  \bibinfo{author}{S.~Ioannidis}, \bibinfo{title}{A severity score for
  retinopathy of prematurity}, in: \bibinfo{booktitle}{International Conference
  on Knowledge Discovery \& Data Mining (KDD)}, \bibinfo{pages}{1809--1819},
  \bibinfo{year}{2019}.

\bibitem[{Y{\i}ld{\i}z et~al.(2019)Y{\i}ld{\i}z, Tian, Dy,
  {Erdo{\u{g}}mu{\c{s}}}, Brown, Kalpathy-Cramer, Ostmo, Campbell, Chiang, and
  Ioannidis}]{yildiz2019classification}
\bibinfo{author}{{\.I}.~Y{\i}ld{\i}z}, \bibinfo{author}{P.~Tian},
  \bibinfo{author}{J.~Dy}, \bibinfo{author}{D.~{Erdo{\u{g}}mu{\c{s}}}},
  \bibinfo{author}{J.~Brown}, \bibinfo{author}{J.~Kalpathy-Cramer},
  \bibinfo{author}{S.~Ostmo}, \bibinfo{author}{J.~P. Campbell},
  \bibinfo{author}{M.~F. Chiang}, \bibinfo{author}{S.~Ioannidis},
  \bibinfo{title}{Classification and comparison via neural networks},
  \bibinfo{journal}{Neural Networks} \bibinfo{volume}{118}
  (\bibinfo{year}{2019}) \bibinfo{pages}{65--80}.

\bibitem[{Guo et~al.(2019)Guo, Dy, Erdo{\u{g}}mu{\c{s}}, Kalpathy-Cramer,
  Ostmo, Campbell, Chiang, and Ioannidis}]{guo2019variational}
\bibinfo{author}{Y.~Guo}, \bibinfo{author}{J.~Dy},
  \bibinfo{author}{D.~Erdo{\u{g}}mu{\c{s}}},
  \bibinfo{author}{J.~Kalpathy-Cramer}, \bibinfo{author}{S.~Ostmo},
  \bibinfo{author}{J.~P. Campbell}, \bibinfo{author}{M.~F. Chiang},
  \bibinfo{author}{S.~Ioannidis}, \bibinfo{title}{Variational Inference from
  Ranked Samples with Features}, in: \bibinfo{booktitle}{Asian Conference on
  Machine Learning (ACML)}, \bibinfo{pages}{599--614}, \bibinfo{year}{2019}.

\bibitem[{Campbell et~al.(2016)Campbell, Kalpathy-Cramer, Erdogmus, Tian,
  Kedarisetti, Moleta, Reynolds, Hutcheson, Shapiro, Repka
  et~al.}]{campbell2016plus}
\bibinfo{author}{J.~P. Campbell}, \bibinfo{author}{J.~Kalpathy-Cramer},
  \bibinfo{author}{D.~Erdogmus}, \bibinfo{author}{P.~Tian},
  \bibinfo{author}{D.~Kedarisetti}, \bibinfo{author}{C.~Moleta},
  \bibinfo{author}{J.~D. Reynolds}, \bibinfo{author}{K.~Hutcheson},
  \bibinfo{author}{M.~J. Shapiro}, \bibinfo{author}{M.~X. Repka}, et~al.,
  \bibinfo{title}{Plus disease in retinopathy of prematurity: {A} continuous
  spectrum of vascular abnormality as a basis of diagnostic variability},
  \bibinfo{journal}{Ophthalmology} \bibinfo{volume}{123}~(\bibinfo{number}{11})
  (\bibinfo{year}{2016}) \bibinfo{pages}{2338--2344}.

\bibitem[{Kalpathy-Cramer et~al.(2016)Kalpathy-Cramer, Campbell, Erdogmus,
  Tian, Kedarisetti, Moleta, Reynolds, Hutcheson, Shapiro, Repka
  et~al.}]{kalpathy2016plus}
\bibinfo{author}{J.~Kalpathy-Cramer}, \bibinfo{author}{J.~P. Campbell},
  \bibinfo{author}{D.~Erdogmus}, \bibinfo{author}{P.~Tian},
  \bibinfo{author}{D.~Kedarisetti}, \bibinfo{author}{C.~Moleta},
  \bibinfo{author}{J.~D. Reynolds}, \bibinfo{author}{K.~Hutcheson},
  \bibinfo{author}{M.~J. Shapiro}, \bibinfo{author}{M.~X. Repka}, et~al.,
  \bibinfo{title}{Plus disease in retinopathy of prematurity: {I}mproving
  diagnosis by ranking disease severity and using quantitative image analysis},
  \bibinfo{journal}{Ophthalmology} \bibinfo{volume}{123}~(\bibinfo{number}{11})
  (\bibinfo{year}{2016}) \bibinfo{pages}{2345--2351}.

\bibitem[{Stewart et~al.(2005)Stewart, Brown, and Chater}]{stewart2005absolute}
\bibinfo{author}{N.~Stewart}, \bibinfo{author}{G.~D. Brown},
  \bibinfo{author}{N.~Chater}, \bibinfo{title}{Absolute identification by
  relative judgment}, \bibinfo{journal}{Psychological Review}
  \bibinfo{volume}{112}~(\bibinfo{number}{4}) (\bibinfo{year}{2005})
  \bibinfo{pages}{881--911}.

\bibitem[{Brun et~al.(2010)Brun, Hamad, Buffet, and Boyer}]{brun2010towards}
\bibinfo{author}{A.~Brun}, \bibinfo{author}{A.~Hamad},
  \bibinfo{author}{O.~Buffet}, \bibinfo{author}{A.~Boyer},
  \bibinfo{title}{Towards preference relations in recommender systems}, in:
  \bibinfo{booktitle}{Workshop on Preference Learning, European Conference on
  Machine Learning and Principle and Practice of Knowledge Discovery in
  Databases (ECML-PKDD)}, \bibinfo{pages}{1--15}, \bibinfo{year}{2010}.

\bibitem[{Desarkar et~al.(2010)Desarkar, Sarkar, and
  Mitra}]{desarkar2010aggregating}
\bibinfo{author}{M.~S. Desarkar}, \bibinfo{author}{S.~Sarkar},
  \bibinfo{author}{P.~Mitra}, \bibinfo{title}{Aggregating preference graphs for
  collaborative rating prediction}, in: \bibinfo{booktitle}{ACM Conference on
  Recommender Systems (RecSys)}, \bibinfo{pages}{21--28}, \bibinfo{year}{2010}.

\bibitem[{Desarkar et~al.(2012)Desarkar, Saxena, and
  Sarkar}]{desarkar2012preference}
\bibinfo{author}{M.~S. Desarkar}, \bibinfo{author}{R.~Saxena},
  \bibinfo{author}{S.~Sarkar}, \bibinfo{title}{Preference relation based matrix
  factorization for recommender systems}, in: \bibinfo{booktitle}{International
  Conference on User Modeling, Adaptation, and Personalization (UMAP)},
  \bibinfo{pages}{63--75}, \bibinfo{year}{2012}.

\bibitem[{Liu et~al.(2014)Liu, Tran, Li, and Jiang}]{liu2014ordinal}
\bibinfo{author}{S.~Liu}, \bibinfo{author}{T.~Tran}, \bibinfo{author}{G.~Li},
  \bibinfo{author}{Y.~Jiang}, \bibinfo{title}{Ordinal random fields for
  recommender systems}, in: \bibinfo{booktitle}{Asian Conference on Machine
  Learning (ACML)}, \bibinfo{pages}{283--298}, \bibinfo{year}{2014}.

\bibitem[{Bradley and Terry(1952)}]{bradley1952rank}
\bibinfo{author}{R.~A. Bradley}, \bibinfo{author}{M.~E. Terry},
  \bibinfo{title}{Rank analysis of incomplete block designs: I. {The} method of
  paired comparisons}, \bibinfo{journal}{Biometrika}
  \bibinfo{volume}{39}~(\bibinfo{number}{3/4}) (\bibinfo{year}{1952})
  \bibinfo{pages}{324--345}.

\bibitem[{Thurstone(1927)}]{thurstone1927law}
\bibinfo{author}{L.~L. Thurstone}, \bibinfo{title}{A law of comparative
  judgment}, \bibinfo{journal}{Psychological Review}
  \bibinfo{volume}{34}~(\bibinfo{number}{4}) (\bibinfo{year}{1927})
  \bibinfo{pages}{266--270}.

\bibitem[{Fligner and Verducci(1993)}]{fligner1993probability}
\bibinfo{author}{M.~A. Fligner}, \bibinfo{author}{J.~S. Verducci},
  \bibinfo{title}{Probability models and statistical analyses for ranking
  data}, \bibinfo{publisher}{Springer}, \bibinfo{year}{1993}.

\bibitem[{Dwork et~al.(2001)Dwork, Kumar, Naor, and Sivakumar}]{dwork2001rank}
\bibinfo{author}{C.~Dwork}, \bibinfo{author}{R.~Kumar},
  \bibinfo{author}{M.~Naor}, \bibinfo{author}{D.~Sivakumar},
  \bibinfo{title}{Rank aggregation methods for the web}, in:
  \bibinfo{booktitle}{International Conference on World Wide Web (WWW)},
  \bibinfo{pages}{613--622}, \bibinfo{year}{2001}.

\bibitem[{Cattelan(2012)}]{cattelan2012models}
\bibinfo{author}{M.~Cattelan}, \bibinfo{title}{Models for paired comparison
  data: {A} review with emphasis on dependent data},
  \bibinfo{journal}{Statistical Science}  (\bibinfo{year}{2012})
  \bibinfo{pages}{412--433}.

\bibitem[{Marden(2014)}]{marden2014analyzing}
\bibinfo{author}{J.~I. Marden}, \bibinfo{title}{Analyzing and Modeling Rank
  Data}, \bibinfo{publisher}{CRC Press}, \bibinfo{year}{2014}.

\bibitem[{Braverman and Mossel(2008)}]{braverman2008noisy}
\bibinfo{author}{M.~Braverman}, \bibinfo{author}{E.~Mossel},
  \bibinfo{title}{Noisy sorting without resampling}, in:
  \bibinfo{booktitle}{Symposium on Discrete Algorithms (SODA)},
  \bibinfo{pages}{268--276}, \bibinfo{year}{2008}.

\bibitem[{Jamieson and Nowak(2011)}]{jamieson2011active}
\bibinfo{author}{K.~G. Jamieson}, \bibinfo{author}{R.~Nowak},
  \bibinfo{title}{Active ranking using pairwise comparisons}, in:
  \bibinfo{booktitle}{Advances in Neural Information Processing Systems
  (NeurIPS)}, \bibinfo{pages}{2240--2248}, \bibinfo{year}{2011}.

\bibitem[{Ammar and Shah(2011)}]{ammar2011ranking}
\bibinfo{author}{A.~Ammar}, \bibinfo{author}{D.~Shah}, \bibinfo{title}{Ranking:
  Compare, don't score}, in: \bibinfo{booktitle}{Allerton Conference on
  Communication, Control, and Computing (Allerton)}, \bibinfo{pages}{776--783},
  \bibinfo{year}{2011}.

\bibitem[{Negahban et~al.(2012)Negahban, Oh, and Shah}]{negahban2012iterative}
\bibinfo{author}{S.~Negahban}, \bibinfo{author}{S.~Oh},
  \bibinfo{author}{D.~Shah}, \bibinfo{title}{Iterative ranking from pair-wise
  comparisons}, in: \bibinfo{booktitle}{Advances in Neural Information
  Processing Systems (NeurIPS)}, \bibinfo{pages}{2474--2482},
  \bibinfo{year}{2012}.

\bibitem[{Shah et~al.(2016)Shah, Balakrishnan, Guntuboyina, and
  Wainwright}]{shah2016stochastically}
\bibinfo{author}{N.~Shah}, \bibinfo{author}{S.~Balakrishnan},
  \bibinfo{author}{A.~Guntuboyina}, \bibinfo{author}{M.~Wainwright},
  \bibinfo{title}{Stochastically transitive models for pairwise comparisons:
  Statistical and computational issues}, in: \bibinfo{booktitle}{International
  Conference on Machine Learning (ICML)}, \bibinfo{pages}{11--20},
  \bibinfo{year}{2016}.

\bibitem[{Rajkumar and Agarwal(2016)}]{rajkumar2016can}
\bibinfo{author}{A.~Rajkumar}, \bibinfo{author}{S.~Agarwal},
  \bibinfo{title}{When can we rank well from comparisons of {$O (n\log (n))$}
  non-actively chosen pairs?}, in: \bibinfo{booktitle}{Conference on Learning
  Theory (COLT)}, \bibinfo{pages}{1376--1401}, \bibinfo{year}{2016}.

\bibitem[{Saha et~al.(2019)Saha, Shivanna, and Bhattacharyya}]{saha2019many}
\bibinfo{author}{A.~Saha}, \bibinfo{author}{R.~Shivanna},
  \bibinfo{author}{C.~Bhattacharyya}, \bibinfo{title}{How many pairwise
  preferences do we need to rank a graph consistently?}, in:
  \bibinfo{booktitle}{AAAI Conference on Artificial Intelligence (AAAI)},
  vol.~\bibinfo{volume}{33}, \bibinfo{pages}{4830--4837}, \bibinfo{year}{2019}.

\bibitem[{Hajek et~al.(2014)Hajek, Oh, and Xu}]{hajek2014minimax}
\bibinfo{author}{B.~Hajek}, \bibinfo{author}{S.~Oh}, \bibinfo{author}{J.~Xu},
  \bibinfo{title}{Minimax-optimal inference from partial rankings}, in:
  \bibinfo{booktitle}{Advances in Neural Information Processing Systems
  (NeurIPS)}, \bibinfo{pages}{1475--1483}, \bibinfo{year}{2014}.

\bibitem[{Plackett(1975)}]{plackett1975analysis}
\bibinfo{author}{R.~L. Plackett}, \bibinfo{title}{The analysis of
  permutations}, \bibinfo{journal}{Journal of the Royal Statistical Society:
  Series C (Applied Statistics)} \bibinfo{volume}{24}~(\bibinfo{number}{2})
  (\bibinfo{year}{1975}) \bibinfo{pages}{193--202}.

\bibitem[{Vojnovic and Yun(2016)}]{vojnovic2016parameter}
\bibinfo{author}{M.~Vojnovic}, \bibinfo{author}{S.~Yun},
  \bibinfo{title}{Parameter estimation for generalized {T}hurstone choice
  models}, in: \bibinfo{booktitle}{International Conference on Machine Learning
  (ICML)}, \bibinfo{pages}{498--506}, \bibinfo{year}{2016}.

\bibitem[{Ailon(2012)}]{ailon2012active}
\bibinfo{author}{N.~Ailon}, \bibinfo{title}{{An active learning algorithm for
  ranking from pairwise preferences with an almost optimal query complexity}},
  \bibinfo{journal}{Journal of Machine Learning Research} \bibinfo{volume}{13}
  (\bibinfo{year}{2012}) \bibinfo{pages}{137--164}.

\bibitem[{Negahban et~al.(2017)Negahban, Oh, and Shah}]{negahban2017rank}
\bibinfo{author}{S.~Negahban}, \bibinfo{author}{S.~Oh},
  \bibinfo{author}{D.~Shah}, \bibinfo{title}{Rank centrality: Ranking from
  pairwise comparisons}, \bibinfo{journal}{Operations Research}
  \bibinfo{volume}{65}~(\bibinfo{number}{1}) (\bibinfo{year}{2017})
  \bibinfo{pages}{266--287}.

\bibitem[{Maystre and Grossglauser(2015)}]{maystre2015fast}
\bibinfo{author}{L.~Maystre}, \bibinfo{author}{M.~Grossglauser},
  \bibinfo{title}{Fast and accurate inference of {Plackett--Luce} models}, in:
  \bibinfo{booktitle}{Advances in Neural Information Processing Systems
  (NeurIPS)}, \bibinfo{pages}{172--180}, \bibinfo{year}{2015}.

\bibitem[{Agarwal et~al.(2018)Agarwal, Patil, and
  Agarwal}]{agarwal2018accelerated}
\bibinfo{author}{A.~Agarwal}, \bibinfo{author}{P.~Patil},
  \bibinfo{author}{S.~Agarwal}, \bibinfo{title}{Accelerated spectral ranking},
  in: \bibinfo{booktitle}{International Conference on Machine Learning (ICML)},
  \bibinfo{pages}{70--79}, \bibinfo{year}{2018}.

\bibitem[{Joachims(2002)}]{joachims2002optimizing}
\bibinfo{author}{T.~Joachims}, \bibinfo{title}{Optimizing search engines using
  clickthrough data}, in: \bibinfo{booktitle}{International Conference on
  Knowledge Discovery and Data Mining (KDD)}, \bibinfo{pages}{133--142},
  \bibinfo{year}{2002}.

\bibitem[{Pahikkala et~al.(2009)Pahikkala, Tsivtsivadze, Airola, J{\"a}rvinen,
  and Boberg}]{pahikkala2009efficient}
\bibinfo{author}{T.~Pahikkala}, \bibinfo{author}{E.~Tsivtsivadze},
  \bibinfo{author}{A.~Airola}, \bibinfo{author}{J.~J{\"a}rvinen},
  \bibinfo{author}{J.~Boberg}, \bibinfo{title}{An efficient algorithm for
  learning to rank from preference graphs}, \bibinfo{journal}{Machine Learning}
  \bibinfo{volume}{75}~(\bibinfo{number}{1}) (\bibinfo{year}{2009})
  \bibinfo{pages}{129--165}.

\bibitem[{Burges et~al.(2005)Burges, Shaked, Renshaw, Lazier, Deeds, Hamilton,
  and Hullender}]{burges2005learning}
\bibinfo{author}{C.~Burges}, \bibinfo{author}{T.~Shaked},
  \bibinfo{author}{E.~Renshaw}, \bibinfo{author}{A.~Lazier},
  \bibinfo{author}{M.~Deeds}, \bibinfo{author}{N.~Hamilton},
  \bibinfo{author}{G.~N. Hullender}, \bibinfo{title}{Learning to rank using
  gradient descent}, in: \bibinfo{booktitle}{International Conference on
  Machine learning (ICML)}, \bibinfo{pages}{89--96}, \bibinfo{year}{2005}.

\bibitem[{Chang et~al.(2016)Chang, Yu, Wang, Ashley, and
  Finkelstein}]{chang2016automatic}
\bibinfo{author}{H.~Chang}, \bibinfo{author}{F.~Yu}, \bibinfo{author}{J.~Wang},
  \bibinfo{author}{D.~Ashley}, \bibinfo{author}{A.~Finkelstein},
  \bibinfo{title}{Automatic triage for a photo series},
  \bibinfo{journal}{Transactions on Graphics (TOG)}
  \bibinfo{volume}{35}~(\bibinfo{number}{4}) (\bibinfo{year}{2016})
  \bibinfo{pages}{1--10}.

\bibitem[{Dubey et~al.(2016)Dubey, Naik, Parikh, Raskar, and
  Hidalgo}]{dubey2016deep}
\bibinfo{author}{A.~Dubey}, \bibinfo{author}{N.~Naik},
  \bibinfo{author}{D.~Parikh}, \bibinfo{author}{R.~Raskar},
  \bibinfo{author}{C.~A. Hidalgo}, \bibinfo{title}{Deep learning the city:
  Quantifying urban perception at a global scale}, in:
  \bibinfo{booktitle}{European Conference on Computer Vision (ECCV)},
  \bibinfo{pages}{196--212}, \bibinfo{year}{2016}.

\bibitem[{Canonne et~al.(2015)Canonne, Ron, and Servedio}]{canonne2015testing}
\bibinfo{author}{C.~L. Canonne}, \bibinfo{author}{D.~Ron},
  \bibinfo{author}{R.~A. Servedio}, \bibinfo{title}{Testing probability
  distributions using conditional samples}, \bibinfo{journal}{SIAM Journal on
  Computing} \bibinfo{volume}{44}~(\bibinfo{number}{3}) (\bibinfo{year}{2015})
  \bibinfo{pages}{540--616}.

\bibitem[{Kane et~al.(2017)Kane, Lovett, Moran, and Zhang}]{kane2017active}
\bibinfo{author}{D.~M. Kane}, \bibinfo{author}{S.~Lovett},
  \bibinfo{author}{S.~Moran}, \bibinfo{author}{J.~Zhang},
  \bibinfo{title}{Active classification with comparison queries}, in:
  \bibinfo{booktitle}{Symposium on Foundations of Computer Science (FOCS)},
  \bibinfo{pages}{355--366}, \bibinfo{year}{2017}.

\bibitem[{Vapnik and Chervonenkis(1971)}]{vapnik1971uniform}
\bibinfo{author}{V.~N. Vapnik}, \bibinfo{author}{A.~Y. Chervonenkis},
  \bibinfo{title}{On the Uniform Convergence of Relative Frequencies of Events
  to Their Probabilities}, \bibinfo{journal}{Theory of Probability \& Its
  Applications} \bibinfo{volume}{16}~(\bibinfo{number}{2})
  (\bibinfo{year}{1971}) \bibinfo{pages}{264--280}.

\bibitem[{Valiant(1984)}]{valiant1984theory}
\bibinfo{author}{L.~G. Valiant}, \bibinfo{title}{A theory of the learnable},
  \bibinfo{journal}{Communications of the ACM}
  \bibinfo{volume}{27}~(\bibinfo{number}{11}) (\bibinfo{year}{1984})
  \bibinfo{pages}{1134--1142}.

\bibitem[{Ehrenfeucht et~al.(1989)Ehrenfeucht, Haussler, Kearns, and
  Valiant}]{ehrenfeucht1989general}
\bibinfo{author}{A.~Ehrenfeucht}, \bibinfo{author}{D.~Haussler},
  \bibinfo{author}{M.~Kearns}, \bibinfo{author}{L.~Valiant}, \bibinfo{title}{A
  general lower bound on the number of examples needed for learning},
  \bibinfo{journal}{Information and Computation}
  \bibinfo{volume}{82}~(\bibinfo{number}{3}) (\bibinfo{year}{1989})
  \bibinfo{pages}{247--261}.

\bibitem[{Kearns et~al.(1994)Kearns, Vazirani, and
  Vazirani}]{kearns1994introduction}
\bibinfo{author}{M.~J. Kearns}, \bibinfo{author}{U.~V. Vazirani},
  \bibinfo{author}{U.~Vazirani}, \bibinfo{title}{An Introduction to
  Computational Learning Theory}, \bibinfo{publisher}{MIT Press},
  \bibinfo{year}{1994}.

\bibitem[{Vapnik(2006)}]{vapnik2006estimation}
\bibinfo{author}{V.~Vapnik}, \bibinfo{title}{Estimation of Dependences Based on
  Empirical Data}, \bibinfo{publisher}{Springer Science \& Business Media},
  \bibinfo{year}{2006}.

\bibitem[{Balcan and Long(2013)}]{balcan2013active}
\bibinfo{author}{M.~F. Balcan}, \bibinfo{author}{P.~Long},
  \bibinfo{title}{Active and passive learning of linear separators under
  log-concave distributions}, in: \bibinfo{booktitle}{Conference on Learning
  Theory (COLT)}, \bibinfo{pages}{288--316}, \bibinfo{year}{2013}.

\bibitem[{Balcan and Zhang(2017)}]{balcan2017sample}
\bibinfo{author}{M.~F. Balcan}, \bibinfo{author}{H.~Zhang},
  \bibinfo{title}{Sample and computationally efficient learning algorithms
  under s-concave distributions}, in: \bibinfo{booktitle}{Advances in Neural
  Information Processing Systems (NeurIPS)}, \bibinfo{pages}{4796--4805},
  \bibinfo{year}{2017}.

\bibitem[{Kalai et~al.(2008)Kalai, Klivans, Mansour, and
  Servedio}]{kalai2008agnostically}
\bibinfo{author}{A.~T. Kalai}, \bibinfo{author}{A.~R. Klivans},
  \bibinfo{author}{Y.~Mansour}, \bibinfo{author}{R.~A. Servedio},
  \bibinfo{title}{Agnostically learning halfspaces}, \bibinfo{journal}{SIAM
  Journal on Computing} \bibinfo{volume}{37}~(\bibinfo{number}{6})
  (\bibinfo{year}{2008}) \bibinfo{pages}{1777--1805}.

\bibitem[{Mammen and Tsybakov(1999)}]{mammen1999smooth}
\bibinfo{author}{E.~Mammen}, \bibinfo{author}{A.~B. Tsybakov},
  \bibinfo{title}{Smooth discrimination analysis}, \bibinfo{journal}{Annals of
  Statistics} \bibinfo{volume}{27}~(\bibinfo{number}{6}) (\bibinfo{year}{1999})
  \bibinfo{pages}{1808--1829}.

\bibitem[{Hanneke(2007)}]{hanneke2007bound}
\bibinfo{author}{S.~Hanneke}, \bibinfo{title}{A bound on the label complexity
  of agnostic active learning}, in: \bibinfo{booktitle}{International
  Conference on Machine Learning (ICML)}, \bibinfo{pages}{353--360},
  \bibinfo{year}{2007}.

\bibitem[{Cavallanti et~al.(2011)Cavallanti, Cesa-Bianchi, and
  Gentile}]{cavallanti2011learning}
\bibinfo{author}{G.~Cavallanti}, \bibinfo{author}{N.~Cesa-Bianchi},
  \bibinfo{author}{C.~Gentile}, \bibinfo{title}{Learning noisy linear
  classifiers via adaptive and selective sampling}, \bibinfo{journal}{Machine
  Learning} \bibinfo{volume}{83}~(\bibinfo{number}{1}) (\bibinfo{year}{2011})
  \bibinfo{pages}{71--102}.

\bibitem[{Awasthi et~al.(2016)Awasthi, Balcan, Haghtalab, and
  Zhang}]{awasthi2016learning}
\bibinfo{author}{P.~Awasthi}, \bibinfo{author}{M.~F. Balcan},
  \bibinfo{author}{N.~Haghtalab}, \bibinfo{author}{H.~Zhang},
  \bibinfo{title}{Learning and 1-bit compressed sensing under asymmetric
  noise}, in: \bibinfo{booktitle}{Conference on Learning Theory (COLT)},
  \bibinfo{pages}{152--192}, \bibinfo{year}{2016}.

\bibitem[{Zhang(2018)}]{zhang2018efficient}
\bibinfo{author}{C.~Zhang}, \bibinfo{title}{Efficient active learning of sparse
  halfspaces}, in: \bibinfo{booktitle}{Conference on Learning Theory (COLT)},
  vol.~\bibinfo{volume}{75}, \bibinfo{pages}{1--26}, \bibinfo{year}{2018}.

\bibitem[{Niranjan and Rajkumar(2017)}]{niranjan2017inductive}
\bibinfo{author}{U.~Niranjan}, \bibinfo{author}{A.~Rajkumar},
  \bibinfo{title}{Inductive pairwise ranking: going beyond the $n \log (n)$
  barrier}, in: \bibinfo{booktitle}{AAAI Conference on Artificial Intelligence
  (AAAI)}, \bibinfo{pages}{2436--2442}, \bibinfo{year}{2017}.

\bibitem[{Chiang et~al.(2017)Chiang, Hsieh, and Dhillon}]{chiang2017rank}
\bibinfo{author}{K.-Y. Chiang}, \bibinfo{author}{C.-J. Hsieh},
  \bibinfo{author}{I.~Dhillon}, \bibinfo{title}{Rank aggregation and prediction
  with item features}, in: \bibinfo{booktitle}{International Conference on
  Artificial Intelligence and Statistics (AISTATS)}, \bibinfo{pages}{748--756},
  \bibinfo{year}{2017}.

\bibitem[{Bartlett and Mendelson(2002)}]{bartlett2002rademacher}
\bibinfo{author}{P.~L. Bartlett}, \bibinfo{author}{S.~Mendelson},
  \bibinfo{title}{Rademacher and Gaussian complexities: Risk bounds and
  structural results}, \bibinfo{journal}{Journal of Machine Learning Research}
  \bibinfo{volume}{3}~(\bibinfo{number}{Nov}) (\bibinfo{year}{2002})
  \bibinfo{pages}{463--482}.

\bibitem[{Hartlap et~al.(2007)Hartlap, Simon, and Schneider}]{hartlap2007your}
\bibinfo{author}{J.~Hartlap}, \bibinfo{author}{P.~Simon},
  \bibinfo{author}{P.~Schneider}, \bibinfo{title}{Why your model parameter
  confidences might be too optimistic. {Unbiased} estimation of the inverse
  covariance matrix}, \bibinfo{journal}{Astronomy \& Astrophysics}
  \bibinfo{volume}{464}~(\bibinfo{number}{1}) (\bibinfo{year}{2007})
  \bibinfo{pages}{399--404}.

\bibitem[{Friedman et~al.(2001)Friedman, Hastie, and
  Tibshirani}]{friedman2001elements}
\bibinfo{author}{J.~Friedman}, \bibinfo{author}{T.~Hastie},
  \bibinfo{author}{R.~Tibshirani}, \bibinfo{title}{The Elements of Statistical
  Learning}, \bibinfo{publisher}{Springer Series in Statistics},
  \bibinfo{year}{2001}.

\bibitem[{Stein(1981)}]{stein1973estimation}
\bibinfo{author}{C.~M. Stein}, \bibinfo{title}{Estimation of the mean of a
  multivariate distribution}, \bibinfo{journal}{The Annals of Statistics}
  \bibinfo{volume}{9}~(\bibinfo{number}{6}) (\bibinfo{year}{1981})
  \bibinfo{pages}{1135--1151}.

\bibitem[{Liu(1994)}]{liu1994siegel}
\bibinfo{author}{J.~S. Liu}, \bibinfo{title}{Siegel's formula via {Stein's}
  identities}, \bibinfo{journal}{Statistics \& Probability Letters}
  \bibinfo{volume}{21}~(\bibinfo{number}{3}) (\bibinfo{year}{1994})
  \bibinfo{pages}{247--251}.

\bibitem[{Dasgupta and Gupta(2003)}]{dasgupta2003elementary}
\bibinfo{author}{S.~Dasgupta}, \bibinfo{author}{A.~Gupta}, \bibinfo{title}{An
  elementary proof of a theorem of {Johnson} and {Lindenstrauss}},
  \bibinfo{journal}{Random Structures \& Algorithms}
  \bibinfo{volume}{22}~(\bibinfo{number}{1}) (\bibinfo{year}{2003})
  \bibinfo{pages}{60--65}.

\bibitem[{Wellner and van~der Vaart(2013)}]{wellner2013weak}
\bibinfo{author}{J.~Wellner}, \bibinfo{author}{A.~W. van~der Vaart},
  \bibinfo{title}{Weak Convergence and Empirical Processes: with Applications
  to Statistics}, \bibinfo{publisher}{Springer Science \& Business Media},
  \bibinfo{year}{2013}.

\bibitem[{Hoeffding(1963)}]{hoeffding1994probability}
\bibinfo{author}{W.~Hoeffding}, \bibinfo{title}{Probability inequalities for
  sums of bounded random variables}, \bibinfo{journal}{Journal of the American
  Statistical Association} \bibinfo{volume}{58}~(\bibinfo{number}{301})
  (\bibinfo{year}{1963}) \bibinfo{pages}{13--30}.

\bibitem[{Vershynin(2012)}]{vershynin2012}
\bibinfo{author}{R.~Vershynin}, \bibinfo{title}{Introduction to the
  Non-Asymptotic Analysis of Random Matrices}, in: \bibinfo{editor}{Y.~C.
  Eldar}, \bibinfo{editor}{G.~Kutyniok} (Eds.), \bibinfo{booktitle}{Compressed
  Sensing: Theory and Practice}, \bibinfo{publisher}{Cambridge University
  Press}, \bibinfo{pages}{210--268}, \bibinfo{year}{2012}.

\bibitem[{Hsu et~al.(2012)Hsu, Kakade, and Zhang}]{hsu2012tail}
\bibinfo{author}{D.~Hsu}, \bibinfo{author}{S.~Kakade},
  \bibinfo{author}{T.~Zhang}, \bibinfo{title}{A tail inequality for quadratic
  forms of subgaussian random vectors}, \bibinfo{journal}{Electronic
  Communications in Probability} \bibinfo{volume}{17}~(\bibinfo{number}{52})
  (\bibinfo{year}{2012}) \bibinfo{pages}{1--6}.

\bibitem[{Haussler et~al.(1994)Haussler, Littlestone, and
  Warmuth}]{haussler1994predicting}
\bibinfo{author}{D.~Haussler}, \bibinfo{author}{N.~Littlestone},
  \bibinfo{author}{M.~K. Warmuth}, \bibinfo{title}{Predicting $\{$0,
  1$\}$-functions on randomly drawn points}, \bibinfo{journal}{Information and
  Computation} \bibinfo{volume}{115}~(\bibinfo{number}{2})
  (\bibinfo{year}{1994}) \bibinfo{pages}{248--292}.

\bibitem[{Vershynin(2018)}]{vershynin2018high}
\bibinfo{author}{R.~Vershynin}, \bibinfo{title}{High-dimensional probability:
  An introduction with applications in data science},
  \bibinfo{publisher}{Cambridge University Press}, \bibinfo{year}{2018}.

\end{thebibliography}
\section*{Biographies}
\textbf{Berkan Kad\i o\u{g}lu} \highlight{is a PhD candidate at the Department of Electrical and Computer Engineering, Northeastern University, Boston, MA, since 2017.
He received his B.Sc. (2017) in Electrical and Electronics Engineering from Bilkent University, Ankara, Turkey.
His research interests include machine learning, statistical learning, deep learning with a focus on learning from comparisons.}

\textbf{Dr.~Peng Tian} \highlight{is a postdoctoral research associate at Northeastern University, Boston, MA. He received his Ph.D.~and M.S.~degree of Electrical Engineering in 2020 and 2017 respectively, from Northeastern University, Boston, MA. He obtained his B.E.~degree of Optoelectronic Information Engineering in 2015 from Huazhong University of Science \& and Technology, Wuhan, China. His research instests span learning from comparisons, deep learning and AI for healthcare.}

\textbf{Dr.~Jennifer G. Dy} \highlight{is a professor at the Department of Electrical and Computer Engineering, Northeastern University, Boston, MA, since 2002. She obtained her MS and PhD in 1997 and 2001 respectively from the School of Electrical and Computer Engineering, Purdue University, West Lafayette, IN, and her BS degree in 1993 from the Department of Electrical Engineering, University of the Philippines. She received an NSF Career award in 2004. She is an editorial board member for the journal, Machine Learning since 2004, publications chair for the International Conference on Machine Learning in 2004, and program committee member for ICML, ACM SIGKDD, AAAI, and SIAM SDM. Her research interests include Machine Learning, Data Mining, Statistical Pattern Recognition, and Computer Vision.}

\textbf{Dr.~Deniz Erdo\u{g}mus} \highlight{graduated with B.S. in Electrical \& Electronics Engineering (EEE), and the B.S. in Mathematics in 1997, and M.S.~in EEE in 1999 from the Middle East Technical University, Ankara, Turkey. He received his Ph.D.~in Electrical \& Computer Engineering from the University of Florida in 2002, where he stayed as a postdoctoral research associate until 2004. Prior to joining the Northeastern faculty in 2008, he held an Assistant Professor position at the Oregon Health and Science University. His expertise is in information theoretic and nonparametric machine learning and adaptive signal processing, specifically focusing on cognitive signal processing including brain interfaces and assistive technologies. Deniz has been serving as an associate editor IEEE Transactions on Signal Processing, Transactions on Neural Networks, Signal Processing Letters, and Elsevier Neurocomputing. He is a member of the IEEE-SPS Machine Learning for Signal Processing Technical Committee.}

\textbf{Dr.~Stratis Ioannidis} \highlight{is an associate professor in the Electrical and Computer Engineering Department of Northeastern University, in Boston, MA, where he also holds a courtesy appointment with the College of Computer and Information Science. He received his B.Sc.~(2002) in Electrical and Computer Engineering from the National Technical University of Athens, Greece, and his M.Sc. (2004) and Ph.D.~(2009) in Computer Science from the University of Toronto, Canada. Prior to joining Northeastern, he was a research scientist at the Technicolor research centers in Paris, France, and Palo Alto, CA, as well as at Yahoo Labs in Sunnyvale, CA. He is the recipient of an NSF CAREER Award, a Google Faculty Research Award, a Facebook Research Award, a Martin W. Essigmann Outstanding Teaching Award, and several best paper awards. His research interests span machine learning, distributed systems, networking, optimization, and privacy. }

\newpage
\appendix
\section{Proof of Lemma~\ref{lem:unbiased}}
\label{app:unbiased}
The expected value of the estimator is
\begin{align*}
    \E[\hat{\bm{\beta}}] &= \E\left[\frac{1}{M}\sum_{m = 1}^M Y_m \bm{\hat\Sigma}^{-1} (\bm X_{I_m} - \bm X_{J_m})\right] ~~~ \text{by Eq.\eqref{eq:estimator},} \nonumber \\
    &= \E\left[Y_m \bm{\hat\Sigma}^{-1} (\bm X_{I_m} - \bm X_{J_m})\right] \nonumber \\
    &= \E\left[\bm{\hat\Sigma}^{-1}\right] \E\left[ Y_m(\bm X_{I_m} - \bm X_{J_m}) \right] ~~~ \text{by $\bm{\hat\Sigma}^{-1} \independent Y_m, \bm X_{I_m}, \bm X_{J_m}$,} \nonumber \\
    &=\bm{\Sigma}^{-1} \E\left[ Y_m(\bm X_{I_m} - \bm X_{J_m}) \right] ~~~ \text{by \cite{hartlap2007your},} \nonumber \\
    &=\bm{\Sigma}^{-1}\E\left[(\bm X_{I_m} - \bm X_{J_m}) \E\left[Y_m |\bm X_{I_m} - \bm X_{J_m} \right] \right] \nonumber \\
    &=\bm{\Sigma}^{-1} \E\left[ (\bm X_{I_m} - \bm X_{J_m})\left( 2\fn\big(\bm\beta^\top(\bm X_{I_m} - \bm X_{J_m})\big) - 1 \right) \right] \nonumber \\
    &=\bm{\Sigma}^{-1} \text{Cov}\left[\bm X_{I_m} - \bm X_{J_m}, 2\fn\left(\bm \beta^T (\bm X_{I_m} - \bm X_{J_m})\right)-1\right] \nonumber \\
    &=2\bm{\Sigma}^{-1} \text{Cov}\left[\bm X_{I_m} \! - \! \bm X_{J_m}, \bm \beta^T (\bm X_{I_m}\! -\! \bm X_{J_m}) \right]\E\left[\fn'\left(\bm \beta^T (\bm X_{I_m} \!-\! \bm X_{J_m}) \right)\right] \nonumber \\
    &= 2\bm{\Sigma}^{-1}\E\left[ (\bm X_{I_m} - \bm X_{J_m})(\bm X_{I_m} - \bm X_{J_m})^T\right] \E\left[ \fn'\left( \bm{\beta}^T(\bm X_{I_m} - \bm X_{J_m}) \right) \right]\bm{\beta} \nonumber \\
    &= 4\bm{\Sigma}^{-1} \bm{\Sigma} \E\left[ \fn'\left( \bm{\beta}^T(\bm X_{I_m} - \bm X_{J_m}) \right) \right] \bm{\beta} = 4\E\left[ \fn'\left( \bm{\beta}^T(\bm{X}_{I_m} - \bm X_{J_m}) \right) \right]\bm{\beta} \nonumber \\ 
    &= c_1 \bm \beta,
\end{align*}
where the fourth to last line is by Lemma~\ref{lemma:stein} and $c_1 = 4\E\left[ \fn'\left( \bm{\beta}^T(\bm X_{I_m} - \bm X_{J_m}) \right) \right]$.
Note that $c_1$ is strictly positive as $f(x)$ is non-decreasing, and has limits $\lim_{x\rightarrow\infty}f(x) = 1$, $\lim_{x\rightarrow-\infty}f(x) = 0$. 
Therefore, there exists an $s \in \R$ at which $f'(s)>0$. 
By continuity, $f'(x)$ is therefore strictly positive in an interval around $s$. 
As a result, the integral in the expectation which defines $c_1$ is strictly positive. \qed
\section{Proof of Lemma~\ref{lem:2terms}}
\label{app:2terms}
We start by dividing the error into two symmetric terms.
We have
\begin{align*}
    \norm{\hat{\bm{\beta}} - c\bm{\beta}} &= \norm{\frac{1}{M}\sum_{m = 1}^M Y_m \bm{\hat \Sigma}^{-1}(\bm X_{I_m} - \bm X_{J_m}) - \E[Y_m \bm \Sigma^{-1} (\bm X_{I_m} - \bm X_{J_m})]} \nonumber\\
    &= \Bigg \| \frac{1}{M}\sum_{m = 1}^M Y_m \bm{\hat \Sigma}^{-1}(\bm{X}_{I_m} - \bm{\mu}) - \frac{1}{M}\sum_{m = 1}^M Y_m \bm{\hat \Sigma}^{-1}(\bm{X}_{J_m} - \bm{\mu}) \nonumber \\
    &- \E[Y_m\bm{\Sigma}^{-1}\bm (X_{I_m} - \bm{\mu})] + \E[Y_m \bm{\Sigma}^{-1}(\bm X_{J_m} - \bm{\mu})] \Bigg \| \nonumber \\
    &\leq \norm{\frac{1}{M}\sum_{m = 1}^M Y_m\bm{\hat \Sigma}^{-1} (\bm{X}_{I_m} - \bm{\mu}) - \E[Y_m\bm{\Sigma}^{-1}(\bm X_{I_m} - \bm{\mu})]} \nonumber \\
    &+ \norm{\frac{1}{M}\sum_{m = 1}^M Y_m \bm{\hat \Sigma}^{-1}(\bm{X}_{J_m} - \bm{\mu}) - \E[Y_m\bm{\Sigma}^{-1}(\bm X_{J_m} - \bm{\mu})]},
\end{align*}
where the last line is by a triangle inequality and the first line is by Eq.~\eqref{eq:estimator} and Lemma~\ref{lem:unbiased}. Then, we show that these terms are bounded by the same probability.
We start by defining $Y_m' = -Y_m$ and note that
\begin{align}
    \label{eq:y'}
    \Pr(Y_m' \mid \bm X_{I_m} = x, \bm X_{J_m} = y) &= \Pr(-Y_m \mid \bm X_{I_m} = x, \bm X_{J_m} = y) \nonumber \\
    &= 1 - \fn(\bm\beta^\top(x-y)) ~~~ \text{by Eq.~\eqref{eq:conditional},} \nonumber \\
    &= \fn(\bm\beta^\top(y-x)) ~~~ \text{by Eq.~\eqref{eq:properties}.}
\end{align}
Then, we have
\begin{align}
    \label{eq:pr_2terms}
    &\Pr\left(\norm{\frac{1}{M}\sum_{m = 1}^M Y_m \bm{\hat \Sigma}^{-1}(\bm{X}_{J_m} - \bm{\mu}) - \E[Y_m\bm{ \Sigma}^{-1}(\bm X_{J_m} - \bm{\mu})]} > \epsilon\right) \nonumber \\
    &= \Pr\left(\norm{\frac{1}{M}\sum_{m = 1}^M -Y_m \bm{\hat \Sigma}^{-1}(\bm{X}_{J_m}- \bm{\mu}) - \E[-Y_m\bm{\Sigma}^{-1}(\bm X_{J_m}- \bm{\mu})]} > \epsilon\right) \nonumber \\
    &= \Pr\left(\norm{\frac{1}{M}\sum_{m = 1}^M Y_m' \bm{\hat \Sigma}^{-1}(\bm{X}_{J_m}- \bm{\mu}) - \E[Y_m'\bm{\Sigma}^{-1}(\bm X_{J_m}- \bm{\mu})]} > \epsilon\right) \nonumber \\
    &= \Pr\left(\norm{\frac{1}{M}\sum_{m = 1}^M Y_m \bm{\hat \Sigma}^{-1}(\bm{X}_{I_m}- \bm{\mu}) - \E[Y_m\bm{ \Sigma}^{-1}(\bm X_{I_m}- \bm{\mu})]} > \epsilon\right),
\end{align}
where the last line is by Eq.~\eqref{eq:y'}. We use this result to show that
\begin{align}
    &\Pr\left(\norm{\hat{\bm{\beta}} - c\bm{\beta}} > \epsilon\right) \leq \Pr\Bigg(\norm{\frac{1}{M}\sum_{m = 1}^M Y_m\bm{\hat \Sigma}^{-1} (\bm{X}_{I_m} - \bm \mu) - \E[Y_m\bm{\Sigma}^{-1}(\bm X_{I_m} - \bm \mu)]} \nonumber \\
    &+ \norm{\frac{1}{M}\sum_{m = 1}^M Y_m \bm{\hat \Sigma}^{-1}(\bm{X}_{J_m} - \bm \mu) - \E[Y_m\bm{ \Sigma}^{-1}(\bm X_{J_m} - \bm \mu)]} > \epsilon\Bigg) \nonumber \\
    &\leq \Pr\left(\norm{\frac{1}{M}\sum_{m = 1}^M Y_m\bm{\hat \Sigma}^{-1} (\bm{X}_{I_m} - \bm \mu) - \E[Y_m\bm{ \Sigma}^{-1}(\bm X_{I_m} - \bm \mu)]} > \epsilon/2 \right)
    \nonumber \\
    &+ \Pr\left(\norm{\frac{1}{M}\sum_{m = 1}^M Y_m \bm{\hat \Sigma}^{-1}(\bm{X}_{J_m} - \bm \mu) - \E[Y_m\bm{\Sigma}^{-1}(\bm X_{J_m} - \bm \mu)]} > \epsilon/2 \right) \nonumber \\
    &= 2\Pr\left(\norm{\!\frac{1}{M}\!\sum_{m = 1}^M \!Y_m \bm{\hat \Sigma}^{-1}(\bm{X}_{I_m} \!-\! \bm \mu) - \E[Y_m\bm{\Sigma}^{-1}(\bm X_{I_m}\! -\! \bm \mu)]}\! >\! \epsilon/2 \right),
\end{align}
where the last line is by Eq.~\eqref{eq:pr_2terms}. Then, 
\begin{align}
    &\norm{\frac{1}{M}\sum_{m = 1}^M Y_m\bm{\hat \Sigma}^{-1}(\bm{X}_{I_m} - \bm{\mu}) - \E[Y_m\bm{\Sigma}^{-1}(\bm X_{I_m} - \bm{\mu})]} \nonumber \\
    &\leq \norm{\frac{1}{M}\sum_{m = 1}^M Y_m\bm{\hat \Sigma}^{-1} (\bm{X}_{I_m} - \bm{\mu}) - \frac{1}{M}\sum_{m = 1}^M Y_m\bm{\Sigma}^{-1} (\bm{X}_{I_m} - \bm{\mu})} \nonumber \\
    &+\norm{\frac{1}{M}\sum_{m = 1}^M Y_m\bm{\Sigma}^{-1}(\bm{X}_{I_m} - \bm{\mu}) - \E[Y_m\bm{\Sigma}^{-1}(\bm X_{I_m} - \bm \mu)]} \nonumber \\
    &\leq \norm{\bm{\hat \Sigma}^{-1} - \bm{\Sigma}^{-1}}\norm{\frac{1}{M}\sum_{m = 1}^M Y_m (\bm{X}_{I_m} - \bm{\mu})} \nonumber \\
    &+\norm{\frac{1}{M}\sum_{m = 1}^M Y_m\bm{\Sigma}^{-1}(\bm{X}_{I_m} - \bm{\mu}) - \E[Y_m\bm{\Sigma}^{-1}(\bm X_{I_m} - \bm \mu)]} \nonumber \\
    &\leq \norm{\bm{\hat \Sigma}^{-1} - \bm{\Sigma}^{-1}}\norm{\frac{1}{M}\sum_{m = 1}^M Y_m (\bm{X}_{I_m} - \bm{\mu}) - \E[Y_m (\bm{X}_{I_m} - \bm{\mu})]} \nonumber \\
    &+ \norm{\bm{\hat \Sigma}^{-1} - \bm{\Sigma}^{-1}}\norm{\E[Y_m (\bm{X}_{I_m} - \bm{\mu})]} \nonumber \\
    &+\norm{\frac{1}{M}\sum_{m = 1}^M Y_m\bm{\Sigma}^{-1}(\bm{X}_{I_m} - \bm{\mu}) - \E[Y_m\bm{\Sigma}^{-1}(\bm X_{I_m} - \bm \mu)]},
\end{align}
where the first and last inequalities are by triangle inequalities and the second inequality is by the Cauchy-Schwarz inequality. Note that,
\begin{align*}
    &\Pr\left(\norm{\bm{\hat \Sigma}^{-1} - \bm{\Sigma}^{-1}}\norm{\frac{1}{M}\sum_{m = 1}^M Y_m (\bm{X}_{I_m} - \bm{\mu}) - \E[Y_m (\bm{X}_{I_m} - \bm{\mu})]} > \epsilon \right) \\
    &\leq \Pr\left(\norm{\bm{\hat \Sigma}^{-1} - \bm{\Sigma}^{-1}} > \sqrt \epsilon\right) \\
    &+ \Pr\left(\norm{\frac{1}{M}\sum_{m = 1}^M Y_m (\bm{X}_{I_m} - \bm{\mu}) - \E[Y_m (\bm{X}_{I_m} - \bm{\mu})]} > \sqrt \epsilon\right) \\
    &\leq \Pr\left(\norm{\bm{\hat \Sigma}^{-1} - \bm{\Sigma}^{-1}} > \epsilon\right) \\
    &+ \Pr\left(\norm{\frac{1}{M}\sum_{m = 1}^M Y_m (\bm{X}_{I_m} - \bm{\mu}) - \E[Y_m (\bm{X}_{I_m} - \bm{\mu})]} > \epsilon\right),
\end{align*}
since we have for $\epsilon < 1$, $\sqrt \epsilon > \epsilon$. 
These terms appear in the bound more than once, therefore we can ignore the higher order term by multiplying the lower order terms with a constant number, e.g. $2$. 
These result in,
\begin{align*}
    &\Pr\left(\norm{\hat{\bm{\beta}} - c\bm{\beta}} > \epsilon\right) \\ 
    &\leq 2\Pr\left(\norm{\frac{1}{M}\sum_{m = 1}^M Y_m \bm{\hat \Sigma}^{-1}(\bm{X}_{I_m} - \bm \mu) - \E[Y_m\bm{\Sigma}^{-1}(\bm X_{I_m} - \bm \mu)]} > \epsilon/2 \right) \\
    &\leq 2\Pr\Bigg(\norm{\bm{\hat \Sigma}^{-1} - \bm{\Sigma}^{-1}}\norm{\frac{1}{M}\sum_{m = 1}^M Y_m (\bm{X}_{I_m} - \bm{\mu}) - \E[Y_m (\bm{X}_{I_m} - \bm{\mu})]}  \\
    &+ \norm{\bm{\hat \Sigma}^{-1} - \bm{\Sigma}^{-1}}\norm{\E[Y_m (\bm{X}_{I_m} - \bm{\mu})]} \\
    &+\norm{\frac{1}{M}\sum_{m = 1}^M Y_m\bm{\Sigma}^{-1}(\bm{X}_{I_m} - \bm{\mu}) - \E[Y_m\bm{\Sigma}^{-1}(\bm X_{I_m} - \bm \mu)]} > \epsilon/2 \Bigg) \\
    &\leq 2\Pr\Bigg(\norm{\bm{\hat \Sigma}^{-1} - \bm{\Sigma}^{-1}}\norm{\frac{1}{M}\sum_{m = 1}^M Y_m (\bm{X}_{I_m} - \bm{\mu}) - \E[Y_m (\bm{X}_{I_m} - \bm{\mu})]} > \epsilon/6\Bigg) \\
    &+ 2\Pr\Bigg(\norm{\bm{\hat \Sigma}^{-1} - \bm{\Sigma}^{-1}}\norm{\E[Y_m (\bm{X}_{I_m} - \bm{\mu})]} > \epsilon/6\Bigg) \\
    &+ 2\Pr\Bigg(\norm{\frac{1}{M}\sum_{m = 1}^M Y_m\bm{\Sigma}^{-1}(\bm{X}_{I_m} - \bm{\mu}) - \E[Y_m\bm{\Sigma}^{-1}(\bm X_{I_m} - \bm \mu)]} > \epsilon/6 \Bigg) \\
    &\leq 4\Pr\Bigg(\norm{\bm{\hat \Sigma}^{-1} - \bm{\Sigma}^{-1}}\norm{\E[Y_m (\bm{X}_{I_m} - \bm{\mu})]} > \epsilon/6\Bigg) \\
    &+ 4\Pr\Bigg(\norm{\frac{1}{M}\sum_{m = 1}^M Y_m\bm{\Sigma}^{-1}(\bm{X}_{I_m} - \bm{\mu}) - \E[Y_m\bm{\Sigma}^{-1}(\bm X_{I_m} - \bm \mu)]} > \epsilon/6 \Bigg) \\
    &\leq 4\Pr\Bigg(\norm{\bm{\hat \Sigma}^{-1} - \bm{\Sigma}^{-1}}\norm{\E[Y_m (\bm{X}_{I_m} - \bm{\mu})]} > \epsilon/6\Bigg) \\
    &+ 4\Pr\!\Bigg(\!\norm{\frac{1}{M}\!\sum_{m = 1}^M\! Y_m\!\bm{\Sigma}^{-1/2}(\bm{X}_{I_m}\! -\! \bm{\mu})\! -\! \E[Y_m\bm{\Sigma}^{-1/2}(\bm X_{I_m} \!-\! \bm \mu)]}\! >\! \sqrt{\lambda_{d}}\epsilon/6 \Bigg). \qed
\end{align*}
\section{Proof of Lemma~\ref{lem:4terms}}
\label{app:4terms}
We remind the reader that $\bm W_n = \bm{\Sigma}^{-1/2}(\bm X_n-\bm{\mu})$ and we show that
\begin{align}
    \label{eq:two_terms}
    &\norm{\frac{1}{M}\sum_{m = 1}^M Y_m\bm{\Sigma}^{-1/2}(\bm{X}_{I_m}-\bm{\mu}) - \E[Y_m\bm{\Sigma}^{-1/2}(\bm{X}_{I_m}-\bm{\mu})]} \nonumber \\
    &= \norm{\frac{1}{M}\sum_{n = 1}^N \bm W_n \sum_{m: I_m = n} Y_m - \E[Y_m\bm{\Sigma}^{-1/2}(\bm{X}_{I_m}-\bm{\mu})]} \nonumber \\
    &= \Bigg{\|} \frac{1}{M}\sum_{n = 1}^N \bm W_n M_n \E[Y_m|I_m = n, \{\bm X_{n'} = \bm x_{n'}\}_{n' = 1}^N] - \E[Y_m\bm{\Sigma}^{-1/2}(\bm{X}_{I_m}-\bm{\mu})] \nonumber \\
    &+\! \frac{1}{M}\!\sum_{n = 1}^N\! \bm W_n M_n \!\left[\frac{1}{M_n}\!\sum_{m: I_m = n} Y_m \!- \!\E[Y_m|I_m = n, \{\bm X_{n'} \!=\! \bm x_{n'}\}_{n' = 1}^N]\right]\! \Bigg{\|} \nonumber\\
    &\leq\! \norm{\frac{1}{M}\!\sum_{n = 1}^N \! \bm W_n  M_n \tilde g_n(\{\bm X_{n'}\}_{n' = 1}^N)\! -\! \E[Y_m\bm{\Sigma}^{-1/2}\!(\bm{X}_{I_m}\!-\!\bm{\mu})]} ~~~ \text{by Eq.~\eqref{eq:gtn},} \nonumber\\
    &+ \norm{\frac{1}{M}\sum_{n = 1}^N \bm W_n M_n \Delta_n(\{\bm X_{n'}\}_{n' = 1}^N)},
\end{align}
where the last line is by Eq.~\eqref{eq:Deltan} and a triangle inequality. We expand the first term in Eq.~\eqref{eq:two_terms}
\begin{align}
    \label{equ:final_term}
    &\norm{\frac{1}{M}\sum_{n = 1}^N \bm W_n M_n \tilde g_n(\{\bm X_{n'}\}_{n' = 1}^N) - \E[Y_m\bm{\Sigma}^{-1/2}(\bm{X}_{I_m}-\bm{\mu})]}\nonumber \\
    &= \Bigg|\Bigg|\frac{1}{M}\sum_{n = 1}^N \bm W_n M_n \tilde g_n(\{\bm X_{n'}\}_{n' = 1}^N) - \frac{1}{N}\sum_{n=1}^N \bm W_n\tilde g_n(\{\bm X_{n'}\}_{n' = 1}^N) \nonumber \\
    &+\frac{1}{N}\sum_{n=1}^N \bm W_n\tilde g_n(\{\bm X_{n'}\}_{n' = 1}^N) -\E[Y_m\bm{\Sigma}^{-1/2}(\bm{X}_{I_m}-\bm{\mu})]\Bigg|\Bigg|\nonumber\\
    &\leq \norm{\frac{1}{M}\sum_{n = 1}^N \bm W_n M_n \tilde g_n(\{\bm X_{n'}\}_{n' = 1}^N) - \frac{1}{N}\sum_{n=1}^N \bm W_n \tilde g_n(\{\bm X_{n'}\}_{n' = 1}^N)} \nonumber \\
    &+\norm{\frac{1}{N}\sum_{n=1}^N \bm W_n\tilde g_n(\{\bm X_{n'}\}_{n' = 1}^N) - \E[Y_m\bm{\Sigma}^{-1/2}(\bm{X}_{I_m}-\bm{\mu})]} \nonumber \\
    &= \norm{\sum_{n=1}^N \left(\frac{M_n}{M}-\frac{1}{N}\right)\bm W_n \tilde g_n(\{\bm X_{n'}\}_{n' = 1}^N)} \nonumber \\
    &+ \norm{\frac{1}{N}\sum_{n=1}^N \bm W_n\tilde g_n(\{\bm X_{n'}\}_{n' = 1}^N) - \E[Y_m\bm{\Sigma}^{-1/2}(\bm{X}_{I_m}-\bm{\mu})]}.
\end{align}
For the second term in Eq.~\eqref{equ:final_term}, we have
\begin{align}
    \label{eq:xngtn}
     &\norm{\frac{1}{N}\sum_{n=1}^N \bm W_n \tilde g_n(\{\bm X_{n'}\}_{n' = 1}^N) - \E[Y_m\bm{\Sigma}^{-1/2}(\bm{X}_{I_m}-\bm{\mu})]} \nonumber \\
     &\leq \norm{\frac{1}{N}\sum_{n=1}^N \bm W_n \big(\tilde g_n(\{\bm X_{n'}\}_{n' = 1}^N) - g_n(\bm X_n)\big)} \nonumber \\
     &+ \norm{\frac{1}{N}\sum_{n=1}^N\bm W_n g_n(\bm X_n) - \E[Y_m\bm{\Sigma}^{-1/2}(\bm{X}_{I_m}-\bm{\mu})]} \nonumber \\
     &= \norm{\frac{1}{N}\sum_{n=1}^N \bm W_n z_n(\{\bm X_{n'}\}_{n' = 1}^N)} \nonumber \\
     &+ \norm{\frac{1}{N}\sum_{n=1}^N\bm W_n g_n(\bm X_n) - \E[Y_m\bm{\Sigma}^{-1/2}(\bm{X}_{I_m}-\bm{\mu})]},
\end{align}
by Eq.~\eqref{eq:zn}. By combining Equations \eqref{eq:two_terms}, \eqref{equ:final_term}, \eqref{eq:xngtn} we get
\begin{align}
    &\norm{\frac{1}{M}\sum_{m = 1}^M Y_m \bm{\Sigma}^{-1/2}(\bm{X}_{I_m} - \bm{\mu})- \E[Y_m\bm{\Sigma}^{-1/2}(\bm X_{I_m} - \bm \mu)]} \nonumber \\
    &\leq \norm{\frac{1}{M}\sum_{n = 1}^N \bm W_n M_n \Delta_n(\{\bm X_{n'}\}_{n' = 1}^N)} + \norm{\frac{1}{N}\sum_{n=1}^N \bm W_n z_n(\{\bm X_{n'}\}_{n' = 1}^N)} \nonumber \\
    &+ \norm{\sum_{n=1}^N \left(\frac{M_n}{M}-\frac{1}{N}\right)\bm W_n \tilde g_n(\{\bm X_{n'}\}_{n' = 1}^N)} \nonumber \\
    &+ \norm{\frac{1}{N}\sum_{n=1}^N\bm W_n g_n(\bm X_n) - \E[Y_m \bm{\Sigma}^{-1/2}(\bm X_{I_m} - \bm \mu)]}. \nonumber \qed
\end{align} 
\section{Proof of Lemma~\ref{lem:binomial_gaussian_combined}}
\label{app:binomial_gaussian_combined}
The term of interest is
\begin{align}
    \label{eq:combination1}
    &\Pr\left(\norm{\sum_{n=1}^N \left(\frac{M_n}{M}-\frac{1}{N}\right)\bm W_n \tilde g_{n}(\{\bm X_{n'}\}_{n' = 1}^N)} > \epsilon \right) \nonumber \\
    &\leq \Pr\left( \sum_{n=1}^N \norm{\bigg(\frac{M_n}{M}-\frac{1}{N}\bigg)\bm W_n \tilde g_{n}(\{\bm X_{n'}\}_{n' = 1}^N)} \geq \epsilon \right) ~~~ \text{by a triangle inequality,} \nonumber \\
    &\leq \Pr\left(\sum_{n=1}^N \bigg| \frac{M_n}{M}-\frac{1}{N}\bigg| \norm{\bm W_n} \geq \epsilon\right) ~~~ \text{by the fact that $\big| \tilde g_{n}(\{\bm X_{n'}\}_{n' = 1}^N) \big| \leq 1$,} \nonumber \\
    &= \Pr\bigg( \sum_{n = 1}^N\left| \frac{M_n}{M} - \frac{1}{N} \right|\norm{\bm W_n} > \epsilon \mid \cap_{n = 1}^N \bigg\{\norm{\bm W_n} \leq \sqrt{\delta_0}\bigg\} \bigg)\times \nonumber \\
    &\Pr\bigg( \cap_{n = 1}^N \bigg\{\norm{\bm W_n} \leq \sqrt{\delta_0}\bigg\}\bigg) \nonumber \\
    &+ \Pr\bigg( \sum_{n = 1}^N\left| \frac{M_n}{M} - \frac{1}{N} \right|\norm{\bm W_n} > \epsilon \mid \cup_{n = 1}^N \bigg\{\norm{\bm W_n} \geq \sqrt{\delta_0}\bigg\} \bigg)\times \nonumber \\
    &\Pr\bigg( \cup_{n = 1}^N \bigg\{\norm{\bm W_n} \geq \sqrt{\delta_0}\bigg\} \bigg) \nonumber \\
    &\leq \Pr\bigg(\left\{\sum_{n = 1}^N\left| \frac{M_n}{M} - \frac{1}{N} \right|\norm{\bm W_n} > \epsilon\right\} \cap \cap_{n = 1}^N \bigg\{\norm{\bm W_n} \leq \sqrt{\delta_0}\bigg\} \bigg) \nonumber \\
    &+ \Pr\bigg(\cup_{n = 1}^N \bigg\{\norm{\bm W_n} \geq \sqrt{\delta_0}\bigg\}\bigg) \nonumber \\
    &\leq \Pr\bigg(\sum_{n = 1}^N\left| \frac{M_n}{M} - \frac{1}{N} \right| > \epsilon/\sqrt{\delta_0} \bigg) + \sum_{n = 1}^N \Pr\bigg(\norm{\bm W_n} \geq \sqrt{\delta_0}\bigg),
\end{align}
\highlight{
where the last line is by a union bound and letting $A = \sum_{n = 1}^N\left| \frac{M_n}{M} - \frac{1}{N} \right|$ $\norm{\bm W_n} > \epsilon \cap \cap_{n = 1}^N \bigg\{\norm{\bm W_n} \leq \sqrt{\delta_0}\bigg\}$ and noticing that the event $A$ implies the event $B = \sum_{n = 1}^N\left| \frac{M_n}{M} - \frac{1}{N} \right| > \epsilon/\sqrt{\delta_0}$, \emph{i.e.} $A \rightarrow B$ and therefore $A \subseteq B$. 
This results in $\Pr(A) \leq \Pr(B)$.
}
Since $M_n$ are binomial distributed with parameter $1/N$, we have
\begin{align}
    \label{eq:mn}
    \Pr\bigg(\sum_{n = 1}^N\left| \frac{M_n}{M} - \frac{1}{N} \right| > \epsilon/\sqrt{\delta_0} \bigg) \leq 2^N e^{ -\frac{\epsilon^2M}{2\delta_0}} ~~~ \text{by Lemma~\ref{lemma:huber_carol}.}
\end{align}
As $\norm{\bm W_n}^2$ is centralized chi-squared distributed with $d$ degrees of freedom,
\begin{align}
    \label{eq:xn}
    \Pr\!\bigg(\!\norm{\bm{W}_n}\! >\!  \sqrt{\delta_0}\bigg)\! =\! \Pr\!\bigg(\!\norm{\bm{W}_n}^2\! >\! \delta_0\! \bigg)\! \leq\! \Big( \frac{\delta_0}{d} e^{1 - \frac{\delta_0}{d}} \Big)^{d/2},
\end{align}
for $\delta_0 > d$ by Lemma~\ref{lem:chi_tail}.
By combining Equations \eqref{eq:combination1}, \eqref{eq:mn} and \eqref{eq:xn}, we have
\begin{align}
    \label{eq:mnxn}
    &\Pr\left(\norm{\sum_{n=1}^N \left(\frac{M_n}{M}-\frac{1}{N}\right)\bm W_n \tilde g_{n}(\{\bm X_{n'}\}_{n' = 1}^N)} > \epsilon \right) \nonumber \\
    &\leq 2^N e^{-\frac{\epsilon^2M}{2\delta_0}} + N\Big( \frac{\delta_0}{d} e^{1 - \frac{\delta_0}{d}} \Big)^{d/2}.\qed
\end{align} 
\section{Proof of Lemma~\ref{lem:xnmnDeltan}}
\label{app:xnmnDeltan}
We omit the dependence on $\{\bm X_{n'}\}_{n' = 1}^N$ for $\Delta_n\left(\{\bm X_{n'}\}_{n' = 1}^N\right)$ and note that
\begin{align}
    \label{eq:after_int_2}
    &\Pr\bigg(\norm{\frac{1}{M}\sum_{n = 1}^N \bm W_n M_n \Delta_n} > \epsilon \bigg) \leq \Pr \bigg(\sum_{n=1}^N\norm{\frac{1}{M}\bm W_n M_n \Delta_n} > \epsilon \bigg) \nonumber \\
    &= \Pr\bigg(\sum_{n=1}^N\norm{\frac{1}{M} \bm W_n M_n}\big|\Delta_n \big| > \epsilon \bigg) \nonumber \\
    &= \Pr\bigg(\sum_{n=1}^N\norm{\frac{1}{M} \bm W_n M_n}\big|\Delta_n \big| > \epsilon \mid \cap_{n=1}^N \left\{\big|\Delta_n \big|<\sqrt{\delta_1} \right\}\bigg)\times \nonumber \\
    &\Pr\bigg(\cap_{n=1}^N \left\{\big|\Delta_n \big|<\sqrt{\delta_1} \right\} \bigg) \nonumber \\
    &+ \Pr\bigg(\sum_{n=1}^N\norm{\frac{1}{M} \bm W_n M_n}\big|\Delta_n \big| > \epsilon \mid \cup_{n=1}^N \left\{\big|\Delta_n \big| \geq \sqrt{\delta_1} \right\}\bigg)\times \nonumber \\
    &\Pr\bigg(\cup_{n=1}^N \left\{\big|\Delta_n \big| \geq \sqrt{\delta_1} \right\}\bigg) \nonumber \\
    &\leq \Pr\bigg(\left\{\sum_{n=1}^N\norm{\frac{1}{M} \bm W_n M_n}\big|\Delta_n \big| > \epsilon\right\} \cap \cap_{n=1}^N \left\{\big|\Delta_n \big|<\sqrt{\delta_1} \right\}\bigg) \nonumber \\
    &+ \Pr\bigg(\cup_{n=1}^N \left\{\big|\Delta_n \big|\geq\sqrt{\delta_1} \right\}\bigg) \nonumber \\
    &\leq \Pr\bigg(\sum_{n=1}^N\norm{\frac{1}{M} \bm W_n M_n} > \epsilon/\sqrt{\delta_1} \bigg) + N\Pr\bigg(\left\{\big|\Delta_n \big|\geq\sqrt{\delta_1} \right\}\bigg),
\end{align}
\highlight{where we follow a similar approach as in Eq.~\eqref{eq:combination1}}. 
The first term in Eq.~\eqref{eq:after_int_2} can be expanded as
\begin{align}
    \label{eq:xnmn}
    &\Pr\!\bigg(\!\sum_{n=1}^N\!\norm{\frac{1}{M} \bm{\Sigma}^{-1/2}(\bm X_n\! -\! \bm\mu) M_n} > \frac{\epsilon}{\sqrt{\delta_1}} \bigg) \nonumber \\
    &= \Pr\left(\sum_{n=1}^N \norm{\frac{M_n}{M}\bm W_n-\frac{1}{N}\bm W_n+\frac{1}{N}\bm W_n}\geq\frac{\epsilon}{\sqrt{\delta_1}}\right)\nonumber\\
    & \leq \Pr\left(\sum_{n=1}^N \norm{\left(\frac{M_n}{M}-\frac{1}{N}\right)\bm W_n}+\sum_{n=1}^N \norm{\frac{1}{N}\bm W_n}\geq\frac{\epsilon}{\sqrt{\delta_1}}\right)\nonumber\\
    & \leq \Pr\left(\sum_{n=1}^N \norm{\left(\frac{M_n}{M}-\frac{1}{N}\right)\bm W_n}\geq \frac{\epsilon}{2\sqrt{\delta_1}}\right) + \Pr\left(\sum_{n=1}^N\norm{\frac{1}{N}\bm W_n}\geq\frac{\epsilon}{2\sqrt{\delta_1}}\right).
\end{align}
We have
\begin{align}
    \label{eq:mnxn2}
    \Pr\!\left(\sum_{n=1}^N \norm{\left(\frac{M_n}{M}\!-\!\frac{1}{N}\right)\!\bm{\Sigma}^{-1/2}(\bm X_n\!-\!\bm\mu)}\! \geq\! \frac{\epsilon}{2\sqrt{\delta_1}}\!\right)\! \leq\! 2^Ne^{-\frac{\epsilon^2M}{8\delta_1 \delta_2}}\! +\! N\left( \frac{\delta_2}{d} e^{1 - \frac{\delta_2}{d}} \right)^{d/2}
\end{align}
by Eq.~\eqref{eq:mnxn}. Then,
\begin{align}
    \label{eq:xn3}
    &\Pr\!\left(\sum_{n=1}^N\norm{\frac{1}{N}\bm{\Sigma}^{-1/2}(\bm X_n\!-\!\bm\mu)}\!\geq\!\frac{\epsilon}{2\sqrt{\delta_1}}\right)\!\leq\! \sum_{n=1}^N \Pr\!\left( \norm{\frac{1}{N}\bm{\Sigma}^{-1/2}(\bm X_n\! -\! \bm\mu)}\! \geq\! \frac{\epsilon}{2N\sqrt{\delta_1}} \right) \nonumber \\
    &= \sum_{n=1}^N \Pr \left( \norm{\bm W_n} \geq \frac{\epsilon}{2\sqrt{\delta_1}} \right) \leq \sum_{n=1}^N \left(\frac{\epsilon^2}{4d\delta_1}e^{1-\frac{\epsilon^2}{4d\delta_1}}\right)^{d/2}  \text{by Lemma~\ref{lem:chi_tail},} \nonumber \\
    &= N\left(\frac{\epsilon^2}{4d\delta_1}e^{1-\frac{\epsilon^2}{4d\delta_1}}\right)^{d/2},
\end{align}
where the first line is by a union bound. 
The second term in Eq.~\eqref{eq:after_int_2} is bounded by
\begin{align}
    \label{eq:to_be_integrated_delta}
    &\Pr\bigg(\big| \Delta_n(\{\bm X_{n'}\}_{n' = 1}^N) \big| \geq \sqrt{\delta_1} \bigg) \nonumber \\
    &= \Pr\bigg(\bigg| \frac{1}{M_n}\sum_{m: I_m = n} Y_m - \tilde g_n(\{\bm X_{n'}\}_{n' = 1}^N) \bigg| \geq \sqrt{\delta_1} \bigg) ~~~ \text{by Eq.~\eqref{eq:Deltan},} \nonumber \\
    &= \sum_{k = 0}^M\! \Bigg[\! \int\! \Pr\! \Bigg( \bigg| \frac{1}{M_n}\! \sum_{m: I_m = n}\! Y_m\! -\! \tilde g_n(\{\bm x_{n'}\}_{n' = 1}^N) \bigg|\! \geq\! \sqrt{\delta_1} \mid\! M_n\! =\! k,\! \{\! \bm X_{n'} \!=\! \bm x_{n'} \}_{n' = 1}^N\! \Bigg) \nonumber \\
    &\prod_{n'=1}^N f_{\bm X}(\bm x_{n'})d\bm x_{n'} \Bigg] \cdot \Pr(M_n = k).
\end{align}
Due to conditioning on $\{ \bm X_{n'} = \bm x_{n'} \}_{n' = 1}^N$, labels $Y_m$ are independent. Therefore,
\begin{align*}
    &\Pr\Bigg( \bigg| \frac{1}{M_n}\sum_{m: I_m = n} Y_m - \tilde g_n(\{\bm x_{n'}\}_{n' = 1}^N) \bigg| \geq \sqrt{\delta_1} \mid M_n = k, \{ \bm X_{n'} = \bm x_{n'} \}_{n' = 1}^N \Bigg) \nonumber \\
    &= \Pr\!\Bigg( \bigg|\! \frac{1}{k}\!\sum_{m: I_m = n}\! Y_m\! -\! \tilde g_n(\{\bm x_{n'}\}_{n' = 1}^N)\! \bigg|\! \geq\! \sqrt{\delta_1} \mid \{ \bm X_{n'}\! =\! \bm x_{n'} \}_{n' = 1}^N\! \Bigg)\! \leq\! 2e^{\! -\frac{k\delta_1}{2}\!},
\end{align*}
where the last line is by Lemma~\ref{lem:hoeffding}. Substituting this result back into Eq.~\eqref{eq:to_be_integrated_delta} yields
\begin{align}
    \label{eq:binomial_mgf}
    \Pr\bigg(\big| \Delta_n \big| \geq \sqrt \delta_1 \bigg) &\leq \sum_{k = 0}^M \int 2e^{-\frac{k\delta_1}{2}}\Big(\prod_{n'=1}^N f_{\bm X}(\bm x_{n'})d\bm x_{n'} \Big) \Pr(M_n = k) \nonumber \\
    &= 2\sum_{k = 0}^M e^{-\frac{k\delta_1}{2}} \Pr(M_n = k). 
\end{align}
By construction, $M_n, n \in [N]$ are binomial distributed with number of trials $M$ and $p = \frac{1}{N}$. 
The moment generating function of $M_n$ is $(1-p+pe^t)^M$. 
The summation in Eq.~\eqref{eq:binomial_mgf} is the moment generating function of $M_n$ with $t=-\delta_1/2$, which results in
\begin{align}
    \label{eq:Deltan2}
    \sum_{n=1}^N \Pr\bigg(\big| \Delta_n \big| \geq \sqrt{\delta_1} \bigg) \leq 2N\Bigg(1 - \frac{1}{N} +\frac{1}{N}e^{-\frac{\delta_1}{2}}\Bigg)^M.
\end{align}
We want to use an equivalent of this quantity in the form of an exponential since it will be easier to compare it with other terms. 
For this reason, we use the following lemma.
\begin{lemma}
\label{lem:seq}
Let $\{a_n\}_{n=1}^\infty$ and $\{b_n\}_{n=1}^\infty$ be positive sequences such that $a_n \rightarrow \infty$. Let $\alpha \in \mathbb{R}$. Then,
\begin{align}
    \bigg( 1 + \frac{\alpha}{a_n} \bigg)^{b_n} = e^{\frac{\alpha b_n}{a_n} + o\bigg(\frac{\alpha b_n}{a_n}\bigg)}.
\end{align}
\end{lemma}
\begin{proof}
By the Taylor approximation $\log(1+x) = x + o(x)$ as $x \rightarrow 0$,
\begin{align*}
    &\bigg(1+\frac{\alpha}{a_n}\bigg)^{b_n} = e^{b_n\log\bigg(1+\frac{\alpha}{a_n}\bigg)} = e^{b_n\bigg( \frac{\alpha}{a_n} + o\bigg(\frac{\alpha}{a_n}\bigg)\bigg)} = e^{\frac{\alpha b_n}{a_n} + o\bigg(\frac{\alpha b_n}{a_n}\bigg)}. 
\end{align*}
\end{proof}
\noindent We remind the reader, the Taylor expansion $e^x = 1 + x + o(x)$ as $x \rightarrow 0$ and note that
\begin{align}
    \label{eq:moment_to_exp}
    &2N\Bigg(1 - \frac{1}{N} +\frac{1}{N}e^{-\frac{\delta_1}{2}}\Bigg)^M = 2N\Bigg(1 + \frac{1}{N}\bigg(e^{-\frac{\delta_1}{2}} - 1 \bigg)\Bigg)^M \nonumber \\
    &= 2N\Bigg(1 + \frac{1}{N}\bigg(-\frac{\delta_1}{2} - o\bigg(\frac{\delta_1}{2}\bigg) \bigg)\Bigg)^M + 2N\Bigg(1 + \bigg(-\frac{\delta_1}{2N} - o\bigg(\frac{\delta_1}{2N}\bigg) \bigg)\Bigg)^M \nonumber \\
    &= 2Ne^{-\frac{M\delta_1}{2N} - o\bigg(\frac{M\delta_1}{2N}\bigg)} = 2e^{\log N -\frac{M\delta_1}{2N} - o\bigg(\frac{M\delta_1}{2N}\bigg)},
\end{align}
where the second line is by the Taylor expansion of $e^x$ and the last line is by Lemma~\ref{lem:seq}. Combining Equations \eqref{eq:after_int_2}, \eqref{eq:xnmn}, \eqref{eq:mnxn2}, \eqref{eq:xn3}, \eqref{eq:Deltan2}, \eqref{eq:moment_to_exp}, we have
\begin{align*}
    &\Pr\bigg(\norm{\frac{1}{M}\sum_{n = 1}^N \bm W_n M_n \Delta_n} > \epsilon \bigg) \leq 2^{N}e^{-\frac{\epsilon^2M}{8\delta_1\delta_2}} + N\bigg(\frac{\delta_2}{d}e^{1-\frac{\delta_2}{d}}\bigg)^{d/2} \\
    & + N\left(\frac{\epsilon^2}{4d\delta_1}e^{1-\frac{\epsilon^2}{4d\delta_1}}\right)^{d/2} + 2e^{\log N -\frac{M\delta_1}{2N} - o\bigg(\frac{M\delta_1}{2N}\bigg)}. \qed
\end{align*}
\section{Proof of Lemma~\ref{lem:xnzn}}
\label{app:xnzn}
For brevity, we use $z_n$ instead of $z_n(\{\bm X_{n'}\}_{n' = 1}^N)$ in below. Note that
\begin{align}
    \label{eq:xnzn}
    &\Pr\bigg(\norm{\frac{1}{N}\sum_{n=1}^N\bm{W}_n z_n} > \epsilon\bigg) \leq \Pr\bigg(\sum_{n=1}^N\norm{\frac{1}{N} \bm{W}_n z_n} > \epsilon\bigg) \nonumber \\
    &= \Pr\bigg(\sum_{n=1}^N\norm{\frac{1}{N} \bm{W}_n} \abs{z_n} > \epsilon \mid \cap_{n=1}^N \left\{\big|z_n \big|<\sqrt{\delta_3}\right\} \bigg)\times \nonumber \\
    &\Pr\bigg(\cap_{n=1}^N \left\{\big|z_n \big|<\sqrt{\delta_3}\right\} \bigg) \nonumber \\
    &+ \Pr\bigg(\sum_{n=1}^N\norm{\frac{1}{N} \bm{W}_n} \abs{z_n} > \epsilon \mid \cup_{n=1}^N \left\{\big|z_n \big| \geq \sqrt{\delta_3} \right\}\bigg)\times \nonumber \\
    &\Pr\bigg(\cup_{n=1}^N \left\{\big|z_n \big| \geq \sqrt{\delta_3} \right\}\bigg) \nonumber \\
    &\leq \Pr\bigg(\left\{\sum_{n=1}^N \norm{\frac{1}{N} \bm{W}_n} \abs{z_n} > \epsilon \right\} \cap  \cap_{n=1}^N \left\{\big|z_n \big|<\sqrt{\delta_3} \right\} \bigg) \nonumber \\ 
    &+ \Pr\bigg(\cup_{n=1}^N \left\{\big|z_n \big| \geq \sqrt{\delta_3} \right\}\bigg) \nonumber \\
    &\leq \Pr\bigg(\sum_{n=1}^N\norm{\frac{1}{N} \bm{W}_n} > \epsilon/\sqrt{\delta_3}\bigg) + N\Pr\bigg(\big|z_n \big| \geq \sqrt{\delta_3} \bigg),
\end{align}
where the first line is by a Triangle inequality \highlight{and by following a similar approach as in Eq.~\eqref{eq:combination1}}.
We have
\begin{equation}
    \label{eq:xn2}
    \Pr\bigg(\sum_{n=1}^N\norm{\frac{1}{N} \bm{W}_n} > \epsilon/\sqrt{\delta_3}\bigg) \leq N\Bigg( \frac{\epsilon^2}{d\delta_3} e^{1-\frac{\epsilon^2}{d\delta_3}} \Bigg)^{d/2},
\end{equation}
by Lemma~\ref{lem:chi_tail} for $\delta_3 \leq \epsilon^2/d$. We also have
\begin{align}
    \label{eq:zn2}
    &\Pr\bigg(\big|z_n(\{\bm X_{n'}\}_{n' = 1}^N) \big| \geq \sqrt{\delta_3}\bigg) \nonumber \\
    &= \int \Pr\bigg(\big|z_n(\{\bm X_{n'}\}_{n' = 1}^N) \big| \geq \sqrt{\delta_3}\mid \bm X_n = \bm x\bigg) f_{\bm X}(\bm x) d\bm x.
\end{align}
For the probability inside the integral, we have
\begin{align}
    &\Pr\bigg(\big|z_n\{\bm X_{n'}\}_{n' = 1}^N \big| \geq \sqrt{\delta_3}\mid \bm X_n = \bm x\bigg) \nonumber \\
    &= \Pr\bigg(\big|\tilde g_n(\{\bm X_{n'}\}_{n' = 1}^N) - g_n(\bm X_n) \big| \geq \sqrt{\delta_3} \mid \bm X_n = \bm x \bigg) \nonumber \\
    &= \Pr\bigg(\bigg|\sum_{j=1}^{N} g_{n,j}(\bm x, \bm X_j) \Pr(J=j) - g_n(\bm x) \bigg| \geq \sqrt{\delta_3}\bigg) \nonumber \\
    &= \Pr\bigg(\bigg|\frac{1}{N}\sum_{j=1}^{N} g_{n,j}(\bm x, \bm X_j) - g_n(\bm x) \bigg| \geq \sqrt{\delta_3}\bigg), \nonumber \\
\end{align}
where the first line is by Eq.~\eqref{eq:zn}. Notice that
\begin{align*}
g_n(\bm x) &= \E[Y_m \mid I_M = n, \bm X_n = \bm x] \\
&=\sum_{j=1}^N \E[Y_m \mid I_m = n, J_m = j, \bm X_n = \bm x]\Pr(J_m = j) \\
&= \frac{1}{N}\sum_{j=1}^N\int\E[Y_m \mid I_m = n, J_m = j, \bm X_n = \bm x, \bm X_j = y] f_{\bm X_j}(\bm y) d\bm y \\
&= \frac{1}{N}g_{n,n}(\bm x, \bm x) + \frac{1}{N}\sum_{j\in[N]\setminus n}\int g_{n,j}(\bm x, \bm y)f_{\bm X}(\bm y)d\bm y \\
&= \frac{1}{N}g_{n,n}(\bm x, \bm x) + \frac{N-1}{N}\int g_{n,j}(\bm x, \bm y)f_{\bm X_j}(\bm y) d\bm y.
\end{align*}
Then we continue with
\begin{align}
    \label{eq:zncond}
    & \Pr\bigg(\bigg|\frac{1}{N}\sum_{j=1}^{N} g_{n,j}(\bm x, \bm X_j) - g_n(\bm x) \bigg| \geq \sqrt{\delta_3}\bigg) \nonumber \\
    &= \Pr\bigg(\bigg|\frac{1}{N}g_{n,n}(\bm x, \bm x) + \frac{1}{N}\sum_{j\in [N]\setminus n} g_{n,j}(\bm x, \bm X_j) \nonumber \\
    &- \frac{1}{N}g_{n,n}(\bm x, \bm x) - \frac{1}{N}\sum_{j\in [N]\setminus n}\int g_{n,j}(\bm x, \bm y)f_{\bm X}(\bm y)d\bm y \bigg| \geq \sqrt{\delta_3}\bigg) \nonumber \\
    &= \Pr\bigg(\bigg|\frac{1}{N}\sum_{j\in [N]\setminus n} g_{n,j}(\bm x, \bm X_j) - \frac{N-1}{N}\int g_{n,j}(\bm x, \bm y)f_{\bm X}(\bm y)d\bm y \bigg| \geq \sqrt{\delta_3}\bigg) \nonumber \\
    &= \Pr\bigg(\bigg|\frac{1}{N-1}\sum_{j\in [N]\setminus n} g_{n,j}(\bm x, \bm X_j) - \int g_{n,j}(\bm x, \bm y)f_{\bm X}(\bm y)d\bm y \bigg| \geq \frac{N}{N-1}\sqrt{\delta_3} \bigg) \nonumber \\
    &\leq 2e^{-\frac{N^3\delta_3}{2(N-1)^2}} \leq 2e^{-\frac{N\delta_3}{2}},
\end{align}
where the last line is due to Lemma~\ref{lem:hoeffding} and since $1 < N/(N-1) \leq 2$ for $N > 1$.
Substituting Eq.~\eqref{eq:zncond} into Eq.~\eqref{eq:zn2} gives
\begin{align}
    \label{eq:znn}
    \Pr\bigg(\big|z_n(\{\bm X_{n'}\}_{n' = 1}^N) \big| \geq \sqrt{\delta_3}\bigg) &\leq \int 2e^{-\frac{N\delta_3}{2}} f_{\bm X}(\bm x) d\bm x = 2e^{-\frac{N\delta_3}{2}}.
\end{align}
By combining Equations \eqref{eq:xnzn}, \eqref{eq:xn2} and \eqref{eq:znn}, we have
\begin{align*}
    \Pr\!\bigg(\!\norm{\frac{1}{N}\!\sum_{n=1}^N \bm{W}_n z_n\{\bm X_{n'}\}_{n' = 1}^N}\! >\! \epsilon\bigg)\! \leq\! N\Bigg( \frac{\epsilon^2}{d\delta_3}e^{1-\frac{\epsilon^2}{d\delta_3}} \Bigg)^{d/2}\! +\! 2Ne^{-\frac{N\delta_3}{2}}. \qed
\end{align*}
\section{Proof of Lemma~\ref{lem:xngn}}
\label{app:xngn}
We have,
\begin{align}
    \label{eq:xngnterm}
    \Pr\left(\norm{\frac{1}{N}\sum_{n=1}^N \bm W_n g_n(\bm X_n) - \E[Y_m \bm \Sigma^{-1/2}(\bm X_{I_m}-\bm\mu)]} > \epsilon \right).
\end{align}
We first prove that $\bm W_n g(\bm X_n)$ is sub-gaussian. 
By Proposition 2.5.2 (iv) of \citet{vershynin2018high}, for every sub-gaussian random variable $X$, there exists a constant $C>0$ such that
\begin{equation}
    \label{eq:prop4}
    \E\left[ e^{X^2/C} \right] \leq 2.
\end{equation}
Let $\bm v \in S^{d-1}$ and $W_n = \bm v^\top \bm W_n \sim \mathcal{N}\left( \bm 0, 1 \right)$.
Note that $W_n$ is sub-gaussian and $|g_n(\bm X_n)| \leq 1$. 
Then, for $s > 0$ we have
\begin{align}
    \label{eq:subgproof}
    &\Pr\!\left( |\bm v^\top \bm W_n g(\bm X_n)|\! >\! t\! \right) = \Pr\!\left( s \left(\bm v^\top \bm W_n g(\bm X_n)\right)^2\! >\! s t^2 \right) \nonumber \\
    &= \Pr\left( e^{s \left(\bm v^\top \bm W_n g(\bm X_n)\right)^2} > e^{s t^2} \right) \nonumber \\
    &\leq e^{-s t^2}\E\left[ e^{s \left(\bm v^\top \bm W_n g(\bm X_n)\right)^2}\right] ~~~ \text{by Markov's inequality,} \nonumber \\
    &\leq e^{-s t^2}\E\left[ e^{s \left(\bm v^\top \bm W_n\right)^2}\right] = e^{-s t^2}\E\left[ e^{s W_n^2}\right] \leq 2e^{-\frac{t^2}{C}},
\end{align}
where the last line is by Eq.~\eqref{eq:prop4} for an appropriate constant $C>0$ and by setting $s=1/C$. As the tail of $\bm v^\top \bm W_n g(\bm X_n)$ decays super exponentially for all $\bm v \in S^{d-1}$, $\bm \xi_n = \bm W_n g_n(\bm X_n) - \E[Y_m \bm \Sigma^{-1/2}(\bm X_{I_m}-\bm\mu)]$ is indeed sub-gaussian for $n \in [N]$.
From Proposition 2.6.1 of \cite{vershynin2018high}, we know that the sum of independent zero-mean sub-gaussian random variables is sub-gaussian and
\begin{equation}
    \label{eq:sumsubg}
    \norm{\sum_{n=1}^N \bm \xi_n}_{\psi_2}^2 \leq c' \sum_{n=1}^N \norm{\bm \xi_n}_{\psi_2}^2,
\end{equation}
where $c'>0$ is a constant.
We have
\begin{align*}
    &\Pr\left(\norm{\frac{1}{N}\sum_{n=1}^N \bm W_n g_n(\bm X_n) - \E[Y_m \bm \Sigma^{-1/2}(\bm X_{I_m}-\bm\mu)]} > \epsilon \right) \\
    &= \Pr\left(\norm{\sum_{n=1}^N\left(\bm W_n g_n(\bm X_n) - \E[Y_m \bm \Sigma^{-1/2}(\bm X_{I_m}-\bm\mu)]\right)} > N\epsilon \right) \\
    &= \Pr\!\left(\norm{\sum_{n=1}^N \bm \xi_n }\! >\! N\epsilon \right)\! \leq\! e^{ -\frac{1}{4}\left(\sqrt{\frac{2N^2\epsilon^2}{\norm{\sum_{n=1}^N \bm \xi_n }_{\psi_2}^2}- d} - \sqrt{d} \right)^2} \leq e^{ -\frac{1}{4}\left(\sqrt{\frac{N\epsilon^2}{c'\norm{\bm \xi_n }_{\psi_2}^2}- d} - \sqrt{d} \right)^2} \\ 
    &= e^{-\frac{1}{4}\left(\sqrt{\frac{N\epsilon^2}{c_2}- d} - \sqrt{d} \right)^2},
\end{align*}
where the last lines are by Lemma~\ref{th:subgnorm}, Eq.~\eqref{eq:sumsubg} and we denote $c_2 = c'\norm{\bm \xi_n }_{\psi_2}^2 > 0$ is an absolute constant does not depend on $d$ or $N$.
\qed
\section{Proof of Lemma~\ref{lem:precision}}
\label{app:precision}
We define $\bm{W} = [\bm W_{N+1}, \dots, \bm W_{2N}]^\top \in \R^{N\times d}$ such that $\bm W_i = \bm{\Sigma}^{-1/2}(\bm X_i - \bm \mu)$ where $i \in [2N]/[N]$.
Also, let $\lambda_{\min}[\bm A]$ be the minimum singular value of a matrix $\bm A$.
We first prove that the minimum singular value of sum of symmetric matrices is lower bounded.
\begin{lemma}
\label{lem:eigenineq}
Let $\bm A, \bm B \in \R^{d\times d}$ be symmetric matrices. Then,
\begin{equation*}
    \lambda_{\min}[\bm A] + \lambda_{\min}[\bm B] \leq \lambda_{\min}[\bm A + \bm B].
\end{equation*}\end{lemma}
\begin{proof}
We have,
\begin{align*}
    \lambda_{\min}[\bm A + \bm B] &= \min_{\bm x \in \R^d}\left\{ \frac{\bm x^\top (\bm A + \bm B) \bm x}{\bm x^\top \bm x}\right\} =\min_{\bm x \in \R^d}\left\{ \frac{\bm x^\top \bm A \bm x}{\bm x^\top \bm x} + \frac{\bm x^\top \bm B \bm x}{\bm x^\top \bm x} \right\} \\
    &\geq \min_{\bm x \in \R^d}\left\{ \frac{\bm x^\top \bm A \bm x}{\bm x^\top \bm x}\right\} + \min_{\bm y \in \R^d}\left\{\frac{\bm y^\top \bm B \bm y}{\bm y^\top \bm y} \right\} = \lambda_{\min}[\bm A] + \lambda_{\min}[\bm B].
\end{align*}
\end{proof}
Then,
\begin{align*}
    &\Pr\!\left(\norm{\bm{\hat \Sigma}^{-1}\! -\! \bm{\Sigma}^{-1}}\! >\! \epsilon\right)\! \leq\! \Pr\left(\norm{\bm{\Sigma}^{-1/2}}\cdot\norm{\bm{\Sigma}^{1/2}\bm{\hat \Sigma}^{-1}\bm{\Sigma}^{1/2} - \bm I}\cdot\norm{\bm{\Sigma}^{-1/2}}\! >\! \epsilon\right) \\
    &=\Pr\left(\lambda_d^{-1}\norm{\bm{\Sigma}^{1/2}\bm{\hat \Sigma}^{-1}\bm{\Sigma}^{1/2} - \bm I} > \epsilon\right) =\Pr\left(\norm{\bm{\Sigma}^{1/2}\bm{\hat \Sigma}^{-1}\bm{\Sigma}^{1/2}} > \lambda_d\epsilon + 1\right) \\
    &=\Pr\left(\norm{\left(\bm{\Sigma}^{-1/2}\bm{\hat \Sigma}\bm{\Sigma}^{-1/2}\right)^{-1}} > \lambda_d\epsilon + 1\right) \\
    &= \Pr\left(\lambda_{\min}\left[\bm{\Sigma}^{-1/2}\bm{\hat \Sigma}\bm{\Sigma}^{-1/2}\right] < \frac{1}{\lambda_d\epsilon + 1}\right) \\
    &=\Pr\!\left(\!\lambda_{\min}\!\left[\sum_{i= 1}^N\bm{\Sigma}^{-1/2}(\bm X_i\! -\! \bm{\hat \mu})(\bm X_i\! -\! \bm{\hat \mu})^\top\bm{\Sigma}^{-1/2}\right]\! <\! \frac{N-d-2}{\lambda_d\epsilon + 1}\!\right) ~~~\text{by Eq.~\eqref{eq:sigma_hat},} \\
    &=\Pr\Bigg(\lambda_{\min}\Bigg[\sum_{i= 1}^N\bm{\Sigma}^{-1/2}(\bm X_i - \bm{\mu})(\bm X_i - \bm{\mu})^\top\bm{\Sigma}^{-1/2} \\
    &+ \sum_{i= 1}^N\bm{\Sigma}^{-1/2}(\bm X_i - \bm{\hat \mu})(\bm{\mu} - \bm{\hat \mu})^\top\bm{\Sigma}^{-1/2} + \sum_{i= 1}^N\bm{\Sigma}^{-1/2}(\bm{\mu} - \bm{\hat \mu})(\bm X_i - \bm{\hat \mu})^\top\bm{\Sigma}^{-1/2} \\
    &+ \sum_{i= 1}^N\bm{\Sigma}^{-1/2}(\bm \mu - \bm{\hat \mu})(\bm \mu - \bm{\hat \mu})^\top\bm{\Sigma}^{-1/2} \Bigg] < \frac{N-d-2}{\lambda_d\epsilon + 1}\Bigg) \\
    &=\Pr\Bigg(\lambda_{\min}\Bigg[\sum_{i= 1}^N\bm{\Sigma}^{-1/2}(\bm X_i - \bm{\mu})(\bm X_i - \bm{\mu})^\top\bm{\Sigma}^{-1/2}\\
    &+ \sum_{i= 1}^N\bm{\Sigma}^{-1/2}(\bm \mu - \bm{\hat \mu})(\bm \mu - \bm{\hat \mu})^\top\bm{\Sigma}^{-1/2} \Bigg] < \frac{N-d-2}{\lambda_d\epsilon + 1}\Bigg) \\
    &\leq\Pr\Bigg(\lambda_{\min}\Bigg[\sum_{i= 1}^N\bm{\Sigma}^{-1/2}(\bm X_i - \bm{\mu})(\bm X_i - \bm{\mu})^\top\bm{\Sigma}^{-1/2}\Bigg] \\
    &+ \lambda_{\min}\Bigg[\sum_{i= 1}^N\bm{\Sigma}^{-1/2}(\bm \mu - \bm{\hat \mu})(\bm \mu - \bm{\hat \mu})^\top\bm{\Sigma}^{-1/2} \Bigg] < \frac{N-d-2}{\lambda_d\epsilon + 1}\Bigg) ~~~ \text{by Lemma~\ref{lem:eigenineq},} \\
    &= \Pr\Bigg(\lambda_{\min}\Bigg[\sum_{i= 1}^N\bm{\Sigma}^{-1/2}(\bm X_i - \bm{\mu})(\bm X_i - \bm{\mu})^\top\bm{\Sigma}^{-1/2}\Bigg] < \frac{N-d-2}{\lambda_d\epsilon + 1}\Bigg) \\
    &= \Pr\Bigg(\lambda_{\min}\Bigg[\bm{W}^\top\bm{W}\Bigg] < \frac{N-d-2}{\lambda_d\epsilon + 1}\Bigg) = \Pr\Bigg(\lambda_{\min}\Bigg[\bm{W}\Bigg] < \sqrt{\frac{N-d-2}{\lambda_d\epsilon + 1}}\Bigg).
\end{align*}
Notice that the rows of $\bm{W}$ are independent isotropic Gaussian vectors.
Therefore, we can apply Lemma~\ref{th:singvalbound} to get,
\begin{align}
    \Pr\left(\norm{\bm{\hat \Sigma}^{-1} - \bm{\Sigma}^{-1}} > \epsilon\right) &\leq \Pr\left(\lambda_{\min}\left[\bm{ W}\right] < \sqrt{\frac{N-d-2}{\lambda_d\epsilon + 1}}\right) \nonumber \\
    &\leq 2e^{-c_4\left(\sqrt N - \sqrt{\frac{N-d-2}{\lambda_d\epsilon + 1}} - c_3\sqrt d \right)^2},
\end{align}
for $\sqrt N > \sqrt{\frac{N-d-2}{\lambda_d\epsilon + 1}} + c_3\sqrt d$ where $c_3, c_4 > 0$ are constants. 
Furthermore, we have
\begin{align*}
    &\norm{\E[Y_m (\bm{X}_{I_m} - \bm{\mu})]} \leq \E[\norm{Y_m (\bm{X}_{I_m} - \bm{\mu})}] ~~~ \text{by Jensen's inequality,} \\
    &\leq \E[\norm{\bm{X}_{I_m} - \bm{\mu}}] \leq \E\left[\norm{\bm \Sigma^{1/2}}\cdot\norm{\bm \Sigma^{-1/2}(\bm{X}_{I_m} - \bm{\mu})}\right] \\
    &\leq \sqrt{\lambda_1} \E[\norm{\bm{W}_{I_m}}] = \sqrt{2\lambda_1} \frac{\Gamma\left(\frac{d+1}{2}\right)}{\Gamma\left(\frac{d}{2}\right)}.
\end{align*}
Finally,
\begin{align}
    &\Pr\left(\norm{\bm{\hat \Sigma}^{-1} - \bm{\Sigma}^{-1}}\cdot\norm{\E[Y_m (\bm{X}_{I_m} - \bm{\mu})]} > \epsilon\right) \nonumber \\
    &\leq \Pr\bigg(\norm{\bm{\hat \Sigma}^{-1} - \bm{\Sigma}^{-1}}\sqrt{2\lambda_1} \frac{\Gamma\left(\frac{d+1}{2}\right)}{\Gamma\left(\frac{d}{2}\right)} > \epsilon\bigg) \leq 2e^{-c_4\bigg(\sqrt N - \sqrt{\frac{N-d-2}{\frac{\Gamma\left(\frac{d}{2}\right)\lambda_d}{ \Gamma\left(\frac{d+1}{2}\right)\sqrt{2\lambda_1}}\epsilon + 1}} - c_3\sqrt d \bigg)^2} \nonumber \\
    &\leq 2e^{-c_4\bigg(\sqrt N - \sqrt{\frac{N-d-2}{\frac{\lambda_d}{ d\sqrt{2\lambda_1}}\epsilon + 1}} - c_3\sqrt d \bigg)^2},
\end{align}
for $\sqrt N > \sqrt{\frac{N-d-2}{\frac{\lambda_d}{ d\sqrt{2\lambda_1}}\epsilon + 1}} + c_3\sqrt d$ where $c_3, c_4 > 0$ are constants that do not depend on $d$ or $N$. \qed
\section{Choosing trade-off variables}
\label{app:choosing_delta}
We have
\begin{align}
    \label{eq:bound_with_deltas2}
    &\Pr\left(\norm{\hat{\bm{\beta}} - c_1\bm{\beta}} \geq \epsilon \right) \leq 8e^{ -c_4\bigg(\sqrt N - \sqrt{\frac{N-d-2}{\frac{\lambda_d}{ 6d\sqrt{2\lambda_1}}\epsilon + 1}} - c_3\sqrt d \bigg)^2} + 2^{N+2}e^{-\frac{\epsilon^2M\lambda_{d}}{4608\delta_1\delta_2}} \nonumber \\
    & + 4N\bigg(\frac{\delta_2}{d}e^{1-\frac{\delta_2}{d}}\bigg)^{d/2} + 8e^{\log N -\frac{M\delta_1}{2N} - o\bigg(\frac{M\delta_1}{2N}\bigg)} + 2^{N+2}e^{-\frac{\epsilon^2M\lambda_{d}}{1152\delta_0}}  \nonumber\\
    &+ 4N\bigg(\frac{\epsilon^2\lambda_{d}}{576d\delta_3} e^{1-\frac{\epsilon^2\lambda_{d}}{576d\delta_3}} \bigg)^{d/2} + 8Ne^{-\frac{N\delta_3}{2}} + 4e^{-\frac{1}{4}\left(\sqrt{\frac{\epsilon^2N\lambda_{d}}{c_2}- d} - \sqrt{d} \right)^2} \nonumber \\
    &+ 4N\Bigg( \frac{\delta_0}{d} e^{1 - \frac{\delta_0}{d}} \Bigg)^{d/2} + 4N\left(\frac{\epsilon^2\lambda_{d}}{2304d\delta_1}e^{1-\frac{\epsilon^2\lambda_{d}}{2304d\delta_1}}\right)^{d/2},
\end{align}
for $\sqrt N > \sqrt{\frac{N-d-2}{\frac{\lambda_d}{ 6d\sqrt{2\lambda_1}}\epsilon + 1}} + c_3\sqrt d$ where $c_1, c_2, c_3, c_4 > 0$ are absolute constants. The bounds we use require $\delta_2 > d$, $\delta_1 < \epsilon^2\lambda_{d}/2304d$, $\delta_1 > 2N\log(N)/M$, $\delta_0 > d$, $\delta_3 < \epsilon^2\lambda_d/576d$ and $N > dc_2/\epsilon^2\lambda_d$. 
Terms $\delta_0$, $\delta_1$, $\delta_2$ and $\delta_3$ need to be defined as functions of $N, M$ and $d$ such that the conditions arising from the tail bounds hold and the exponential terms' limit are $0$ as $N, M \rightarrow \infty$. 
In order to achieve this, we start dealing with $\delta_0$ first. The condition $\delta_0$ needs to satisfy is
\begin{align*}
    d(1+\log(\delta_0/d)) + 2\log N < \delta_0 < \frac{\epsilon^2M\lambda_d}{1152N\log2}.
\end{align*}
The condition for $\delta_1$ is
\begin{align*}
    \frac{2N\log N}{M} < \delta_1 < \frac{\epsilon^2\lambda_d}{2304\left[ d\left(\log\left(\frac{\epsilon^2\lambda_d}{2304d\delta_1}\right)+1\right)+2\log N \right]}.
\end{align*}
The condition for $\delta_2$ is
\begin{align*}
    d(1+\log(\delta_2/d)) + 2\log N < \delta_2 < \frac{\epsilon^2M\lambda_d}{4608N\delta_1\log 2 }.
\end{align*}
The condition for $\delta_3$ is
\begin{align*}
    \frac{2\log N}{N} < \delta_3 < \frac{\epsilon^2\lambda_d}{576\left[ 2\log N + d\left( \log\left( \frac{\epsilon^2\lambda_d}{576d\delta_3} \right) +1 \right) \right]}.
\end{align*}
We let $M = O\left(\frac{dN\log^3 N}{\lambda_d\epsilon^2}\right)$ together with, $\delta_0 = d\log^2 N, \delta_1 = \frac{4\lambda_d\epsilon^2}{d\log^2 N}, \delta_2 = d\log^2 N, \delta_3 = \frac{\epsilon^2\lambda_d}{1152d\log N}$. Substituting these quantities in \eqref{eq:bound_with_deltas2} gives,
\begin{align}
    &\Pr\left(\norm{\hat{\bm{\beta}} - c_1\bm{\beta}} \geq \epsilon \right) \leq 8e^{-c_4\bigg(\sqrt N - \sqrt{\frac{N-d-2}{\frac{\lambda_d}{ 6d\sqrt{2\lambda_1}}\epsilon + 1}} - c_3\sqrt d \bigg)^2} + 2^{N+2}e^{-\frac{N\log^3N}{18432\lambda_d\epsilon^2}} \nonumber \\
    &+ 4N(\log^2Ne^{1-\log^2N})^{d/2} + 8e^{-4N(\log N -1)} + 2^{N+2}e^{-\frac{N\log N}{1152}} \nonumber \\
    &+ 4N( 2\log N e^{1-2\log N})^{d/2} + 8Ne^{-\frac{N\epsilon^2\lambda_d}{2304d\log N}} + 4N(\log^2 N e^{1 - \log^2N })^{d/2} \nonumber \\
    &+ 4e^{-\frac{1}{4}\left(\sqrt{\frac{\epsilon^2N\lambda_{d}}{c_2}- d} - \sqrt{d} \right)^2} + 4N\left(\frac{\log^2N}{9216}e^{1-\frac{\log^2N}{9216}}\right)^{d/2},
\end{align}
for $\sqrt N > \sqrt{\frac{N-d-2}{\frac{\lambda_d}{ 6d\sqrt{2\lambda_1}}\epsilon + 1}} + c_3\sqrt d$ where $c_1, c_2, c_3, c_4 > 0$ are constants, $N > dc_2/\epsilon^2\lambda_d$, $N > \frac{2304d\log^2N}{\epsilon^2\lambda_d}$, $\log N > 347$, $\log N > 18 (\lambda_d\epsilon^2)^{1/3}$. 
We can simplify this bound for large enough $N$. We consider terms with $\log^2N$ and $1-2\log N$ first. One of the $\log^2N$ terms is divided by $9216$. When $\log^2 N$ is higher than this number, it will be dominated by the $2\log N$ term. Therefore, for $\log N > c_5$ where $c_5 = 18432$ we have,
\begin{align}
    &\Pr\left(\norm{\hat{\bm{\beta}} - c_1\bm{\beta}} \geq \epsilon \right) \leq 8e^{-c_4\bigg(\sqrt N - \sqrt{\frac{N-d-2}{\frac{\lambda_d}{ 6d\sqrt{2\lambda_1}}\epsilon + 1}} - c_3\sqrt d \bigg)^2} + 2^{N+2}e^{-\frac{N\log^3N}{18432\lambda_d\epsilon^2}} \nonumber \\
    & + 8e^{-4N(\log N -1)} + 16N( 2\log N e^{1-2\log N})^{d/2} + 8Ne^{-\frac{N\epsilon^2\lambda_d}{2304d\log N}} \nonumber \\
    &+ 4e^{-\frac{1}{4}\left(\sqrt{\frac{\epsilon^2N\lambda_{d}}{c_2}- d} - \sqrt{d} \right)^2} + 2^{N+2}e^{-\frac{N\log N}{1152}}.
\end{align}
Now we consider the terms with the exponent $N\log^3 N$ and $N\log N$. For $\log N > 4\epsilon\sqrt{\lambda_d}$, we can reduce the bound to
\begin{align}
    &\Pr\left(\norm{\hat{\bm{\beta}} - c_1\bm{\beta}} \geq \epsilon \right) \leq 8e^{-c_4\bigg(\sqrt N - \sqrt{\frac{N-d-2}{\frac{\lambda_d}{ 6d\sqrt{2\lambda_1}}\epsilon + 1}} - c_3\sqrt d \bigg)^2} + 8Ne^{-\frac{N\epsilon^2\lambda_d}{2304d\log N}} \nonumber \\
    & + 16N( 2\log N e^{1-2\log N})^{d/2} + 2^{N+3}e^{-\frac{N\log N}{1152}} + 4e^{-\frac{1}{4}\left(\sqrt{\frac{\epsilon^2N\lambda_{d}}{c_2}- d} - \sqrt{d} \right)^2},
\end{align}
for $\sqrt N > \sqrt{\frac{N-d-2}{\frac{\lambda_d}{ 6d\sqrt{2\lambda_1}}\epsilon + 1}} + c_3\sqrt d$, $N > dc_2/\epsilon^2\lambda_d$, $N > \frac{2304d\log^2N}{\epsilon^2\lambda_d}$, $\log N > 18 (\lambda_d\epsilon^2)^{1/3}$, $\log N > c_5$, $\log N > 4\epsilon\sqrt{\lambda_d}$ where $c_1, c_2, c_3, c_4, c_5 > 0$ are absolute constants. Note that under given conditions, we have 2 terms that are competing, i.e. the term with $\log N$ and the term with $N/\log N$. Therefore, the bound reduces to
\begin{align}
\Pr\left(\norm{\hat{\bm{\beta}} - c_1\bm{\beta}} \geq \epsilon \right) &\leq c_6N\max\left\{\left(\frac{\sqrt{6\log N}}{N}\right)^d, e^{-\frac{N\epsilon^2\lambda_d}{c_7d\log N}} \right\},
\end{align}
for $N > \frac{c_8d\log^2N}{\epsilon^2\lambda_d}$ where $c_1, c_6, c_7, c_8 > 0$ are absolute constants.
\highlight{\section{Approximating variables}
\label{app:est_c_p}
\noindent Our synthetic experiments require the knowledge of $c_1 \in \R$ from Lemma~\ref{lem:unbiased} and the probability of a label being flipped $p_e \in [0, 1]$ for a sigmoid with adjustable slope $f(x) = (1+e^{-\alpha x})^{-1}$ where $\alpha > 0$. We estimate these values as explained in below.}

\noindent\textbf{Approximating $c_1$.} \highlight{We remind the reader that $c_1 = 4\E[f'(\bm\beta^\top(\bm X_{I_m} - \bm X_{J_m}))$ and the score of item $i \in [2N]$ is $s_i = \bm\beta^\top\bm X_i \in \R$ and we let $s_{i,j} = \bm\beta^\top\bm X_i - \bm\beta^\top\bm X_j \sim \mathcal{N}(0, \sigma^2)$ where $\sigma^2 = 2\bm\beta^\top\bm\Sigma\bm\beta$ since $\bm X_i \sim \mathcal{N}(\bm \mu, \bm \Sigma)$.
Note that $c_1$ is the result of a sigmoid-Gaussian type integral, \emph{i.e.} $c_1 = 4 \int f'(s)f_{s_{i,j}}(s) d_s$ where $f_{s_{i,j}}(s)$ is the probability density function of the distribution of $s_{i,j}$. As this integral does not have an analytical solution, we estimate it with trapezoidal rule by taking finely spaced values $s \in [-4\sigma, 4\sigma]$.}

\noindent\textbf{Approximating $p_e$.} \highlight{We remind the reader that $p_e$ is the probability of a true label being flipped and is a function of $f(x)$ in Eq.\eqref{eq:conditional} and $\bm \beta, \bm \Sigma$. We show that
\begin{align*}
    \Pr(\text{"Error in label $m$"}) &= \Pr\left(Y_m = 1 \cap s_i < s_j\right) + \Pr\left(Y_m = -1 \cap s_i > s_j\right) \\
    &= \Pr\left(Y_m = 1 \cap s_{i,j} < 0 \right) + \Pr\left(Y_m = -1 \cap s_{i,j} > 0 \right).
\end{align*}
Note that,
\begin{align*}
    \Pr\left(Y_m = 1 \cap s_{i,j} < 0 \right) &= \int_{-\infty}^{\infty} \Pr\left(Y_m = 1 \cap s_{i,j} < 0 | s_{i,j} = s \right)f_{s_{i,j}}(s) ds \\
    &= \int_{-\infty}^{0} \Pr\left(Y_m = 1 | s_{i,j} = s\right)f_{s_{i,j}}(s) ds \\
    &= \int_{-\infty}^{0} f(s) f_{s_{i,j}}(s) ds,
\end{align*}
which is again a sigmoid-Gaussian type integral. Furthermore, it is straightforward to show that $\Pr\left(Y_m = 1 \cap s_{i,j} < 0 \right) = \Pr\left(Y_m = -1 \cap s_{i,j} > 0 \right)$. We evaluate this integral with trapezoidal rule by taking finely spaced values $s \in [-4\sigma, 0]$.}

\end{document}